\documentclass[sigconf]{aamas}

\usepackage{balance} 
\usepackage{cleveref}
\usepackage{mathtools}
\usepackage{subcaption}
\usepackage{enumitem}
\usepackage{pifont}
\usepackage{amsmath}

\DeclareMathOperator*{\argmax}{arg\,max}

\newcommand{\human}[1]{#1_h}
\newcommand{\chain}[1]{#1_\text{chain}}

\newcommand{\ai}[1]{#1_{AI}}

\newcommand{\vgoal}{V^{\pi_g}}
\newcommand{\vdis}{V^{\pi_d}}

\newcommand{\mdp}{\mathcal{M}}
\newcommand{\indicator}[1]{\mathbb{I}\left \{ #1 \right \}}

\newcommand{\mdpclass}{\mathscr{M}}
\newcommand{\chainworld}{\mdp_\theta}
\newcommand{\simplechains}{\chain{\mdpclass}}
\newcommand{\mono}[1]{#1_\text{mono}}
\newcommand{\monotonicchains}{\mono{\mdpclass}}
\newcommand{\policyclass}{\bar \Pi}
\newcommand{\parenfrac}[2]{\left(\frac{#1}{#2}\right)}
\newcommand{\tmin}{\human{t}^\text{min}}
\newcommand{\aigoal}{\ai{r}^g}
\newcommand{\aidis}{\ai{r}^d}
\newcommand{\aiint}{\ai{r}^i}
\newcommand{\twoabs}[1]{{#1}_{\text{prog}}}
\newcommand{\multi}[1]{{#1}_\text{multi}}

\setcopyright{ifaamas}
\acmConference[AAMAS '24]{Proc.\@ of the 23rd International Conference
on Autonomous Agents and Multiagent Systems (AAMAS 2024)}{May 6 -- 10, 2024}
{Auckland, New Zealand}{N.~Alechina, V.~Dignum, M.~Dastani, J.S.~Sichman (eds.)}
\copyrightyear{2024}
\acmYear{2024}
\acmDOI{}
\acmPrice{}
\acmISBN{}

\acmSubmissionID{926}

\title{Reinforcement Learning Interventions on Boundedly Rational Human Agents in Frictionful Tasks}

\author{Eura Nofshin}
\affiliation{
  \institution{Harvard University}
  \city{Cambridge}
  \country{USA}}
\email{eurashin@g.harvard.edu}

\author{Siddharth Swaroop}
\affiliation{
  \institution{Harvard University}
  \city{Cambridge}
  \country{USA}}
\email{}

\author{Weiwei Pan}
\affiliation{
  \institution{Harvard University}
  \city{Cambridge}
  \country{USA}}
\email{}

\author{Susan Murphy}
\affiliation{
  \institution{Harvard University}
  \city{Cambridge}
  \country{USA}}
\email{}

\author{Finale Doshi-Velez}
\affiliation{
  \institution{Harvard University}
  \city{Cambridge}
  \country{USA}}
\email{}

\begin{abstract}
Many important behavior changes are \emph{frictionful}; they require individuals to expend effort over a long period with little immediate gratification. Here, an artificial intelligence (AI) agent can provide personalized interventions to help individuals stick to their goals. In these settings, the AI agent must personalize \emph{rapidly} (before the individual disengages) and \emph{interpretably}, to help us understand the behavioral interventions. In this paper, we introduce Behavior Model Reinforcement Learning (BMRL), a framework in which an AI agent intervenes on the parameters of a Markov Decision Process (MDP) belonging to a \emph{boundedly rational human agent}. Our formulation of the human decision-maker as a planning agent allows us to attribute undesirable human policies (ones that do not lead to the goal) to their maladapted MDP parameters, such as an extremely low discount factor. 
Furthermore, we propose a class of tractable human models that captures fundamental behaviors in frictionful tasks. Introducing a notion of \emph{MDP equivalence} specific to BMRL, we theoretically and empirically show that AI planning with our human models can lead to helpful policies on a wide range of more complex, ground-truth humans. 
\end{abstract}


\newcommand{\BibTeX}{\rm B\kern-.05em{\sc i\kern-.025em b}\kern-.08em\TeX}

\begin{document}


\pagestyle{fancy}
\fancyhead{}


\maketitle 


\section{Introduction}
In many AI+human applications of behavior change, AI agents assist the human in performing \emph{frictionful} tasks, where making progress toward the human's goal requires sustained effort over time with little immediate gratification. Examples include physical therapy (PT) programs, adherance to scheduled medication, or passing an online course. Two key challenges for AI agents in these settings are rapid personalization \citep{wang2021optimizingRLSimulatorHealth, park2023understandingDisengagement, tabatabaei2018narrowing} and learning interpretable policies for intervention \citep{trella2022designingRLforDigitalHealth, xu2022towardsBridgingModels}. In frictionful tasks, since effort exerted by the human does not reap immediate benefits, the AI agent must learn a personalized intervention policy for each human in a small number of interactions, or risk disengagement. These policies must also be interpretable to experts in behavioral science so that they can discover which interventions work for which individuals, and investigate why. 


Grounded in behavioral literature that treats humans as sequential decision-makers (e.g. \citep{taylor2021awarenessEatingBehavior, taylor2021awarenessSmoking, niv2009Dopamine1, shteingart2014Dopamine2, zhou2018personalizingFitness}), we model the human as an agent planning under a ``maladapted'' Markov Decision Process (MDP). In maladapted human MDPs, the optimal policy does not reach the human's stated goal; for example, in physical therapy (PT), the goal may be a rehabilitated shoulder and the maladapted MDP parameter may be an extremely low discount rate, $\gamma$. This results in myopic decision-making, wherein an individual forgoes the long-term goal (rehabilitated shoulder) to avoid experiencing friction in the short-term (unpleasantness of PT). The AI agent helps the individual achieve their long-term goal by changing the maladapted human MDP (and thereby the optimal policy).

While there is existing reinforcement learning (RL) literature for optimizing interventions on human utility functions (i.e. reward) in maladapted MDPs \citep{yu2022environmentDesignBiased, zhou2018personalizingFitness, mintz2023behavioralAnalytics}, interventions on $\gamma$ have not been optimized from an RL perspective.  On the other hand, in behavioral science, humans have been observed to use a problematically low $\gamma$ \citep{story2014doesTemporal} and scientists have developed interventions to change a human's $\gamma$ (e.g. \citep{magen2008hiddenZero}).  However, no work optimizes for \emph{when} and with what mechanisms to intervene on the parameters of the human's maladapted MDP. 

In this paper, we introduce a \emph{flexible} and \emph{behaviorally interpretable} framework called ``Behavior-Model RL'' (BMRL). In BMRL, the human is modeled as an RL agent, whose actions are \emph{behaviors}, such as performing or skipping PT; the AI agent provides personalized assistance by delivering \emph{interventions} on the human's maladapted MDP parameters. 
By linking the behaviors of our human agents to their MDP parameters, BMRL allows us to \emph{interpret} the mechanism behind the human's maladapted decision-making. Our framework is also more \emph{flexible} than existing ones, since we allow the AI agent's actions to include operations on any part of the human MDP (such as $\gamma$). 
By solving for the AI agent's optimal policy, we learn the best set of interventions to change the human agent's behavior and to help the human reach their goal.

Unfortunately, current RL approaches have two major drawbacks when used to solve for the optimal AI agent policy in BMRL. 
First, most planning methods are too data-intensive for our setting, in which personalization occurs online.
For example, online algorithms in robotics require thousands of interactions to learn reasonable policies (e.g. in \citep{yang2020dataEfficientRobots, thabet2019sampleDeepHRI, tebbe2021TableTennis}), but in frictionful tasks, we are limited to \emph{tens to hundreds} of interactions \citep{trella2022designingRLforDigitalHealth}. 
Second, existing planning methods model the human as a black-box transition or value function. Unfortunately, in learning black-box representations of the human agent, we lose the ability to interpretably attribute human behavior to their MDP parameters.

In this paper, we propose a tractable planning method for the AI agent in our BMRL framework. Our method provides the AI agent with a useful inductive bias, in the form of a human model that captures important behavioral patterns in frictionful tasks. Specifically, we identify a small, behaviorally-grounded model of the human that the AI agent can leverage to rapidly personalize interventions, including previously under-explored interventions on $\gamma$. 
Then, we introduce the concept of ``AI equivalence'' to identify a class of more complex human models for which AI policies learned in our simple human model can be lifted with no loss of performance. 
In our empirical analysis, we test whether AI planning with our small model is robust to complex human models that are not covered by our equivalence result. 
Throughout all of this, our small model preserves scientific interpretability-- in fact, it has an analytical solution for the human behavior policy-- which allows experts to inspect and learn from the AI policies. 

\section{Related Works}
\vskip0.15cm \emph{Computational modeling of human behaviors.} 
Behavioral scientists have developed and verified several computational models of \emph{dynamic} human decision-making. Unlike \emph{static} models, such as Social Cognitive Theory \citep{bandura1999social},  \emph{dynamic} models of decision-making apply to interactive human-AI settings, since they capture person-level variation and changes over time, as in \citet{zhang2022theoryBasedHabit}.
Scientists developed these models to explain \emph{offline data} from frictionful settings such as health (e.g. \citep{martin2018ControlFluidSCT, zhang2021usingCognitiveModels, wang2021optimizingRLSimulatorHealth}), energy \citep{mogles2018computationalEnergy}, and experience sampling \citep{khanshan2023simulatingESM} or to capture broader behaviors such as risk \citep{liu2019modelingRisk} and adherence \citep{pirolli2016computationalACTR}. However, these models involve too many latent variables-- corresponding to internal human processes-- to facilitate rapid AI learning from \emph{online data}.
In contrast, we propose a minimal, behaviorally-grounded model, one whose set of latent parameters is small and structured enough that our AI can learn.

\vskip0.15cm \emph{Computational modeling of human agent deficiencies.} 
RL is frequently used to model the complex mechanisms underlying human behavior, from the firing of dopaminergic neurons in the brain (e.g. in \citep{niv2009Dopamine1, shteingart2014Dopamine2}) to frictionful tasks such as mindful eating \citep{taylor2021awarenessEatingBehavior}, weight loss \citep{aswani2019behavioralWeightLoss}, and smoking cessation \citep{taylor2021awarenessSmoking}. Although these works use RL to model humans, the models themselves are not used to enrich planning for an AI agent.
One exception is inverse reinforcement learning, in which the AI agent infers the human agent's rewards (e.g. \citep{zhi2020onlineGoalInference, brown2019extrapolatingTREX}), transitions \citep{reddy2018you}, discount factor \citep{giwa2021estimationDiscount}, or entire MDP \citep{shah2019feasibility, evans2016learningIgnorant, jarrett2021inverseDecisonModeling}, but does not \emph{intervene} on the parts of the human MDP that are maladapted.
When the AI agent does intervene, the changes are limited to the human's reward \citep{yu2022environmentDesignBiased, tabrez2019explanationReward, zhou2018personalizingFitness, mintz2023behavioralAnalytics} or states \citep{chen2022mirror, reddy2021CommunicationState}. Our BMRL framework is flexible enough to incorporate AI interventions on multiple parts of the human MDP, including the discount factor or transitions.

\vskip0.15cm \emph{Equivalence of (human) MDP models.} In RL, there are notions of equivalence that can reduce larger human MDPs to smaller, more manageable ones. 
Equivalence, as defined in bisimulations \citep{givan2003equivalenceBisimulation}, homomorphism \citep{ravindran2002modelHomomorphisms}, and approximate homomorphisms (e.g. \citep{ravindran2004approximateHomomorphism, vanderpol2020plannable}), requires that one human MDP strictly preserves the transition and reward functions of another, given a mapping between the state and action spaces. 
State abstraction methods, which equate the optimal value function between two human-level MDPs, are less strict \citep{li2006towardsUnifiedAbstraction}. However, these equivalences are still \emph{stricter than necessary} in our setting, where we only care that the human MDPs are similar enough that the AI agent policy will not differ.
Furthermore, the simpler MDPs recovered by these methods are not guaranteed to be behaviorally valid or interpretable. In our approach, we define two human MDPs as equivalent if they lead to the same \emph{AI optimal policies}, and we use this definition to build up to more complex human MDPs from a behaviorally interpretable one. 


\section{The Behavior Model RL (BMRL) Framework for AI Interventions}
\label{sec: BMRL}

\begin{figure}[t] 
    \centering    \includegraphics[width=1\linewidth]{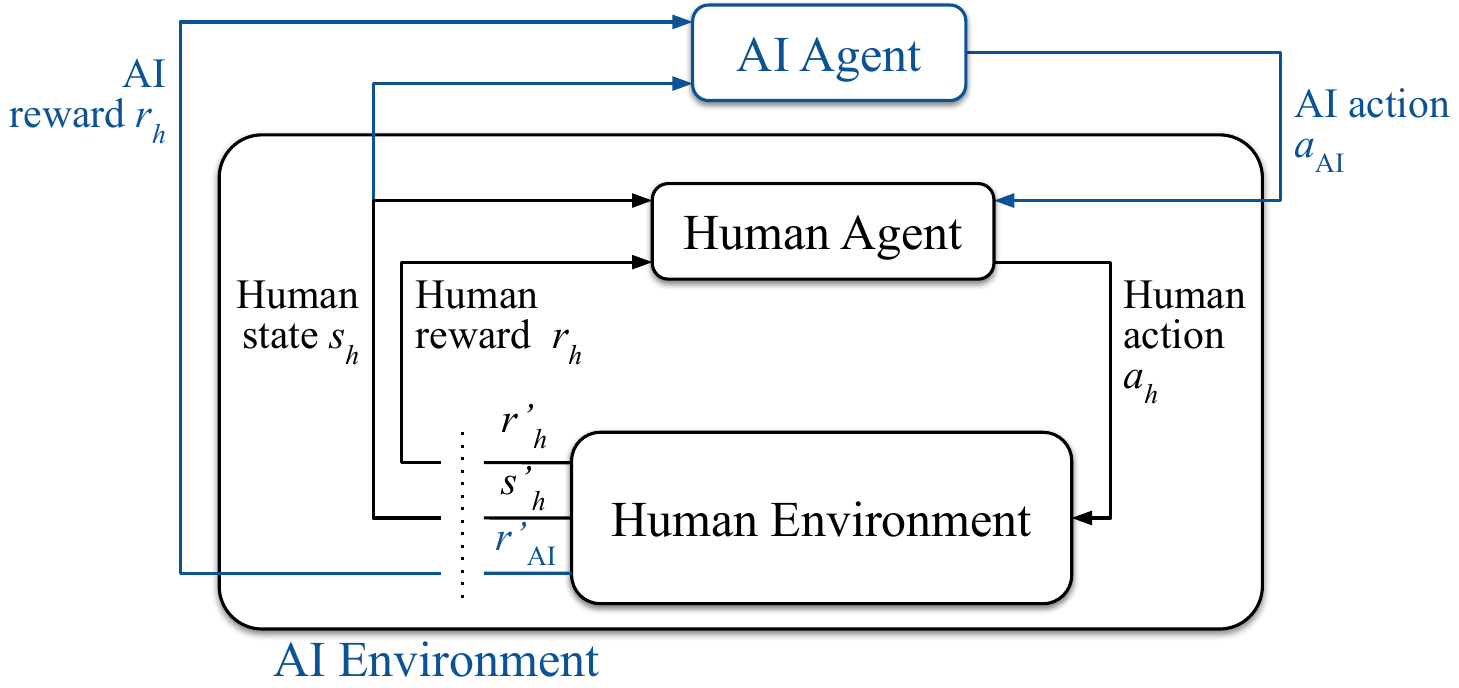}
    \caption{Overview of BMRL. \textnormal{The human agent interacts with the environment as in standard RL. The AI agent's actions affect the human agent. The human agent + environment form the AI environment. }}
    \label{fig: bmrl}
\end{figure}%

We define a formal framework,  called BMRL, in which an AI agent learns to intervene on a human agent's maladapted MDP parameters (overview in \cref{fig: bmrl}). 


\subsection{Assumptions on human agent}
\label{sec: human-agent-assumptions}
In BMRL, human agents perform optimal planning on (subconscious) knowledge of their MDP,
\begin{equation}
\label{eq: human-mdp}
    \human{\mathcal{M}} = \langle \human{\mathcal{S}}, \human{\mathcal{A}}, \human{T}, \human{R}, \human{\gamma}, s_g, s_d \rangle,
\end{equation}
where $s_g, s_d \in \human{\mathcal{S}}$ are absorbing goal (e.g. a rehabilitated shoulder) and disengagement states (e.g. quitting PT).

Though in general, it is possible for the human's perception of the states $\human{\mathcal{S}}$, actions $\human{\mathcal{A}}$ and transitions $\human{T}$ to be maladapted, in this paper we assume that the human's perception matches the true environment. 
On the other hand, we allow the human's rewards $\human{R}$ and discount $\human{\gamma}$ to vary by perception. For example, one may skip PT because of a tendency to ignore long-term rewards (low $\human{\gamma}$) while another may skip PT because they find the workout to be extremely unpleasant (bad $\human{R}$).

We assume that at any point the human subconsciously ``knows'' their own MDP, solves for the optimal policy, and uses it to select actions. In future work, BMRL can extend to sub-optimal human planning. 
Despite being optimal, our human agents are still boundedly rational because their MDP is maladapted. That is, under certain values of $\human{\gamma}, \human{R}$, even an optimal human policy will \emph{never} lead to the goal state (e.g. if the path to the goal reward is laced with extremely negative rewards). The existence of maladapted MDPs in humans is shown in behavior science, where myopic discounting has been linked to excessive alcohol intake \citep{story2014doesTemporal} or miscalibrated rewards have been linked to unhealthy eating \citep{taylor2021awarenessEatingBehavior}. Despite subconscious knowledge of their own MDP, our human agents are still boundedly rational because (1) they may not be \emph{conscious} of their deficiencies and unable to target them; (2) even if aware, they may still struggle to change their deficiencies. 
In both cases, behavioral interventions (delivered by the AI agent) can help.

\subsection{AI agent}
\label{sec: ai-agent-mdp}
Our AI agent encourages the human agent toward the goal by intervening on the human's decision-making parameters, such as $\human{\gamma}$. 
To do so, the AI agent plans according to an MDP, 
\begin{equation}
    \label{eq: ai-mdp}
    \ai{\mathcal{M}} = \langle \ai{\mathcal{S}}, \ai{\mathcal{A}}, \ai{T}, \ai{R}, \ai{\gamma} \rangle,
\end{equation} 
with known rewards $\ai{R}$ and unknown transitions $\ai{T}$.

Upon observing state $\ai{s} = [\human{s}, \human{a}]$, which consists of the human's current state and \emph{previous} action, the AI agent must decide whether to intervene on the human's discount ($\ai{a} = a_\gamma$), reward ($\ai{a} = a_R$), or to do nothing ($\ai{a} = 0$). In practice, a discounting intervention $a_\gamma$ could be “episodic future thinking,” where individuals imagine future events as if they are presently occurring \citep{brown2022episodicFutureThinking}; this could executed as a guided activity in-app. A common intervention on reward  $a_R$ is to offer extrinsic rewards, such as badges \citep{fanfarelli2015understandingBadges}. 
Domain experts would determine how the interventions are executed, e.g. if the burden intervention should be a badge, motivational message, or cash.

To encourage policies that quickly lead to the goal state, the AI agent receives a positive reward when the human reaches the goal state, a negative reward when the human disengages, and a negative reward for the ``cost'' of intervening. 
The AI's transitions factorize into two probability distributions, $\ai{T}(\ai{s}, \ai{a}, \ai{s}') = P(\human{s}' | \human{s}, \human{a} ) P(\human{a}' | \human{s}, \ai{a} ) =  \human{T}(\human{s}, \human{a}, \human{s}')\human{\pi}(\human{a}' | \human{s}, \ai{a})$. The first distribution refers to the human-level transitions $\human{T}$. The second distribution is over human actions; it is the human policy that results from the AI's intervention on the human's MDP. Importantly, we assume that the effect of AI actions on the human MDP is \emph{temporary}. For example, if the AI agent increases the human's discount factor $\human{\gamma}$ to $\human{\gamma}'$ in the current time step, the human's discounting will have reverted to $\human{\gamma}$ at the next time step.

In \cref{tab: who-knows-what}, we provide a comparison on what the AI and human agents separately know and observe.
Note that all of the AI agent's unknown parameters pertain to the human MDP $\human{\mdp}$ and are contained in the AI's transitions $\ai{T}$. Instead of explicitly learning $\human{\mdp}$ to form $\ai{T}$,  we could directly estimate $\ai{T}$ or $\ai{Q}^*$ using standard model-based or model-free techniques. 
However, by learning $\human{\mathcal{M}}$, we take advantage of the known structure of the problem; the better the AI's model of $\human{\mdp}$, the better the inductive bias for forming $\ai{T}$ (and therefore $\ai{\pi}^*$).


\begin{table}[t] 
    \centering
    \small
    \begin{tabular}{lcc}
        & Human agent & AI agent \\
        \hline
        \textbf{Knows}... & 
        $\human{\mathcal{S}}, \human{\mathcal{A}}, \human{T}, \human{R}, \human{\gamma}
        $ &
        $\ai{\mathcal{S}}, \ai{\mathcal{A}},\ai{R}
        $ \\
        \hline
        \textbf{Does not know}... &
        --- &
        $\ai{T}$ (includes $\human{T}, \human{R}, \human{\gamma}
        $)\\
        \hline
        \textbf{Observes}... &
        $\human{\mathcal{S}}, \human{\mathcal{A}}, \ai{\mathcal{A}}$ &
        $\human{\mathcal{S}}, \human{\mathcal{A}}, \ai{\mathcal{A}}$%
    \end{tabular}%
    \caption{Overview of what is known, unknown, and observable to the human and AI agent. \textnormal{the AI agent does not know (and must infer) the human agent's MDP ($\human{R}, \human{\gamma}$) and the true environmental transitions ($\human{T}$).}}%
    \label{tab: who-knows-what}
\end{table}%

\section{Rapid personalization in BMRL via a simple human model}

\subsection{Chainworlds: a simple human model that captures progress-based decision-making}
\label{sec: chain-worlds}
In this section we define \emph{chainworlds}, a class of simple human MDPs that the AI agent will use as a stand-in model for the \emph{true} human decision-making process. Chainworlds are based on the observation that many frictionful tasks contain a notion of human progress toward a goal; for example, in PT, the progress toward a rehabilitated shoulder may be summarized by the current strength of the joint.
We summarize these ``progress-based'' settings with a ``progress-only'' class of human MDPs, shown in \cref{fig: chainworld}, which we call \emph{chainworlds} and denote $\simplechains$. 

\begin{figure}[t] 
    \centering
    \begin{subfigure}{1\linewidth}
         \centering              \includegraphics[width=0.7\linewidth]{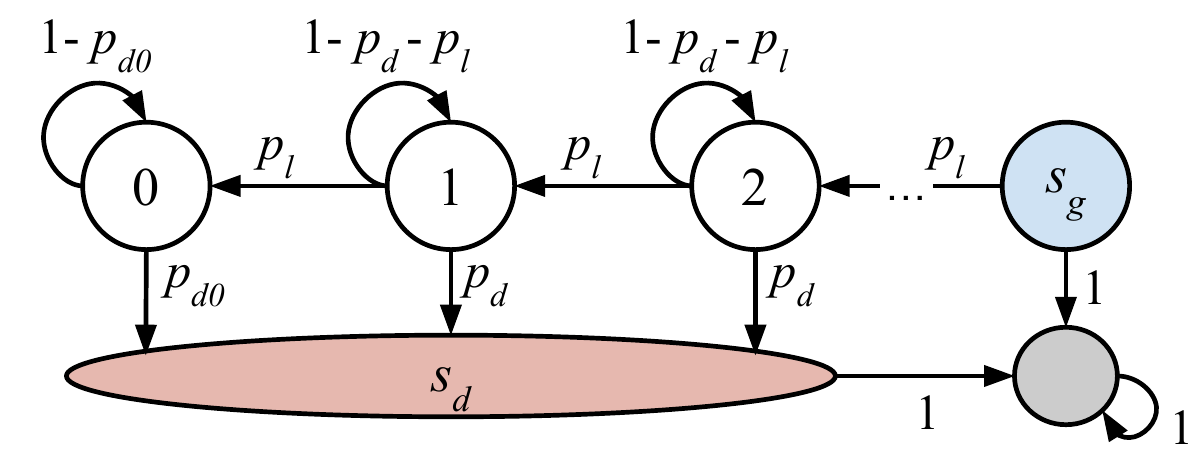}
         \label{fig: chain-action-0}
         \caption{$\human{a} = 0$}
     \end{subfigure}
    \begin{subfigure}{1\linewidth}
         \centering              \includegraphics[width=0.7\linewidth]{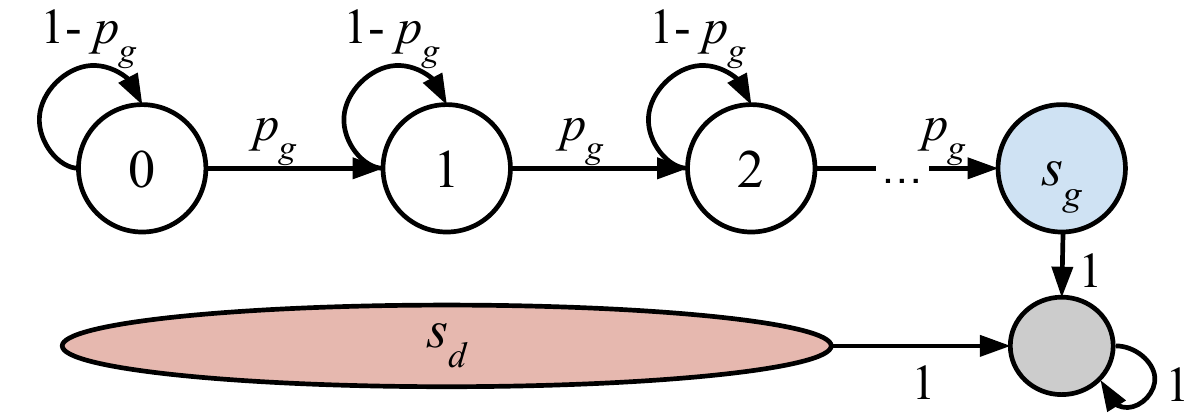}
         \label{fig: chain-action-1}
         \caption{$\human{a} = 1$}
     \end{subfigure}
    \caption{Graphical representation of the chainworld.}
    \label{fig: chainworld}
\end{figure}

Each element of $\simplechains$ is as follows: 
\begin{itemize}[leftmargin=*]
    \item \emph{States $\human{s} \in \{s_0, s_1, \ldots, s_N = s_g, s_d\}$.} The $N$ states are 1-D, discrete, and represent progress toward the goal. The goal state at the end of the chain, $s_N = s_g$ means that the human has rehabilitated their shoulder. The disengagement state $s_d$ means that the human has disengaged from PT.
    \item \emph{Actions  $\human{a} \in \{0, 1\}$.} The human decides to perform ($\human{a} = 1$) or not perform ($\human{a} = 0$) the goal-directed behavior. In the future, this could be extended to categorical actions. That said, many important applications have binary actions, such as "exercise or not" in PT, "smoke or not" in smoking cessation, and "adhered or not" in medication adherence.
    \item \emph{Rewards.} The human's utility function is the reward, 
    \begin{equation}
        \label{eq: chainworld-reward}
        \human{R}(s, a, s') = \begin{cases}
            r_b, &  a = 1 \\
            r_\ell, &  s' < s \\
            r_g, &  s = s_g \\
            r_d, &  s = s_d . 
        \end{cases}
    \end{equation}
    
    Goal behaviors, such as doing PT, incur a cost representing burden $r_b < 0$. Similarly, losing progress incurs $r_\ell< 0$. The goal and disengagement states have positive utility, $r_g > 0$ and $r_d > 0$. 
    \item \emph{Transitions.} The human knows that there is $p_g$ probability that they will move toward the goal as a result of the behavior, $p_\ell$ probability that they will lose progress from abstaining, and $p_d$ probability that they will disengage from abstaining. These probabilities are fixed across states, except for the first state $s_0$, which has a separate probability of disengagement $p_{d0} \ge p_d$. 
    \item \emph{Discount.} The human exponentially discounts future rewards via $\human{\gamma} \in [0, 1)$. We leave other behaviorally relevant forms of discounting, such as hyperbolic discounting \citep{fedus2019hyperbolic}, as future work. 
    \item \emph{Effect of AI interventions.} 
    When $\ai{a} = a_\gamma$ the human's discount $\human{\gamma}$ \emph{increases} by $\Delta_\gamma > 0$, and when $\ai{a} = a_b$ the human's burden $r_b < 0$ \emph{decreases} by $\Delta_b$. We clip $\human{\gamma} + \Delta_\gamma$ to be between $0$ and $1$.  
\end{itemize}

Each individual is an instance of the chainworld, $\chainworld \in \simplechains,$ with parameters $\theta$ = $\{r_b$, $r_\ell$, $r_g$, $r_d$, $p_g$, $p_\ell$, $p_d$, $p_{d0}$, $\human{\gamma}$, $\Delta_\gamma$, $\Delta_b \}$. For example, some people tend to prioritize short-term rewards (with a low $\human{\gamma}$) while others prioritize long-term rewards (with a high $\human{\gamma}$). The parameters $\theta$ must be inferred by the AI.

\vskip0.15cm \emph{Closed-form Solutions for Human Policies in Chainworlds.}
Chainworlds are inspectable to behavioral experts because there is an analytical solution for the optimal value function (all derivations in \cref{appendix: solving}). 
For a chainworld MDP $\chainworld \in \simplechains$, the optimal value function maximizes between the value of a policy that always pursues the goal, $\pi_g(s_n) = 1$, and the value of a policy that always chooses to disengage, $\pi_d(s_n) = 0$, where $s_n$ for $n \in {0, \ldots, N}$ refers to the $n$-th state on the chain. 
The value of goal pursuit is,
\begin{align}
    V^{\pi_g}_\theta (s_n) =  r_g \left (\frac{\gamma p_g }{z}\right)^{N - n} + r_b \left(\frac{1 - (\gamma p_g / z)^{N - n}}{1 - \gamma} \right),
\end{align}
where $z = 1 - \gamma(1 - p_g)$. 
The value of goal pursuit, $V^{\pi_g}_\theta (s_n)$, trades off between the long-term utility of the goal (the $r_g$ term) and the burden one accumulates to get there (the $r_b$ term). 
The value of disengagement is, 
\begin{align} 
\begin{split}
    V^{\pi_d}_\theta (s_n) = r_d \left (\frac{\gamma\ p_{d0}}{v} \right)\left (\frac{p_\ell\ \gamma}{u} \right )^n + (\gamma\ p_d\ r_d + p_\ell\ r_\ell) \left( \frac{1 - (\gamma p_\ell / u)^n}{1 - \gamma(1 - p_d)} \right),
\end{split}
\end{align}
where $v = 1 - \gamma (1 - p_{d0})$ and $u = 1 - \gamma(1 - p_d - p_\ell)$. The first term in the equation (with $r_d$), represents the value of disengagement from state $0$, after having lost all prior progress. The second term represents the value of disengagement after state $0$, which factors in the cost of disengagement $r_d$ and of losing progress $r_\ell$.

These equations allow us to hypothesize about the diverse space of AI actions that will encourage the human towards the goal, such as actions to increase the human's level of motivation (increasing $r_g$) or that highlight the consequences of quitting (decreasing $r_d$).

\subsection{Different humans yield different AI policies}
At this point, we have fully specified an AI MDP as defined in \cref{sec: ai-agent-mdp}, in which the human MDP is a chainworld $\chainworld \in \simplechains$. Solving this AI MDP will yield an optimal AI policy, which is the best intervention plan for a given human with parameters $\theta$. Importantly \cref{fig: ai-policy-differs} demonstrates that personalization is necessary because humans with different $\theta$ require different optimal AI policies. 

\begin{figure}[ht]
    \centering
    \begin{subfigure}{1\linewidth}
         \centering              \includegraphics[width=0.85\linewidth]{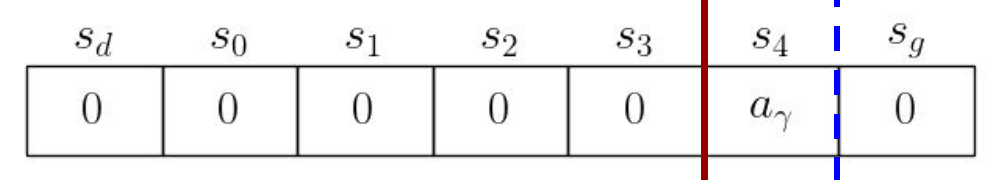}
         \caption{Highly myopic human ($\boldsymbol{\gamma = 0.1}$) with high burden ($\boldsymbol{r_b = -2}$).}
         \label{fig: ai-policy-user-1}
     \end{subfigure}
    \begin{subfigure}{1\linewidth}
         \centering              \includegraphics[width=0.85\linewidth]{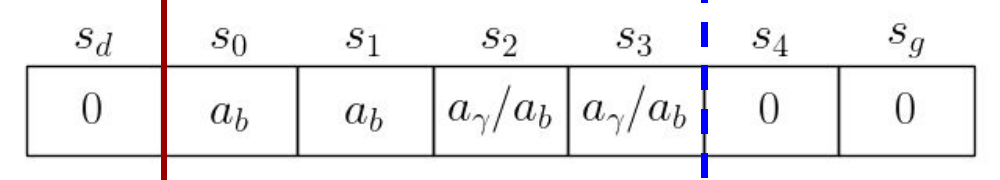}
         \caption{Highly myopic human ($\boldsymbol{\gamma = 0.1}$) with low burden ($\boldsymbol{r_b = -0.3}$).}
         \label{fig: ai-policy-user-2}
     \end{subfigure}%
    \caption{Example of different optimal AI policies for two humans with different chainworld parameters. \textnormal{Each square is a chainworld state. An $a_b$ means AI should select action to reduce $r_b$, while $a_\gamma$ means AI should select action to increase $\gamma$. Red solid and blue dotted lines show start and end of intervention window.}}
    \label{fig: ai-policy-differs}
\end{figure}

\section{Theoretical analysis: When is chainworld good enough?}
\label{sec: theoretical-equivalence-classes}
In this section, we define an \emph{equivalence class} of more complex human MDPs for which an AI agent that plans with the chainworld can still learn the optimal policy.  

\begin{definition}[AI equivalence of human MDPs]
\label{def: ai-equivalence}
We consider two human MDPs $\human{\mdp} \equiv \human{\widehat{\mdp}}$ under state mapping $f: \human{\mathcal{S}} \rightarrow \human{\widehat {\mathcal{S}}}$ and action mapping $g_s: \human{\mathcal{A}} \rightarrow \human{\widehat {\mathcal{A}}}$ if the corresponding optimal AI policies are equal, so that $\ai{\pi}^*\left(\ [s_h, a_h]\ \right) = \ai{\widehat{\pi}}^*\left(\ [f(s_h), g_{s_h}(a_h)] \ \right)$ for all $[s_h, a_h] \in \ai{\mathcal{S}}$. 
\end{definition}

The state mapping $f$ and (state-specific) action mapping $g_s$ let us map from the state and action space of the one MDP to the other. 
In terms of the chainworld, our definition states that if the optimal AI action in the chainworld MDP is the same as the optimal AI action in the true MDP for all states (after applying the mappings), then the two are equivalent. 

Our equivalence in \cref{def: ai-equivalence} is not as strict as the homomorphisms equivalence. Unlike homomorphisms, we \emph{do not} seek human MDPs that have the same rewards and transitions as chainworld. In fact, we do not even seek MDPs that result in the same optimal human policy as chainworld. Instead, we only care that the two human MDPs are similar enough to result in the same \emph{optimal AI policy.}  As a result, we get the largest set of human MDPs that admits simple planning of optimal interventions by the AI agent.


\subsection{Optimal AI policies for chainworld MDPs}
Under \cref{def: ai-equivalence}, the class of MDPs that is equivalent to chainworlds is determined by the space of AI policies that chainworlds can express. In this section, we show that all chainworld MDPs $\chainworld \in \simplechains$ result in AI optimal policies that follow a ``three-window format,'' which we refer to as $\policyclass$. 
Throughout this section, we describe the AI policy in terms of the chainworld states, $s_n$, where $n$ refers to $n$-th state on the chain; even though the \emph{previous} human actions are technically part of the AI state, they do not affect the \emph{best current} action in the AI's optimal policy.

A ``three-window'' AI policy consists of: window 1 (no intervention is effective enough to make human perform the behavior), window 2 (intervention window), and window 3 (human performs behavior without intervention). Two examples are in \cref{fig: ai-policy-differs}.
The size of these windows varies and may even be $0$. For example, if the interventions have no effect ($\Delta_\gamma = 0, \Delta_b = 0$) then the intervention window will be size $0$. 
The three windows are a consequence of how the AI's action affects the human's optimal policy; when the AI agent intervenes on the human, it changes the human's MDP parameters, which in turn, might change the human's optimal policy. 

To succinctly describe the human's optimal policy, we introduce ``human thresholds'' $t$ in \cref{def: threshold}; when the human is in a state past the threshold, their optimal policy is to pursue the goal. A human with a smaller threshold $t$ will pursue the goal from farther away. An effective AI action is one that moves the threshold $t$ to a state \emph{preceding} the human's current state, so that the human chooses to move.

\begin{definition}[Human threshold]
\label{def: threshold}
For a chainworld $\chainworld \in \simplechains$, define $t \in \{0, \ldots, N - 1\}$ as the threshold where $\pi_\theta^*(s_n) = 0$ for  $n \le t$ and $\pi_\theta^*(s_n) = 1$ for $n > t$. 
\end{definition}

Even if the AI agent \emph{can} intervene to prompt the human toward the goal, whether or not the optimal AI \emph{does} intervene depends on the configuration of the AI rewards. If intervening has negligible cost, then the AI agent will intervene as soon as it is able. On the other hand, if there is a high cost, then the AI agent will wait until the human is closer to the goal, to minimize the total number of interventions needed. We define AI threshold $\ai{t}$ below, as the point at which the reward of reaching the goal outweighs the cost of interventions required to reach it:
\begin{definition}[AI threshold]
\label{def: ai-threshold}
For a human chainworld $\chainworld \in \simplechains$ and AI MDP $\ai{\mdp}$, define AI threshold $\ai{t} \in \{0, \ldots, N-1\}$ as the chainworld state in which the value of the goal is greater than the value of disengagement. For states $s_n$ where $n > \ai{t}$, the AI values are $\ai{V}^{\pi_g}(s_n) > \ai{V}^{\pi_d}(s_n)$, and for states where $n \le \ai{t}$, the AI values are $\ai{V}^{\pi_g}(s_n) \le \ai{V}^{\pi_d}$. 

\end{definition}

\noindent The human and AI thresholds define the intervention windows for the AI policy in \cref{thm: chainworld-policies}. 

\begin{theorem}[Chainworld AI policies]
\label{thm: chainworld-policies}
Suppose we are given: 
\begin{itemize}
    \item An AI MDP $\ai{\mdp} = \langle \ai{\mathcal{S}}, \ai{\mathcal{A}}, \ai{T}, \ai{R}, \ai{\gamma} \rangle$, where the actions are to do nothing ($\ai{a} = 0$), intervene on the discount ($\ai{a} = a_\gamma$), or to intervene on burden ($\ai{a} = a_b$)
    \item A human MDP $\chainworld \in \simplechains$, which results in human thresholds $\human{t}^0$, $\human{t}^\gamma$, and $\human{t}^b$ under AI actions $0$, $a_\gamma$, and $a_b$, respectively 
\end{itemize}

\noindent Let $\tmin = \min\left \{\human{t}^0, \human{t}^\gamma, \human{t}^b\right \}$ denote the earliest human threshold as a result of \emph{any AI action}. Let $\ai{t}$ denote the AI intervention threshold, as in \cref{def: ai-threshold}. 
Then, the optimal AI policy is,
\begin{equation}
\label{eq: chainworld-ai-policy}
    \ai{\pi}^*(s_n) = \begin{cases}
        0, & n \le \tmin\\
        0, & \tmin < n \le \ai{t}\\
        a_\gamma, & \max\{\ai{t}, \human{t}^\gamma\} < n \le \human{t}^0\\
        a_b, &\max\{\ai{t}, \human{t}^b\} < n \le \human{t}^0\\
        0, & n > \human{t}^0,
\end{cases}
\end{equation}
and $\ai{\pi}^*$ belongs to the three-window policy class, $\policyclass$. 
\end{theorem}
The proof is in \cref{appendix: proof-ai-policy-chainworld}. 
Note that if both $a_b$ and $a_\gamma$ are valid options in the intervention window (when $\ai{t} < n \le \human{t}^0$), then the AI agent will prefer the less expensive intervention. 
Theorem \ref{thm: chainworld-policies} shows that every chainworld results in an optimal AI policy belonging to $\policyclass$. Theorem \ref{thm: window-based-equivalence} shows the reverse; for any human MDP whose corresponding AI policy is $\ai{\pi} \in \policyclass$, there exists a chainworld MDP whose AI policy is also $\ai{\pi}$.

\begin{theorem}[Chainworld equivalence class]
    \label{thm: window-based-equivalence}
    If human MDP $\human{\mathcal{M}}$ has corresponding AI policy $\ai{\pi} \in \policyclass$, then $\exists \theta$ for $\chainworld \in \simplechains$ such that $\chainworld \equiv \human{\mdp}$. 
\end{theorem}

\noindent Proof in \cref{appendix: proof-chainworld-equivalence-class}.
Theorem \ref{thm: window-based-equivalence} means that \emph{any human MDP} that results in a three-window AI policy---that is, consists of three regions: impossible to help, can be helped by the AI, and does not need help--- belongs to the chainworld equivalence class. 
In \cref{sec: empirical}, we will show that the AI agent can plan interventions using chainworld as a substitute for another human MDP in the same class, without any loss in performance.

\subsection{Realistic human models that are equivalent to chainworld}
\label{sec: theoretically-equivalent-worlds}
Ultimately, we care that the chainworld equivalence class contains \emph{realistic} models of humans that align with the behavioral literature. In this section, we provide examples of human MDPs that capture a meaningful behavior not covered by chainworlds, yet whose optimal AI policy is still in the equivalence class $\policyclass$. 

\vskip0.15cm \noindent \emph{\textbf{Monotonic chainworlds.}}
In monotonic chainworlds, the closer one gets to the goal, the higher the relative value of pursuing it. 

\begin{definition}[Monotonic chainworlds]
    \label{def: monotonic-chains}
    For a monotonic chainworld $\mdp$, the value of goal-pursuit increases closer to the goal: $\vgoal(s_n) - \vdis(s_n)$ $\le \vgoal(s_{n+1}) - \vdis(s_{n+1})$ for all states $n = 1, \ldots, N-1$. 
\end{definition}
For example, consider chainworlds in which the probability of disengagement $p_d$ decreases the closer the agent is to the goal (the human is less likely the quit the closer they are to recovery). Monotonic chainworlds relate to the goal-gradient hypothesis, which states that motivation to reach a goal increases with proximity \citep{mutter2014GoalGradient}.
In \cref{appendix: monotonic-proof}, we prove that all monotonic chainworlds are AI equivalent to our chainworld.

\vskip0.15cm \noindent \emph{\textbf{Progress worlds.}}  Progress worlds, while potentially multi-dimensional, have a one-dimensional notion of progress.

\begin{definition}[Progress worlds]
\label{def: progress-worlds}
Suppose $\mdp$ is a $D$ dimensional, path-connected graph with an absorbing goal state $s_g$, an absorbing disengagement state $s_d$, and actions that allow movement between states on the graph. 
Let $d(s, s')$ denote the shortest graph distance from $s$ to $s'$. $\mdp$ is a progress world if $d(s, s_d) = d(s', s_d)$ and $d(s, s_g) = d(s', s_g)$ for all pairs of $s, s' \in \mathcal{S}$. 
\end{definition}

In our PT example, ``progress'' may depend on a combination of metrics such as joint strength, the ability to perform daily tasks, and so on. We show in \cref{appendix: progress-proof} that worlds in which states can be mapped to a one-dimensional distance are equivalent to our chainworlds. This type of equivalence is simple yet useful, as it lets us reduce high-dimensional worlds to a single dimension of interest. 
Definition \ref{def: progress-worlds} restricts us to graphs 
in which all shortest paths between the disengagement and goal state are of the same length. Intuitively, this means that a single chainworld can represent all paths (and therefore, the entire world). Though not all graphs are progress worlds, in our empirical experiments, we test the chainworld AI's robustness to graphs that break this definition.

\vskip0.15cm \noindent \emph{\textbf{Multi-chain worlds.}} In multi-chain worlds, there is a principle dimension that corresponds to progress toward the goal (as in our simple chainworld) but there may be several additional dimensions associated with different ways of dropping out.  


\begin{definition}[Multi-chain worlds]
    A multi-chain world $\mdp$ consists of $C$ chains, each of length $N_c$.     
    The first chain, $c = 0$, is the \emph{goal chain}; when the human reaches the end of this chain, they have reached the goal. The remaining chains, $s_1, \ldots, s_{C-1}$, are disengagement chains; when the human reaches the end of \emph{any} of these chains, they disengage. When $a = 1$, the human moves along the goal chain with probability $p_0$ while staying still in the disengagement chains. When $a = 0$, the human stays still in the goal chain and (independently) moves along each of the $c$ disengagement chains with probability $p_c$. 
\end{definition}

In our PT example, the principle chain might still correspond to the overall strength of the joint as a measure of progress toward recovery.  Additional chains, corresponding to the level of motivation, level of pain, etc., may all represent mechanisms that cause disengagement.  This form of multi-chain reflects how disengagement is described in the behavioral literature (e.g. \citep{moshe2022predictorsDropoutBack, moroshko2011predictorsDropoutWeight}). In \cref{appendix: multi-chain-case-A} we show equivalence to multi-chain worlds whose disengagement chains are of length $2$, which corresponds to real-world situations in which one of many factors can abruptly trigger disengagement at any point (e.g. the PT patient is re-injured). 

\vskip0.15cm \noindent \emph{\textbf{Negative effect worlds}.} These are chainworlds in which the AI intervention has the opposite intended effect on the human. 

\begin{definition}[Negative effect worlds]
A negative effect world $\mdp$ is defined exactly as the chainworld, except that $\Delta_\gamma < 0$ (AI intervention on discount $\human{\gamma}$ \emph{decreases} it) or $\Delta_b > 0$ (AI intervention on burden $r_b$ \emph{increases} it). 
    
\end{definition}

The efficacy of a behavioral intervention is known to vary by individual (e.g. \citep{bryan2021behaviouralHetero}). In  \cref{appendix: negative-ai-effect-proof}, we prove that negative effect worlds result in AI policies that correspond to chainworlds where the intervention has \emph{no effect} (i.e. $\Delta_\gamma = 0$ and $\Delta_b = 0$). 

\section{Empirical Analysis: Testing Robustness of Chainworld}
\label{sec: empirical}
We test how AI planning using chainworld benefits performance, especially as we remove our assumptions and make the \emph{true} human model dissimilar to chainworld. 

\subsection{Setup}
All experiments are over $200$ trials of $15$ episodes each, and each trial corresponds to a human whose MDP parameters $\theta$ are sampled.
Not all settings of $\theta$ correspond to individuals that can reach their goal---for example, consider a human whose burden is so high that no AI intervention can make them act.  Here, we report results for the subset of sampled humans that can reach the goal under the \emph{oracle AI policy}. Doing so preserves the relative ordering of method performances and reduces noise; in \cref{fig: helpless-keep-vs-remove} we give an example of results that include individuals who never reach the goal.

\vskip0.15cm \emph{Baselines.}
Our baselines are ways to learn the AI policy online. Using data $\ai{\mathcal{D}} = \{(\ai{s}, \ai{a}, \ai{s'}, \ai{r})\}$,  the \emph{\textbf{model-free}} approach directly estimates $\ai{Q}^*$ via Q-learning. 
The \emph{\textbf{model-based}} method estimates $\ai{T}$ using the observed transitions and then solves for $\ai{\pi}^*$ with certainty equivalence. 
Both approaches bypass the need for explicitly solving for a human policy. 
The \emph{\textbf{always $\boldsymbol{\gamma}$}} and \emph{\textbf{always $\boldsymbol{B}$}} are ``no personalization,'' in which the AI policy is to always intervene on $\gamma$ and $B$, respectively. Our method, \emph{\textbf{chainworld}}, estimates the parameters $\theta$ from $\ai{\mathcal{D}}$.

\subsection{Results under \emph{no model misspecification}}

\textbf{In perfect conditions, the AI agent can use chainworld to reach oracle-level performance in the fewest episodes.} When the true human matches our inductive bias, i.e. both are chainworlds, we achieve the fastest personalization in \cref{fig: res-perfect-conditions}. In contrast, model-free requires hundreds of episodes before it learns policies that are better than random (which we demonstrate in \cref{fig: model-free} of the appendix).

\begin{figure}[ht]
    \centering
    \includegraphics[width=1\linewidth]{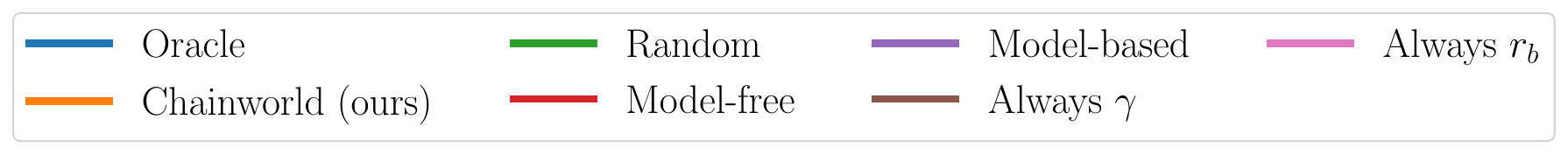}
    \includegraphics[width=0.7\linewidth]{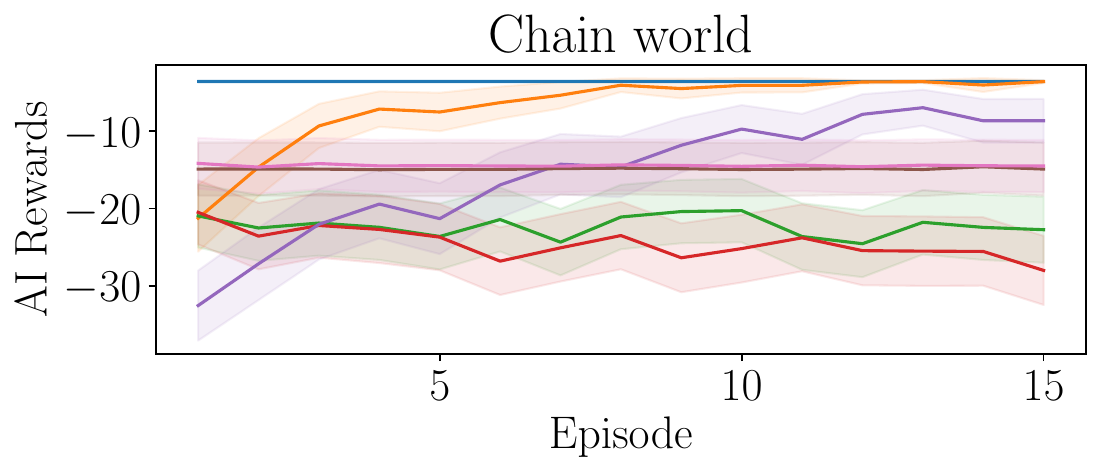}
    \caption{When the true human model is a chainworld, our method rapidly personalizes. \textnormal{Plot is AI rewards (y-axis) over multiple episodes (x-axis). Lines in upper-left personalize quicker.}}
    \label{fig: res-perfect-conditions}
\end{figure} 

\begin{figure*}[ht]
    \centering
    \includegraphics[width=0.7\linewidth]{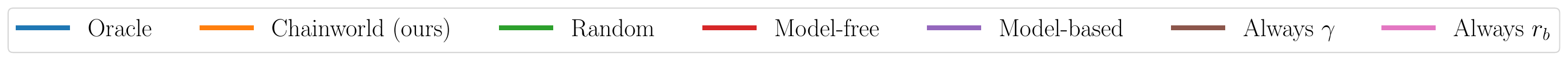}
    \begin{subfigure}{0.15\linewidth}
         \centering              \includegraphics[width=0.95\linewidth]{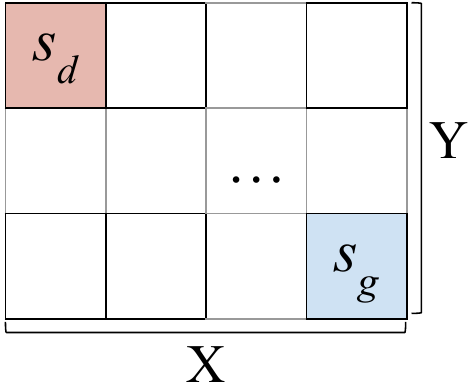}
         \caption{}
         \label{fig: big-small-world}
     \end{subfigure}%
    \begin{subfigure}{0.24\linewidth}
         \centering              \includegraphics[width=0.9\linewidth]{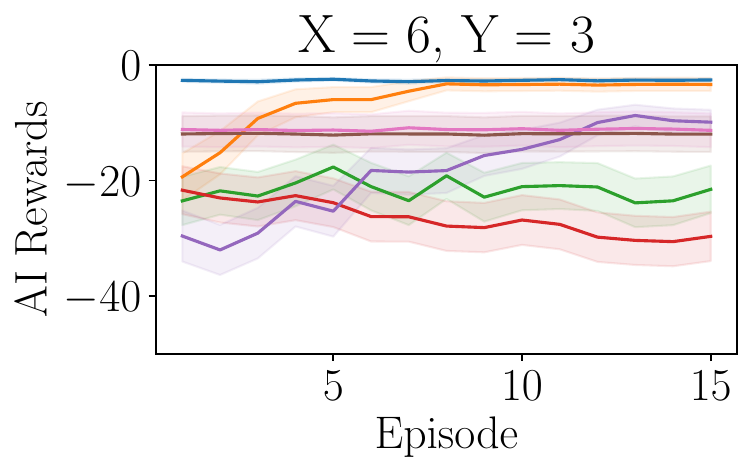}
         \caption{18 total states}
     \end{subfigure}%
     \begin{subfigure}{0.24\linewidth}
         \centering              \includegraphics[width=0.9\linewidth]{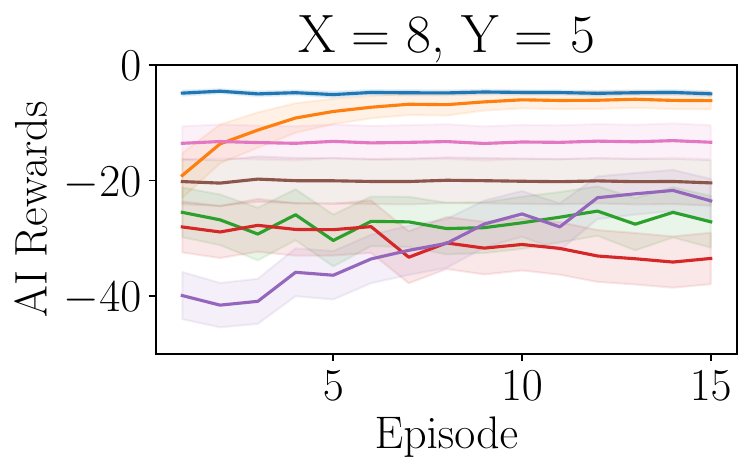}
         \caption{40 total states}
     \end{subfigure}%
    \begin{subfigure}{0.24\linewidth}
         \centering              \includegraphics[width=0.9\linewidth]{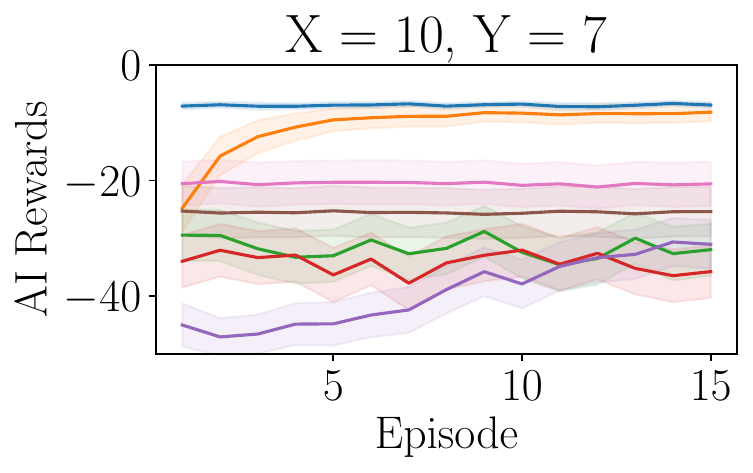}
         \caption{70 total states}
     \end{subfigure}
    \caption{Chainworld scales to large gridworlds. \textnormal{Example gridworld on left. Going right, the grid's width (X) and height (Y) increases.}}
    \label{fig: exact-mapping-invariance}
\end{figure*} 
\begin{table*}[ht]
\centering
\begin{tabular}{l |l| l l | l l}
Assumption & Equiv? & \multicolumn{2}{c}{Low misspecification}  & \multicolumn{2}{c}{High misspecification} \\ \toprule 
& & Chainworld (ours) & Top baseline & Chainworld (ours) & Top baseline\\ \toprule 
Noise in burden $r_b$ 
    & No  
    & $\boldsymbol{-14.47 \pm 3.63}$ 
    & $\boldsymbol{-14.43 \pm 3.63}$ 
    & $\boldsymbol{-35.96 \pm 3.36}$ 
    & $\boldsymbol{-33.43 \pm 3.4}$\\
Noise in utility of goal $r_g$ 
    & No  
    & $\boldsymbol{-5.53 \pm 1.71}$ 
    & $-14.76 \pm 3.38$ 
    & $\boldsymbol{-6.9 \pm 2.22}$ 
    & $-14.66 \pm 3.34$\\
Noise in utility of progress loss $r_\ell$          
    & No  
    & $\boldsymbol{-5.97 \pm 1.94}$ 
    & $-14.78 \pm 3.39$  
    & $\boldsymbol{-11.01 \pm 3.29}$ 
    & $\boldsymbol{-15.43 \pm 3.54}$\\
Noise in utility of disen. $r_d$               
    & No  
    & $\boldsymbol{-8.08 \pm 2.58}$
    & $-15.18 \pm 3.44$ 
    & $\boldsymbol{-13.38 \pm 3.54}$ 
    & $\boldsymbol{-14.63 \pm 3.41}$ \\
Noise in prob. of disen. $p_d$ 
    & No  
    & $\boldsymbol{-5.03 \pm 1.46}$ 
    & $-14.78 \pm 3.39$ 
    & $\boldsymbol{-6.41 \pm 2.45}$ 
    &$\boldsymbol{-12.13 \pm 4.05}$\\
Noise in prob. of disen. at state 0, $p_{d0}$           
    & No  
    & $\boldsymbol{-5.8 \pm 1.86}$ 
    & $-14.81 \pm 3.4$   
    & $\boldsymbol{-5.83 \pm 1.86}$ 
    & $-14.36 \pm 3.3$ \\      
Noise in prob. of losing progress $p_\ell$
    & No  
    & $\boldsymbol{-5.05 \pm 1.51}$ 
    & $-14.78 \pm 3.39$ 
    & $\boldsymbol{-5.19 \pm 1.81}$ 
    & $-13.38 \pm 4.13$ \\
Noise in prob. of making progress $p_g$
    & No  
    & $\boldsymbol{-5.82 \pm 1.77}$ 
    & $-15.24 \pm 3.49$ 
    & $\boldsymbol{-19.38 \pm 4.34}$ 
    & $\boldsymbol{-17.85 \pm 3.72}$ \\
Noise in discount $\human{\gamma}$ 
    & No  
    & $\boldsymbol{-7.75 \pm 2.42}$ 
    & $-15.83 \pm 3.56$ 
    & $\boldsymbol{-20.7 \pm 4.03}$ 
    & $\boldsymbol{-21.19 \pm 3.93}$ \\
Params. fixed across states                              
    & Yes 
    &   ---  
    &  ---    \\
Mapping many dimensions to chainworld                  
    & Yes 
    &  ---    
    & ---   \\
Wrong distance to goal in mapping                           
    & No  
    & $\boldsymbol{-21.18 \pm 3.84}$
    & $\boldsymbol{-15.62 \pm 3.15}$
    & $-35.8 \pm 3.8$ 
    & $\boldsymbol{-24.52 \pm 3.3}$\\
Wrong distance to disengagement in mapping                  
    & No  
    & $\boldsymbol{-10.11 \pm 2.44}$ 
    & $\boldsymbol{-15.62 \pm 2.44}$   
    & $\boldsymbol{-27.27 \pm 3.86}$ 
    & $\boldsymbol{-24.52 \pm 3.3}$ \\
Diseng. from multiple factors                        
    & Yes     
    & ---    
    & ---   \\
Human selects actions non-optimally                  
    & No  
    & $\boldsymbol{-7.23 \pm 2.27}$ 
    & $-16.01 \pm 4.01$ 
    & $\boldsymbol{-24.27 \pm 3.85}$ 
    & $\boldsymbol{-23.39 \pm 3.68}$ \\
AI intervention has negative effect    
    & Yes 
    &  
    &                   
\end{tabular}
\caption{
Reward earned by the AI in episode six. Each row is an assumption violated by the environment. 
Chainworld is better than or within 95\% confidence interval of the top-performing baseline (out of five total baselines) in all but one setting. \textnormal{Conditions marked with ``yes'' in the ``Equivalence?'' column were shown in \cref{sec: theoretically-equivalent-worlds} to preserve theoretical equivalence under misspecification.}}
\label{tab: assumptions}
\end{table*}

\textbf{Our method's performance scales to high-dimensional human models equivalent to the chainworld.}
In the prior theoretical section, we provided examples of human MDPs that reduce to the chainworld. The gridworld in \cref{fig: big-small-world} is one such world since it is a type of distance world. In \cref{fig: exact-mapping-invariance}, our method still personalizes the fastest in increasingly large state spaces, because the number of chainworld parameters is invariant to the size of the gridworld. On the other hand, model-based degrades; it is \emph{worse} than the personalization-free baselines and the same as random baselines, even after $15$ episodes. This is because the transition matrix that model-based must estimate scales with the size of the gridworld. Model-free approaches are even more inefficient in the 2-D setting than in the 1-D chainworld.

\subsection{Robustness results under \emph{model misspecification}}
\label{sec: robustness-experiments}
In true frictionful settings, the AI agent will encounter humans that are more sophisticated than the chainworld. Our remaining experiments in \cref{tab: assumptions} test if AI performance is robust to misspecification when we remove our assumptions about humans. 
In \cref{sec: theoretically-equivalent-worlds}, we theoretically showed that a subset of these assumptions can be removed without affecting the AI. The remaining assumptions we test empirically, and we show our method is more robust to increasing levels of misspecification than baselines. The definition of ``low'' vs. ``high'' misspecification is specific to the experiment.

\begin{figure*}[ht] 
    \centering
    \includegraphics[width=0.8\linewidth]{figures/res-legend-1-column.pdf}
    \begin{subfigure}{0.3\linewidth}
         \centering              \includegraphics[width=0.8\linewidth]{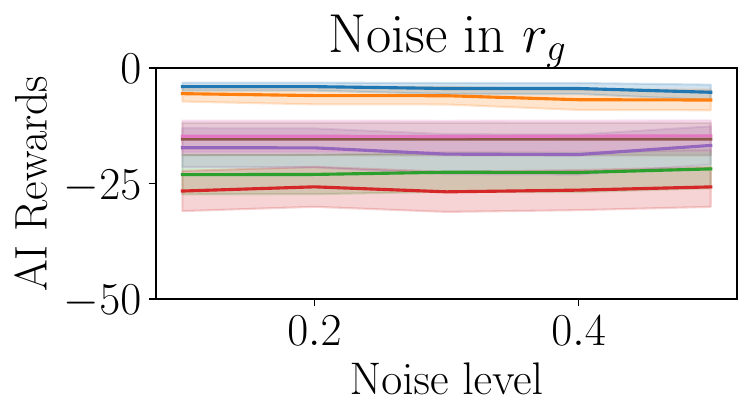}
         \caption{Chainworld robust to low and high mis.}
         \label{fig: noisy-r-g}
     \end{subfigure} \hfill
    \begin{subfigure}{0.3\linewidth}
         \centering              \includegraphics[width=0.8\linewidth]{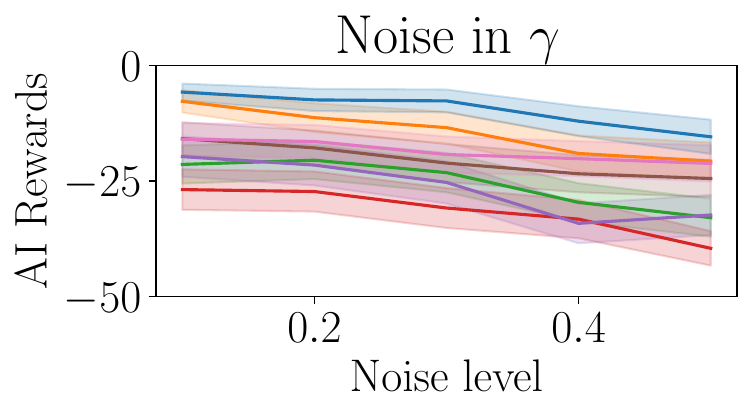}
         \caption{Chainworld is robust to low mis.}
         \label{fig: noisy-gamma}
    \end{subfigure}  \hfill
    \begin{subfigure}{0.3\linewidth}
         \centering              
         \includegraphics[width=0.8\linewidth]{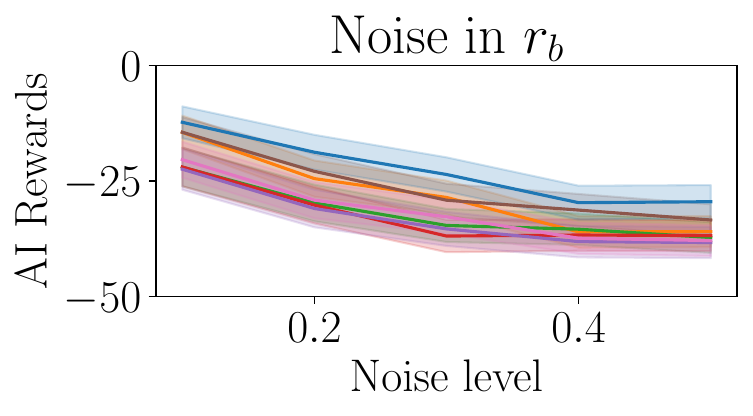}
         \caption{Environment challenges all methods. }
         \label{fig: noisy-r-b}
    \end{subfigure}%
    \caption{Examples of robustness experiments. \textnormal{Chainworld is robust to all levels of misspecification \cref{fig: noisy-r-g}, robust to low levels of misspecification with maintainence at high levels \cref{fig: noisy-gamma}, and all methods, including oracle, struggle to perform well in \cref{fig: noisy-r-b}. Details and plots for all environments in \cref{appendix: environment-descriptions} and  \cref{appendix: full-results}, respectively.}}
    \label{fig: noisy-parameters}
\end{figure*} 

\vskip0.15cm \emph{Experiment on noise in chainworld parameters.}
In this experiment, we test AI performance when the true human model is a chainworld whose parameters vary each timestep due to noise. This mimics situations in which unobservable factors, such as mood, affect parameters, such as burden $r_b$. We vary each parameter in isolation. Our comparison must account for the domains of different parameters, since $\human{\gamma} \in (0, 1)$ while rewards such as $r_b \in \mathbb{R}$. At each timestep, the parameter of interest $x$ is sampled uniformly from $x \sim \text{Uniform}(\bar x -  \epsilon c, \bar x +  \epsilon c)$, where $\bar x$ is the mean parameter value for that individual and the noise level is determined by the parameter range $c$ and the error level $\epsilon \in [0, 1]$. We set parameter range $c = 5$ for reward parameters and to $c = 1$ for transition parameters and $\human{\gamma}$. We define low misspecification as is $\epsilon = 0.1$ and high misspecification as $\epsilon = 0.5$. 

\textit{Experiment on action selection.}
Instead of selecting actions via the optimal policy, humans in this experiment select actions according to softmax policy, $\human{\pi}(a | s) \propto \exp\{\human{Q}(s, a) / \epsilon \}$, where $\epsilon$ is the level of noise. 
We define low misspecification as $\epsilon = 0.05$ and high misspecification as $\epsilon = 0.2$. 

\textit{Experiment on misspecified mapping.}
This experiment tests robustness to differences in \emph{model structure}. The true human is no longer a chainworld, but a gridworld as in \cref{fig: big-small-world}. However, the gridworld in this experiment is no longer equivalent to our chainworld because the goal state $s_g$ is not in the lower-right corner at $[X, 0]$. In fact, the equivalence degrades as $\epsilon$ increases for $[X, \epsilon]$. We set the grid dimensions as $X = 8, Y = 5$ and define low misspecification as $\epsilon = 1$ and high misspecification as $\epsilon = 4$. 

\textbf{We are robust to low levels of misspecification.}
In \cref{tab: assumptions}, our method outperforms baselines in \emph{9 out of 12 robustness experiments} under low levels of misspecification. With high misspecification, when our method is not the best, it falls within two standard errors of the next-best method in all but one condition. 

\textbf{Some humans are difficult to intervene on overall, even for the oracle.}
All methods, including the oracle, earn fewer rewards when the burden parameter $r_b$ is noisy (see \cref{fig: noisy-r-b}). This indicates that it is particularly important to model $r_b$ well in frictionful tasks. For example, we may ensure that features predictive of burden, such as mood, are part of the AI's state space, so that we can estimate $r_b$. 

\textbf{To reduce (non-equivalent) human models to the chainworld, it is important that we capture distance to goal well.}
Since chainworld is one-dimensional, it can only represent worlds whose multi-dimensional states can be mapped to one dimension. 
When such a mapping is not possible, we must choose between capturing progress toward goal (e.g. how far does the patient feel from shoulder recovery?) or distance from disengagement (e.g. how close to giving up does the patient feel?). 
Under the ``wrong distance to goal / disengagement mapping'' condition in \cref{tab: assumptions}, we show that capturing progress toward goal matters more. This implies that chainworlds can still be applied to settings where we cannot model all factors that lead to disengagement, so long as we have an accurate way of measuring the human's progress to the goal. 

\section{Conclusion and Future Work}
In this paper, we introduced Behavior Model Reinforcement Learning (BMRL), a framework for AI agents to intervene on human agents performing frictionful tasks. We proposed a simple model of the human agent-- the chainworld-- that the AI agent can use to rapidly personalize. Using a novel definition of equivalence between human models in BMRL, we defined a theoretical class of human MDPs that chainworld can generalize to and showed that this class contains behaviorally meaningful models of humans. 

Our chainworlds are not psychologically verified human models; in future work, we will formally test the modeling assumptions with user studies. 
To apply BMRL in the real world, we must also consider the ethics of AI intervention. Mainly, we must ensure the AI does not manipulate the human. BMRL should only be used for people who already have a long-term goal, and the AI must not change that goal. Subgroup fairness should also be considered during learning and personalization.

Although we aimed to be comprehensive in testing chainworld's robustness, there were limitations to our approach. First, we 
did not evaluate how multiple misspecifications may compound to affect AI performance. Second, our analyses assumed that the mapping from the true MDP to the chainworld is given. In some applications this is reasonable; in PT, a domain expert is likely to know which factors contribute to a patient's perception of ``progress'' (the mapping from a distance world to a chainworld). In other cases, one will need to learn this mapping in conjunction with the chainworld parameters. 

We made several simplifying assumptions on the human + AI interactions. We avoided a POMDP formulation by assuming that there are no delayed effects of the AI's actions on the human MDP. However, habituation (reduced effectiveness of repeated interventions) is a well-studied phenomenon in digital interventions (e.g. \citep{gotzian2023modelingDecreasingEffect}). Furthermore, we avoided multi-agent RL by assuming that the human is \emph{not learning}, and instead, is solving an (implicitly) known MDP at each time step. We did not consider suboptimality of the human agent's planning, such as (small) fixed-horizon planning.
Finally (and excitingly), BMRL is adaptable to more diverse AI interventions. Our paper focused exclusively on interventions to the human's discount and reward. In many applications, the human's perception of state, actions, and transitions may also be impaired. Similarly, behavioral interventions on perceptions of state, actions, and transitions exist and could be incorporated into our framework. 

\section{Acknowledgements}
This material is based upon work supported by the National Science Foundation under Grant No. IIS-2107391 and the National Institute of Biomedical Imaging and Bioengineering of the National Institutes of Health under OD P41EB028242.  Any opinions, findings, and conclusions or recommendations expressed in this material are those of the author(s) and do not necessarily reflect the views of the National Science Foundation.
ES’s work was supported by a gift fund
from Benshi.ai and the National Science Foundation Graduate Research Fellowship Program under Grant No. DGE2140743.

\bibliographystyle{ACM-Reference-Format} 
\bibliography{aamas}


\begin{thebibliography}{49}


\ifx \showCODEN    \undefined \def \showCODEN     #1{\unskip}     \fi
\ifx \showDOI      \undefined \def \showDOI       #1{#1}\fi
\ifx \showISBNx    \undefined \def \showISBNx     #1{\unskip}     \fi
\ifx \showISBNxiii \undefined \def \showISBNxiii  #1{\unskip}     \fi
\ifx \showISSN     \undefined \def \showISSN      #1{\unskip}     \fi
\ifx \showLCCN     \undefined \def \showLCCN      #1{\unskip}     \fi
\ifx \shownote     \undefined \def \shownote      #1{#1}          \fi
\ifx \showarticletitle \undefined \def \showarticletitle #1{#1}   \fi
\ifx \showURL      \undefined \def \showURL       {\relax}        \fi
\providecommand\bibfield[2]{#2}
\providecommand\bibinfo[2]{#2}
\providecommand\natexlab[1]{#1}
\providecommand\showeprint[2][]{arXiv:#2}

\bibitem[Aswani et~al\mbox{.}(2019)]%
        {aswani2019behavioralWeightLoss}
\bibfield{author}{\bibinfo{person}{Anil Aswani}, \bibinfo{person}{Philip
  Kaminsky}, \bibinfo{person}{Yonatan Mintz}, \bibinfo{person}{Elena Flowers},
  {and} \bibinfo{person}{Yoshimi Fukuoka}.} \bibinfo{year}{2019}\natexlab{}.
\newblock \showarticletitle{Behavioral modeling in weight loss interventions}.
\newblock \bibinfo{journal}{\emph{European journal of operational research}}
  \bibinfo{volume}{272}, \bibinfo{number}{3} (\bibinfo{year}{2019}),
  \bibinfo{pages}{1058--1072}.
\newblock


\bibitem[Bandura(1999)]%
        {bandura1999social}
\bibfield{author}{\bibinfo{person}{Albert Bandura}.}
  \bibinfo{year}{1999}\natexlab{}.
\newblock \showarticletitle{Social cognitive theory: An agentic perspective}.
\newblock \bibinfo{journal}{\emph{Asian journal of social psychology}}
  \bibinfo{volume}{2}, \bibinfo{number}{1} (\bibinfo{year}{1999}),
  \bibinfo{pages}{21--41}.
\newblock


\bibitem[Brown et~al\mbox{.}(2019)]%
        {brown2019extrapolatingTREX}
\bibfield{author}{\bibinfo{person}{Daniel Brown}, \bibinfo{person}{Wonjoon
  Goo}, \bibinfo{person}{Prabhat Nagarajan}, {and} \bibinfo{person}{Scott
  Niekum}.} \bibinfo{year}{2019}\natexlab{}.
\newblock \showarticletitle{Extrapolating beyond suboptimal demonstrations via
  inverse reinforcement learning from observations}. In
  \bibinfo{booktitle}{\emph{International conference on machine learning}}.
  \bibinfo{publisher}{PMLR}, \bibinfo{address}{California USA},
  \bibinfo{pages}{783--792}.
\newblock


\bibitem[Brown and Stein(2022)]%
        {brown2022episodicFutureThinking}
\bibfield{author}{\bibinfo{person}{Jeremiah~Michael Brown} {and}
  \bibinfo{person}{Jeffrey~Scott Stein}.} \bibinfo{year}{2022}\natexlab{}.
\newblock \showarticletitle{Putting prospection into practice: Methodological
  considerations in the use of episodic future thinking to reduce delay
  discounting and maladaptive health behaviors}.
\newblock \bibinfo{journal}{\emph{Frontiers in Public Health}}
  \bibinfo{volume}{10} (\bibinfo{year}{2022}), \bibinfo{pages}{1020171}.
\newblock


\bibitem[Bryan et~al\mbox{.}(2021)]%
        {bryan2021behaviouralHetero}
\bibfield{author}{\bibinfo{person}{Christopher~J Bryan},
  \bibinfo{person}{Elizabeth Tipton}, {and} \bibinfo{person}{David~S Yeager}.}
  \bibinfo{year}{2021}\natexlab{}.
\newblock \showarticletitle{Behavioural science is unlikely to change the world
  without a heterogeneity revolution}.
\newblock \bibinfo{journal}{\emph{Nature human behaviour}} \bibinfo{volume}{5},
  \bibinfo{number}{8} (\bibinfo{year}{2021}), \bibinfo{pages}{980--989}.
\newblock


\bibitem[Chen et~al\mbox{.}(2022)]%
        {chen2022mirror}
\bibfield{author}{\bibinfo{person}{Kaiqi Chen}, \bibinfo{person}{Jeffrey Fong},
  {and} \bibinfo{person}{Harold Soh}.} \bibinfo{year}{2022}\natexlab{}.
\newblock \showarticletitle{Mirror: Differentiable deep social projection for
  assistive human-robot communication}. In \bibinfo{booktitle}{\emph{Robotics:
  Science and Systems}}. \bibinfo{publisher}{Robotics: Science and Systems},
  \bibinfo{address}{New York USA}.
\newblock


\bibitem[Evans et~al\mbox{.}(2016)]%
        {evans2016learningIgnorant}
\bibfield{author}{\bibinfo{person}{Owain Evans}, \bibinfo{person}{Andreas
  Stuhlm{\"u}ller}, {and} \bibinfo{person}{Noah Goodman}.}
  \bibinfo{year}{2016}\natexlab{}.
\newblock \showarticletitle{Learning the preferences of ignorant, inconsistent
  agents}. In \bibinfo{booktitle}{\emph{Proceedings of the AAAI Conference on
  Artificial Intelligence}}, Vol.~\bibinfo{volume}{30}.
  \bibinfo{publisher}{AAAI}, \bibinfo{address}{Arizona USA}.
\newblock


\bibitem[Fanfarelli et~al\mbox{.}(2015)]%
        {fanfarelli2015understandingBadges}
\bibfield{author}{\bibinfo{person}{Joseph Fanfarelli},
  \bibinfo{person}{Stephanie Vie}, {and} \bibinfo{person}{Rudy McDaniel}.}
  \bibinfo{year}{2015}\natexlab{}.
\newblock \showarticletitle{Understanding digital badges through feedback,
  reward, and narrative: a multidisciplinary approach to building better badges
  in social environments}.
\newblock \bibinfo{journal}{\emph{Communication Design Quarterly Review}}
  \bibinfo{volume}{3}, \bibinfo{number}{3} (\bibinfo{year}{2015}),
  \bibinfo{pages}{56--60}.
\newblock


\bibitem[Fedus et~al\mbox{.}(2019)]%
        {fedus2019hyperbolic}
\bibfield{author}{\bibinfo{person}{William Fedus}, \bibinfo{person}{Carles
  Gelada}, \bibinfo{person}{Yoshua Bengio}, \bibinfo{person}{Marc~G.
  Bellemare}, {and} \bibinfo{person}{Hugo Larochelle}.}
  \bibinfo{year}{2019}\natexlab{}.
\newblock \bibinfo{title}{Hyperbolic Discounting and Learning over Multiple
  Horizons}.
\newblock
\newblock
\showeprint[arxiv]{1902.06865}~[stat.ML]


\bibitem[Givan et~al\mbox{.}(2003)]%
        {givan2003equivalenceBisimulation}
\bibfield{author}{\bibinfo{person}{Robert Givan}, \bibinfo{person}{Thomas
  Dean}, {and} \bibinfo{person}{Matthew Greig}.}
  \bibinfo{year}{2003}\natexlab{}.
\newblock \showarticletitle{Equivalence notions and model minimization in
  Markov decision processes}.
\newblock \bibinfo{journal}{\emph{Artificial Intelligence}}
  \bibinfo{volume}{147}, \bibinfo{number}{1-2} (\bibinfo{year}{2003}),
  \bibinfo{pages}{163--223}.
\newblock


\bibitem[Giwa and Lee(2021)]%
        {giwa2021estimationDiscount}
\bibfield{author}{\bibinfo{person}{Babatunde~H Giwa} {and}
  \bibinfo{person}{Chi-Guhn Lee}.} \bibinfo{year}{2021}\natexlab{}.
\newblock \bibinfo{title}{Estimation of Discount Factor in a Model-Based
  Inverse Reinforcement Learning Framework}.
\newblock
\newblock
\urldef\tempurl%
\url{https://hdl.handle.net/1807/125220}
\showURL{%
\tempurl}


\bibitem[Gotzian(2023)]%
        {gotzian2023modelingDecreasingEffect}
\bibfield{author}{\bibinfo{person}{Lisa Gotzian}.}
  \bibinfo{year}{2023}\natexlab{}.
\newblock \bibinfo{title}{Modeling the decreasing intervention effect in
  digital health: a computational model to predict the response for a walking
  intervention}.
\newblock
\newblock
\urldef\tempurl%
\url{https://doi.org/10.31219/osf.io/6v7d5}
\showDOI{\tempurl}


\bibitem[Jarrett et~al\mbox{.}(2021)]%
        {jarrett2021inverseDecisonModeling}
\bibfield{author}{\bibinfo{person}{Daniel Jarrett}, \bibinfo{person}{Alihan
  H{\"u}y{\"u}k}, {and} \bibinfo{person}{Mihaela Van Der~Schaar}.}
  \bibinfo{year}{2021}\natexlab{}.
\newblock \showarticletitle{Inverse decision modeling: Learning interpretable
  representations of behavior}. In \bibinfo{booktitle}{\emph{International
  Conference on Machine Learning}}. PMLR, \bibinfo{publisher}{PMLR},
  \bibinfo{address}{Virtual}, \bibinfo{pages}{4755--4771}.
\newblock


\bibitem[Khanshan et~al\mbox{.}(2023)]%
        {khanshan2023simulatingESM}
\bibfield{author}{\bibinfo{person}{Alireza Khanshan}, \bibinfo{person}{Pieter
  Van~Gorp}, {and} \bibinfo{person}{Panos Markopoulos}.}
  \bibinfo{year}{2023}\natexlab{}.
\newblock \showarticletitle{Simulating Participant Behavior in Experience
  Sampling Method Research}. In \bibinfo{booktitle}{\emph{Extended Abstracts of
  the 2023 CHI Conference on Human Factors in Computing Systems}} (<conf-loc>,
  <city>Hamburg</city>, <country>Germany</country>, </conf-loc>)
  \emph{(\bibinfo{series}{CHI EA '23})}. \bibinfo{publisher}{Association for
  Computing Machinery}, \bibinfo{address}{New York, NY, USA}, Article
  \bibinfo{articleno}{250}, \bibinfo{numpages}{7}~pages.
\newblock
\showISBNx{9781450394222}
\urldef\tempurl%
\url{https://doi.org/10.1145/3544549.3585586}
\showDOI{\tempurl}


\bibitem[Li et~al\mbox{.}(2006)]%
        {li2006towardsUnifiedAbstraction}
\bibfield{author}{\bibinfo{person}{Lihong Li}, \bibinfo{person}{Thomas~J
  Walsh}, {and} \bibinfo{person}{Michael~L Littman}.}
  \bibinfo{year}{2006}\natexlab{}.
\newblock \bibinfo{title}{Towards a unified theory of state abstraction for
  MDPs}.
\newblock
\newblock


\bibitem[Liu et~al\mbox{.}(2019)]%
        {liu2019modelingRisk}
\bibfield{author}{\bibinfo{person}{Quanying Liu}, \bibinfo{person}{Haiyan Wu},
  {and} \bibinfo{person}{Anqi Liu}.} \bibinfo{year}{2019}\natexlab{}.
\newblock \bibinfo{title}{Modeling and Interpreting Real-world Human Risk
  Decision Making with Inverse Reinforcement Learning}.
\newblock
\newblock
\showeprint[arxiv]{1906.05803}~[cs.LG]


\bibitem[Magen et~al\mbox{.}(2008)]%
        {magen2008hiddenZero}
\bibfield{author}{\bibinfo{person}{Eran Magen}, \bibinfo{person}{Carol~S
  Dweck}, {and} \bibinfo{person}{James~J Gross}.}
  \bibinfo{year}{2008}\natexlab{}.
\newblock \showarticletitle{The hidden-zero effect: Representing a single
  choice as an extended sequence reduces impulsive choice}.
\newblock \bibinfo{journal}{\emph{Psychological Science}} \bibinfo{volume}{19},
  \bibinfo{number}{7} (\bibinfo{year}{2008}), \bibinfo{pages}{648--649}.
\newblock


\bibitem[Martin et~al\mbox{.}(2018)]%
        {martin2018ControlFluidSCT}
\bibfield{author}{\bibinfo{person}{Cesar~A Martin}, \bibinfo{person}{Daniel~E
  Rivera}, \bibinfo{person}{Eric~B Hekler}, \bibinfo{person}{William~T Riley},
  \bibinfo{person}{Matthew~P Buman}, \bibinfo{person}{Marc~A Adams}, {and}
  \bibinfo{person}{Alicia~B Magann}.} \bibinfo{year}{2018}\natexlab{}.
\newblock \showarticletitle{Development of a control-oriented model of social
  cognitive theory for optimized mHealth behavioral interventions}.
\newblock \bibinfo{journal}{\emph{IEEE Transactions on Control Systems
  Technology}} \bibinfo{volume}{28}, \bibinfo{number}{2}
  (\bibinfo{year}{2018}), \bibinfo{pages}{331--346}.
\newblock


\bibitem[Mintz et~al\mbox{.}(2023)]%
        {mintz2023behavioralAnalytics}
\bibfield{author}{\bibinfo{person}{Yonatan Mintz}, \bibinfo{person}{Anil
  Aswani}, \bibinfo{person}{Philip Kaminsky}, \bibinfo{person}{Elena Flowers},
  {and} \bibinfo{person}{Yoshimi Fukuoka}.} \bibinfo{year}{2023}\natexlab{}.
\newblock \showarticletitle{Behavioral analytics for myopic agents}.
\newblock \bibinfo{journal}{\emph{European Journal of Operational Research}}
  \bibinfo{volume}{310}, \bibinfo{number}{2} (\bibinfo{year}{2023}),
  \bibinfo{pages}{793--811}.
\newblock


\bibitem[Mogles et~al\mbox{.}(2018)]%
        {mogles2018computationalEnergy}
\bibfield{author}{\bibinfo{person}{Nataliya Mogles}, \bibinfo{person}{Julian
  Padget}, \bibinfo{person}{Elizabeth Gabe-Thomas}, \bibinfo{person}{Ian
  Walker}, {and} \bibinfo{person}{JeeHang Lee}.}
  \bibinfo{year}{2018}\natexlab{}.
\newblock \showarticletitle{A computational model for designing energy
  behaviour change interventions}.
\newblock \bibinfo{journal}{\emph{User Modeling and User-Adapted Interaction}}
  \bibinfo{volume}{28} (\bibinfo{year}{2018}), \bibinfo{pages}{1--34}.
\newblock


\bibitem[Moroshko et~al\mbox{.}(2011)]%
        {moroshko2011predictorsDropoutWeight}
\bibfield{author}{\bibinfo{person}{Irena Moroshko}, \bibinfo{person}{Leah
  Brennan}, {and} \bibinfo{person}{Paul O'Brien}.}
  \bibinfo{year}{2011}\natexlab{}.
\newblock \showarticletitle{Predictors of dropout in weight loss interventions:
  a systematic review of the literature}.
\newblock \bibinfo{journal}{\emph{Obesity reviews}} \bibinfo{volume}{12},
  \bibinfo{number}{11} (\bibinfo{year}{2011}), \bibinfo{pages}{912--934}.
\newblock


\bibitem[Moshe et~al\mbox{.}(2022)]%
        {moshe2022predictorsDropoutBack}
\bibfield{author}{\bibinfo{person}{Isaac Moshe}, \bibinfo{person}{Yannik
  Terhorst}, \bibinfo{person}{Sarah Paganini}, \bibinfo{person}{Sandra
  Schlicker}, \bibinfo{person}{Laura Pulkki-R{\aa}back},
  \bibinfo{person}{Harald Baumeister}, \bibinfo{person}{Lasse~B Sander}, {and}
  \bibinfo{person}{David~Daniel Ebert}.} \bibinfo{year}{2022}\natexlab{}.
\newblock \showarticletitle{Predictors of dropout in a digital intervention for
  the prevention and treatment of depression in patients with chronic back
  pain: secondary analysis of two randomized controlled trials}.
\newblock \bibinfo{journal}{\emph{Journal of Medical Internet Research}}
  \bibinfo{volume}{24}, \bibinfo{number}{8} (\bibinfo{year}{2022}),
  \bibinfo{pages}{e38261}.
\newblock


\bibitem[Mutter and Kundisch(2014)]%
        {mutter2014GoalGradient}
\bibfield{author}{\bibinfo{person}{Tobias Mutter} {and} \bibinfo{person}{Dennis
  Kundisch}.} \bibinfo{year}{2014}\natexlab{}.
\newblock \showarticletitle{Behavioral mechanisms prompted by badges: The
  goal-gradient hypothesis}. In \bibinfo{booktitle}{\emph{ICIS 2014
  Proceedings}}, Vol.~\bibinfo{volume}{12}. \bibinfo{publisher}{ICIS},
  \bibinfo{address}{New Zealand}.
\newblock


\bibitem[Niv(2009)]%
        {niv2009Dopamine1}
\bibfield{author}{\bibinfo{person}{Yael Niv}.} \bibinfo{year}{2009}\natexlab{}.
\newblock \showarticletitle{Reinforcement learning in the brain}.
\newblock \bibinfo{journal}{\emph{Journal of Mathematical Psychology}}
  \bibinfo{volume}{53}, \bibinfo{number}{3} (\bibinfo{year}{2009}),
  \bibinfo{pages}{139--154}.
\newblock


\bibitem[Park and Lee(2023)]%
        {park2023understandingDisengagement}
\bibfield{author}{\bibinfo{person}{Joonyoung Park} {and}
  \bibinfo{person}{Uichin Lee}.} \bibinfo{year}{2023}\natexlab{}.
\newblock \showarticletitle{Understanding Disengagement in Just-in-Time Mobile
  Health Interventions}.
\newblock \bibinfo{journal}{\emph{Proceedings of the ACM on Interactive,
  Mobile, Wearable and Ubiquitous Technologies}} \bibinfo{volume}{7},
  \bibinfo{number}{2} (\bibinfo{year}{2023}), \bibinfo{pages}{1--27}.
\newblock


\bibitem[Pirolli(2016)]%
        {pirolli2016computationalACTR}
\bibfield{author}{\bibinfo{person}{Peter Pirolli}.}
  \bibinfo{year}{2016}\natexlab{}.
\newblock \showarticletitle{A computational cognitive model of self-efficacy
  and daily adherence in mHealth}.
\newblock \bibinfo{journal}{\emph{Translational behavioral medicine}}
  \bibinfo{volume}{6}, \bibinfo{number}{4} (\bibinfo{year}{2016}),
  \bibinfo{pages}{496--508}.
\newblock


\bibitem[Ravindran and Barto(2002)]%
        {ravindran2002modelHomomorphisms}
\bibfield{author}{\bibinfo{person}{Balaraman Ravindran} {and}
  \bibinfo{person}{Andrew~G Barto}.} \bibinfo{year}{2002}\natexlab{}.
\newblock \showarticletitle{Model minimization in hierarchical reinforcement
  learning}. In \bibinfo{booktitle}{\emph{Abstraction, Reformulation, and
  Approximation: 5th International Symposium, SARA}},
  Vol.~\bibinfo{volume}{2371}. Springer, \bibinfo{publisher}{Springer, Berlin,
  Heidelberg}, \bibinfo{address}{Canada}, \bibinfo{pages}{196--211}.
\newblock


\bibitem[Ravindran and Barto(2004)]%
        {ravindran2004approximateHomomorphism}
\bibfield{author}{\bibinfo{person}{Balaraman Ravindran} {and}
  \bibinfo{person}{Andrew~G Barto}.} \bibinfo{year}{2004}\natexlab{}.
\newblock \bibinfo{title}{Approximate homomorphisms: A framework for non-exact
  minimization in Markov decision processes}.
\newblock
\newblock


\bibitem[Reddy et~al\mbox{.}(2018)]%
        {reddy2018you}
\bibfield{author}{\bibinfo{person}{Siddharth Reddy}, \bibinfo{person}{Anca~D.
  Dragan}, {and} \bibinfo{person}{Sergey Levine}.}
  \bibinfo{year}{2018}\natexlab{}.
\newblock \showarticletitle{Where do you think you're going? inferring beliefs
  about dynamics from behavior}. In \bibinfo{booktitle}{\emph{Proceedings of
  the 32nd International Conference on Neural Information Processing Systems}}
  (Montr\'{e}al, Canada) \emph{(\bibinfo{series}{NIPS'18})}.
  \bibinfo{publisher}{Curran Associates Inc.}, \bibinfo{address}{Red Hook, NY,
  USA}, \bibinfo{pages}{1461–1472}.
\newblock


\bibitem[Reddy et~al\mbox{.}(2021)]%
        {reddy2021CommunicationState}
\bibfield{author}{\bibinfo{person}{Siddharth Reddy}, \bibinfo{person}{Sergey
  Levine}, {and} \bibinfo{person}{Anca Dragan}.}
  \bibinfo{year}{2021}\natexlab{}.
\newblock \showarticletitle{Assisted perception: optimizing observations to
  communicate state}. In \bibinfo{booktitle}{\emph{Conference on Robot
  Learning}}. PMLR, \bibinfo{publisher}{PMLR}, \bibinfo{address}{London UK},
  \bibinfo{pages}{748--764}.
\newblock


\bibitem[Shah et~al\mbox{.}(2019)]%
        {shah2019feasibility}
\bibfield{author}{\bibinfo{person}{Rohin Shah}, \bibinfo{person}{Noah
  Gundotra}, \bibinfo{person}{Pieter Abbeel}, {and} \bibinfo{person}{Anca
  Dragan}.} \bibinfo{year}{2019}\natexlab{}.
\newblock \showarticletitle{On the feasibility of learning, rather than
  assuming, human biases for reward inference}. In
  \bibinfo{booktitle}{\emph{International Conference on Machine Learning}}.
  PMLR, \bibinfo{publisher}{PMLR}, \bibinfo{address}{California, USA},
  \bibinfo{pages}{5670--5679}.
\newblock


\bibitem[Shteingart and Loewenstein(2014)]%
        {shteingart2014Dopamine2}
\bibfield{author}{\bibinfo{person}{Hanan Shteingart} {and}
  \bibinfo{person}{Yonatan Loewenstein}.} \bibinfo{year}{2014}\natexlab{}.
\newblock \showarticletitle{Reinforcement learning and human behavior}.
\newblock \bibinfo{journal}{\emph{Current opinion in neurobiology}}
  \bibinfo{volume}{25} (\bibinfo{year}{2014}), \bibinfo{pages}{93--98}.
\newblock


\bibitem[Story et~al\mbox{.}(2014)]%
        {story2014doesTemporal}
\bibfield{author}{\bibinfo{person}{Giles~W Story}, \bibinfo{person}{Ivo Vlaev},
  \bibinfo{person}{Ben Seymour}, \bibinfo{person}{Ara Darzi}, {and}
  \bibinfo{person}{Raymond~J Dolan}.} \bibinfo{year}{2014}\natexlab{}.
\newblock \showarticletitle{Does temporal discounting explain unhealthy
  behavior? A systematic review and reinforcement learning perspective}.
\newblock \bibinfo{journal}{\emph{Frontiers in behavioral neuroscience}}
  \bibinfo{volume}{8} (\bibinfo{year}{2014}), \bibinfo{pages}{76}.
\newblock


\bibitem[Tabatabaei et~al\mbox{.}(2018)]%
        {tabatabaei2018narrowing}
\bibfield{author}{\bibinfo{person}{Seyed~Amin Tabatabaei},
  \bibinfo{person}{Mark Hoogendoorn}, {and} \bibinfo{person}{Aart van
  Halteren}.} \bibinfo{year}{2018}\natexlab{}.
\newblock \showarticletitle{Narrowing reinforcement learning: Overcoming the
  cold start problem for personalized health interventions}. In
  \bibinfo{booktitle}{\emph{PRIMA 2018: Principles and Practice of Multi-Agent
  Systems: 21st International Conference}}. Springer,
  \bibinfo{publisher}{Springer}, \bibinfo{address}{Tokyo Japan},
  \bibinfo{pages}{312--327}.
\newblock


\bibitem[Tabrez et~al\mbox{.}(2019)]%
        {tabrez2019explanationReward}
\bibfield{author}{\bibinfo{person}{Aaquib Tabrez}, \bibinfo{person}{Shivendra
  Agrawal}, {and} \bibinfo{person}{Bradley Hayes}.}
  \bibinfo{year}{2019}\natexlab{}.
\newblock \showarticletitle{Explanation-Based Reward Coaching to Improve Human
  Performance via Reinforcement Learning}. In
  \bibinfo{booktitle}{\emph{ACM/IEEE International Conference on Human-Robot
  Interaction (HRI)}}. \bibinfo{publisher}{IEEE}, \bibinfo{address}{Korea},
  \bibinfo{pages}{249--257}.
\newblock
\urldef\tempurl%
\url{https://doi.org/10.1109/HRI.2019.8673104}
\showDOI{\tempurl}


\bibitem[Taylor et~al\mbox{.}(2021a)]%
        {taylor2021awarenessEatingBehavior}
\bibfield{author}{\bibinfo{person}{V{\'e}ronique~A Taylor},
  \bibinfo{person}{Isabelle Moseley}, \bibinfo{person}{Shufang Sun},
  \bibinfo{person}{Ryan Smith}, \bibinfo{person}{Alexandra Roy},
  \bibinfo{person}{Vera~U Ludwig}, {and} \bibinfo{person}{Judson~A Brewer}.}
  \bibinfo{year}{2021}\natexlab{a}.
\newblock \showarticletitle{Awareness drives changes in reward value which
  predict eating behavior change: Probing reinforcement learning using
  experience sampling from mobile mindfulness training for maladaptive eating}.
\newblock \bibinfo{journal}{\emph{Journal of behavioral addictions}}
  \bibinfo{volume}{10}, \bibinfo{number}{3} (\bibinfo{year}{2021}),
  \bibinfo{pages}{482--497}.
\newblock


\bibitem[Taylor et~al\mbox{.}(2021b)]%
        {taylor2021awarenessSmoking}
\bibfield{author}{\bibinfo{person}{V{\'e}ronique~A Taylor},
  \bibinfo{person}{Isabelle Moseley}, \bibinfo{person}{Shufang Sun},
  \bibinfo{person}{Ryan Smith}, \bibinfo{person}{Alexandra Roy},
  \bibinfo{person}{Vera~U Ludwig}, {and} \bibinfo{person}{Judson~A Brewer}.}
  \bibinfo{year}{2021}\natexlab{b}.
\newblock \showarticletitle{Awareness drives changes in reward value which
  predict eating behavior change: Probing reinforcement learning using
  experience sampling from mobile mindfulness training for maladaptive eating}.
\newblock \bibinfo{journal}{\emph{Journal of behavioral addictions}}
  \bibinfo{volume}{10}, \bibinfo{number}{3} (\bibinfo{year}{2021}),
  \bibinfo{pages}{482--497}.
\newblock


\bibitem[Tebbe et~al\mbox{.}(2021)]%
        {tebbe2021TableTennis}
\bibfield{author}{\bibinfo{person}{Jonas Tebbe}, \bibinfo{person}{Lukas
  Krauch}, \bibinfo{person}{Yapeng Gao}, {and} \bibinfo{person}{Andreas Zell}.}
  \bibinfo{year}{2021}\natexlab{}.
\newblock \showarticletitle{Sample-efficient reinforcement learning in robotic
  table tennis}. In \bibinfo{booktitle}{\emph{2021 IEEE international
  conference on robotics and automation (ICRA)}}. IEEE,
  \bibinfo{publisher}{IEEE}, \bibinfo{address}{China},
  \bibinfo{pages}{4171--4178}.
\newblock


\bibitem[Thabet et~al\mbox{.}(2019)]%
        {thabet2019sampleDeepHRI}
\bibfield{author}{\bibinfo{person}{Mohammad Thabet},
  \bibinfo{person}{Massimiliano Patacchiola}, {and} \bibinfo{person}{Angelo
  Cangelosi}.} \bibinfo{year}{2019}\natexlab{}.
\newblock \showarticletitle{Sample-efficient deep reinforcement learning with
  imaginary rollouts for human-robot interaction}. In
  \bibinfo{booktitle}{\emph{2019 IEEE/RSJ International Conference on
  Intelligent Robots and Systems (IROS)}}. IEEE, \bibinfo{publisher}{IEEE},
  \bibinfo{address}{Macau}, \bibinfo{pages}{5079--5085}.
\newblock


\bibitem[Trella et~al\mbox{.}(2022)]%
        {trella2022designingRLforDigitalHealth}
\bibfield{author}{\bibinfo{person}{Anna~L Trella}, \bibinfo{person}{Kelly~W
  Zhang}, \bibinfo{person}{Inbal Nahum-Shani}, \bibinfo{person}{Vivek Shetty},
  \bibinfo{person}{Finale Doshi-Velez}, {and} \bibinfo{person}{Susan~A
  Murphy}.} \bibinfo{year}{2022}\natexlab{}.
\newblock \showarticletitle{Designing reinforcement learning algorithms for
  digital interventions: pre-implementation guidelines}.
\newblock \bibinfo{journal}{\emph{Algorithms}} \bibinfo{volume}{15},
  \bibinfo{number}{8} (\bibinfo{year}{2022}), \bibinfo{pages}{255}.
\newblock


\bibitem[van~der Pol et~al\mbox{.}(2020)]%
        {vanderpol2020plannable}
\bibfield{author}{\bibinfo{person}{Elise van~der Pol}, \bibinfo{person}{Thomas
  Kipf}, \bibinfo{person}{Frans~A. Oliehoek}, {and} \bibinfo{person}{Max
  Welling}.} \bibinfo{year}{2020}\natexlab{}.
\newblock \bibinfo{title}{Plannable Approximations to MDP Homomorphisms:
  Equivariance under Actions}.
\newblock
\newblock
\showeprint[arxiv]{2002.11963}~[cs.LG]


\bibitem[Wang et~al\mbox{.}(2021)]%
        {wang2021optimizingRLSimulatorHealth}
\bibfield{author}{\bibinfo{person}{Shihan Wang}, \bibinfo{person}{Chao Zhang},
  \bibinfo{person}{Ben Kr{\"o}se}, {and} \bibinfo{person}{Herke van Hoof}.}
  \bibinfo{year}{2021}\natexlab{}.
\newblock \showarticletitle{Optimizing adaptive notifications in mobile health
  interventions systems: reinforcement learning from a data-driven behavioral
  simulator}.
\newblock \bibinfo{journal}{\emph{Journal of medical systems}}
  \bibinfo{volume}{45} (\bibinfo{year}{2021}), \bibinfo{pages}{1--8}.
\newblock


\bibitem[Xu(2022)]%
        {xu2022towardsBridgingModels}
\bibfield{author}{\bibinfo{person}{Xuhai Xu}.} \bibinfo{year}{2022}\natexlab{}.
\newblock \showarticletitle{Towards Future Health and Well-being: Bridging
  Behavior Modeling and Intervention}. In \bibinfo{booktitle}{\emph{Adjunct
  Proceedings of the 35th Annual ACM Symposium on User Interface Software and
  Technology}}. \bibinfo{publisher}{Association for Computing Machinery},
  \bibinfo{address}{New York, USA}, \bibinfo{pages}{1--5}.
\newblock


\bibitem[Yang et~al\mbox{.}(2020)]%
        {yang2020dataEfficientRobots}
\bibfield{author}{\bibinfo{person}{Yuxiang Yang}, \bibinfo{person}{Ken
  Caluwaerts}, \bibinfo{person}{Atil Iscen}, \bibinfo{person}{Tingnan Zhang},
  \bibinfo{person}{Jie Tan}, {and} \bibinfo{person}{Vikas Sindhwani}.}
  \bibinfo{year}{2020}\natexlab{}.
\newblock \showarticletitle{Data efficient reinforcement learning for legged
  robots}. In \bibinfo{booktitle}{\emph{Conference on Robot Learning}}. PMLR,
  \bibinfo{publisher}{PMLR}, \bibinfo{address}{Virtual},
  \bibinfo{pages}{1--10}.
\newblock


\bibitem[Yu and Ho(2022)]%
        {yu2022environmentDesignBiased}
\bibfield{author}{\bibinfo{person}{Guanghui Yu} {and} \bibinfo{person}{Chien-Ju
  Ho}.} \bibinfo{year}{2022}\natexlab{}.
\newblock \showarticletitle{Environment Design for Biased Decision Makers}. In
  \bibinfo{booktitle}{\emph{Proceedings of the International Joint Conference
  on Artificial Intelligence (IJCAI)}}. \bibinfo{publisher}{International Joint
  Conferences on Artificial Intelligence Organization},
  \bibinfo{address}{Austria}, \bibinfo{pages}{592--598}.
\newblock


\bibitem[Zhang et~al\mbox{.}(2022)]%
        {zhang2022theoryBasedHabit}
\bibfield{author}{\bibinfo{person}{Chao Zhang}, \bibinfo{person}{Joaquin
  Vanschoren}, \bibinfo{person}{Arlette van Wissen},
  \bibinfo{person}{Dani{\"e}l Lakens}, \bibinfo{person}{Boris de Ruyter}, {and}
  \bibinfo{person}{Wijnand~A IJsselsteijn}.} \bibinfo{year}{2022}\natexlab{}.
\newblock \showarticletitle{Theory-based habit modeling for enhancing behavior
  prediction in behavior change support systems}.
\newblock \bibinfo{journal}{\emph{User Modeling and User-Adapted Interaction}}
  \bibinfo{volume}{32}, \bibinfo{number}{3} (\bibinfo{year}{2022}),
  \bibinfo{pages}{389--415}.
\newblock


\bibitem[Zhang et~al\mbox{.}(2021)]%
        {zhang2021usingCognitiveModels}
\bibfield{author}{\bibinfo{person}{Chao Zhang}, \bibinfo{person}{Shihan Wang},
  \bibinfo{person}{Henk Aarts}, {and} \bibinfo{person}{Mehdi Dastani}.}
  \bibinfo{year}{2021}\natexlab{}.
\newblock \bibinfo{title}{Using Cognitive Models to Train Warm Start
  Reinforcement Learning Agents for Human-Computer Interactions}.
\newblock
\newblock
\showeprint[arxiv]{2103.06160}~[cs.AI]


\bibitem[Zhi-Xuan et~al\mbox{.}(2020)]%
        {zhi2020onlineGoalInference}
\bibfield{author}{\bibinfo{person}{Tan Zhi-Xuan}, \bibinfo{person}{Jordyn
  Mann}, \bibinfo{person}{Tom Silver}, \bibinfo{person}{Josh Tenenbaum}, {and}
  \bibinfo{person}{Vikash Mansinghka}.} \bibinfo{year}{2020}\natexlab{}.
\newblock \showarticletitle{Online bayesian goal inference for boundedly
  rational planning agents}.
\newblock \bibinfo{journal}{\emph{Advances in neural information processing
  systems}}  \bibinfo{volume}{33} (\bibinfo{year}{2020}),
  \bibinfo{pages}{19238--19250}.
\newblock


\bibitem[Zhou et~al\mbox{.}(2018)]%
        {zhou2018personalizingFitness}
\bibfield{author}{\bibinfo{person}{Mo Zhou}, \bibinfo{person}{Yonatan Mintz},
  \bibinfo{person}{Yoshimi Fukuoka}, \bibinfo{person}{Ken Goldberg},
  \bibinfo{person}{Elena Flowers}, \bibinfo{person}{Philip Kaminsky},
  \bibinfo{person}{Alejandro Castillejo}, {and} \bibinfo{person}{Anil Aswani}.}
  \bibinfo{year}{2018}\natexlab{}.
\newblock \showarticletitle{Personalizing mobile fitness apps using
  reinforcement learning}. In \bibinfo{booktitle}{\emph{CEUR workshop
  proceedings}}, Vol.~\bibinfo{volume}{2068}. NIH Public Access,
  \bibinfo{publisher}{CEUR workshop proceedings}, \bibinfo{address}{Japan}.
\newblock


\end{thebibliography}

\newpage
\appendix 
\onecolumn  

\section{Optimal value functions in chainworlds}
\label{appendix: solving}
In this section, we solve for the analytical solution of the optimal value function (and therefore the optimal policy) for chainworlds $\chainworld \in \simplechains$. 

In our setting, once the optimal action is to go right in a given state, the best strategy is to continue going right in subsequent states that are closer to the goal. That is, if $\pi^*(w) = 1$, then $\pi^*(w + 1) = 1$. The opposite is also true; if the optimal action is to stay in place in a given state, then the best strategy in a state that is farther away from the goal is also to stay in place --- if $\pi^*(w) = 0$, then $\pi^*(w - 1) = 0$. 

In other words, the optimal value function maximizes between a policy that goes to the goal state $\pi_g$ and a policy that goes to disengagement $\pi_d$.
Specifically, for MDP $\chainworld \in \simplechains$, the corresponding optimal value function $V^*_\theta$ is 
$V^*_\theta(s) = \max\{\vdis(s), \vgoal(s)\},$
and the optimal policy $\pi_\theta^*$ is
$\pi^*_\theta(s) = \mathbb{I}\{\vgoal(s) > \vdis(s)\},$
for all $s \in \mathcal{S}$. 

\subsection{Derivation of \texorpdfstring{$\vgoal$}{V goal}}
We will start by deriving $\vgoal$ for states close to the goal state $s_g$, and generalize these findings. First, note that $\vgoal(s_N)  = \vgoal(s_g) = r_g$ because $s_g$ is absorbing. 

Next, we will derive the value of a state which is right before the goal state, $s_{N - 1}$. Recall that when $a=1$ the human moves right with probability $p_g$ and stays in place with probability $1 - p_g$. The human always receives a reward of $r_b$ for choosing $a=1$. Using Bellman recursion for the value function results in, 
\begin{align}
\begin{split}
\label{eq: v1}
    & \vgoal(s_{N - 1}) \\
    & = r(s_{N - 1}, a = 1) + \gamma \sum\limits_{s'} P(s' | s = s_{N - 1}, a = 1) \vgoal(s')\\
    & = r_b + \gamma \left [ p_g \vgoal(s_N) + (1 - p_g) \vgoal(s_{N-1}) \right ] \\
    & = r_b + \gamma \left [ p_g r_g + (1 - p_g) \vgoal(s_{N-1}) \right ] \\
    & = r_b + \gamma p_g r_g +  \gamma (1 - p_g) \vgoal(N-1) \\
    & = r_b + \gamma p_g r_g +  \gamma (1 - p_g) [ r_b + \gamma p_g r_g +  \gamma (1 - p_g) \vgoal(s_{N-1})] \\
    & = \underbrace{r_b + \gamma p_g r_g}_{\text{when } s = s_{N - 1} \text{ at time} 0} +  \underbrace{\gamma (1 - p_g) [r_b + \gamma p_g r_g]}_{\text{when } s = s_{N - 1} \text{ at time } 1} + \underbrace{\gamma^2 (1 - p_g)^2 \vgoal(s_{N-1})}_{\text{when } s = s_{N - 1} \text{ at time} 2} + \ldots \\
    & = \sum \limits_{t = 0}^\infty \gamma^t (1 - p_r)^t [r_b + \gamma p_g r_g] \\
    & = \frac{p_g \gamma r_g + r_b}{1 - \gamma (1 - p_g)}\\
    & = \frac{p_g \gamma r_g + r_b}{z},
\end{split}
\end{align}

for $z = 1 - \gamma(1 - p_g)$. Next, using a similar strategy, we derive the value of a state which is \emph{two spaces} away from the goal state, $s_{N - 2}$: 
\begin{align}
\begin{split}
    \label{eq: v2}
    & \vgoal(s_{s_{N - 2}}) \\
    & = r(s = s_{N - 2}, a = 1) + \gamma \sum\limits_{s'} P(s' | s = s_{N - 2}, a = 1) \vgoal(s')\\
    & = r_b + \gamma \left [ p_g \vgoal(s_{N - 1}) + (1 - p_g) \vgoal(s_{N-2}) \right ] \\
    & = r_b + \gamma  p_g \vgoal(s_{N - 1}) + \gamma (1 - p_g) \vgoal(s_{N-2})  \\
    & = \underbrace{ r_b + \gamma  p_g \vgoal(s_{N - 1})}_{\text{ when } s = s_{N - 2} \text{ at time 0}} 
    + \underbrace{\gamma (1 - p_g) [ r_b + \gamma  p_g \vgoal(s_{N - 1})]}_{\text{ when } s = s_{N - 2} \text{ at time 1}} 
    + \underbrace{\gamma^2 (1 - p_g)^2 [ r_b + \gamma  p_g \vgoal(s_{N - 1})]}_{\text{ when } s = s_{N - 2} \text{ at time 2}} + \ldots\\
    & = \sum\limits_{t = 0}^\infty \gamma^t(1-p_g)^t [ r_b + \gamma  p_g \vgoal(s_{N - 1})]\\
    & = \frac{ r_b + \gamma p_g \vgoal(s_{N - 1})}{1 - \gamma(1 - p_g)}\\
    & = \frac{ r_b + \gamma p_g \frac{\gamma p_g r_g + r_b}{z}}{z}\\
    & = \frac{\gamma^2 p_g^2 r_g}{z^2} + \frac{\gamma p_g r_b}{z^2} + \frac{r_b}{z}.\\
\end{split}
\end{align}
 
In general, we can apply the Bellman equation to ``recursively'' expand the form of the value function, so that the value at a given state $s_n$ 
 can be written as an infinite geometric series: 

\begin{align}
\label{eq: recursive_value}
    \vgoal(s_n) = \sum \limits_{t = 0}^\infty \gamma^t (1-p)^t \left [ r_b + \gamma p_g \vgoal(s_{n+1}) \right ]
\end{align}

We will apply \cref{eq: recursive_value} to our final derivation of $\vgoal(s_{N- 3})$:
\begin{align}
\begin{split}
    & \vgoal(s_{N - 3}) \\
    & = \sum\limits_{t = 0}^\infty \gamma^t(1-p_g)^t [r_b + \gamma p_g \vgoal(s_{N-2}) ]\\
    & = \frac{r_b + \gamma p_g \vgoal(s_{N - 2}) }{1 - \gamma(1 - p_g)}\\
    & = \frac{r_b + \gamma p_g \left(\frac{\gamma^2 p^2 r_g}{z^2} + \frac{\gamma p_g r_b}{z^2} + \frac{r_b}{z}\right)}{z}\\
    & = \frac{\gamma^3 p_g^3}{z^3} r_g + \frac{\gamma^2 p_g^2}{z^3}r_b + \frac{\gamma p_g}{z^2}r_b + \frac{1}{z}r_b.\\
\end{split}
\end{align}

In general, for any state $s_{N - \delta}$, the value function is: 
\begin{align}
\begin{split}
    & \vgoal(s_{N - \delta}) \\
    & = r_g \left (\frac{\gamma\ p_g}{z} \right )^\delta+ \frac{r_b}{z} \sum\limits_{n = 0}^{\delta - 1} \left (\frac{\gamma \ p_g}{z} \right )^n \\
    & = r_g \left (\frac{\gamma\ p_g}{z} \right )^\delta+ r_b  \left (\frac{1 - (\gamma \ p_g / z)^\delta}{1 - \gamma}\right),
\end{split}
\end{align} 
where $z = 1 - \gamma(1 - p_g)$. 

\subsection{Derivation of \texorpdfstring{$\vdis$}{V dis}} 
This derivation is similar in nature to the one on $\vgoal$. 
Note that $\vdis(s_0) = r_d \left (\frac{\gamma p_{d0}}{1 - \gamma(1 - p_{d0})}\right )$. 
We will begin by solving for $\vdis(s_1)$: 

\begin{align}
\label{eq: v1_backwards}
\begin{split}
    & \vdis(s_1) = \sum\limits_{s'} P(s' | s = s_1, a = 0) [\ r(s = s_1, a = 0, s') + \gamma \vdis(s')\ ]\\
    & = \underbrace{p_{d}[0 + \gamma \vdis(s_d)]}_{s' = s_d} + \underbrace{p_\ell [r_\ell + \gamma \vdis(s_0)]}_{s' = s_0} + \underbrace{(1 - p_\ell - p_d)[0 + \gamma \vdis(s_1)]}_{s' = s_1}\\
    & = \gamma p_{d} r_d + p_\ell r_\ell + p_\ell \gamma \vdis(s_0) + (1 - p_\ell - p_d) \gamma \vdis(s_1)\\
    & = \sum\limits_{t = 0}^\infty \gamma^t(1 - p_\ell - p_d)^t[\gamma p_{d} r_d + p_\ell r_\ell + p_\ell \gamma \vdis(s_0)]\\
    & = \frac{\gamma p_{d} r_d + p_\ell r_\ell + p_\ell \gamma \vdis(s_0)}{1 - \gamma(1 - p_\ell - p_d)}\\
    & = r_d \left(\frac{\gamma p_{d}}{u}\right) + r_\ell \left(\frac{p_\ell}{u}\right) + r_d \left (\frac{\gamma^2 p_\ell p_{d0}}{u z} \right ),
\end{split}
\end{align} 
where $u = 1 - \gamma(1 - p_\ell - p_d)$ and $z = 1 - \gamma(1 - p_{d0})$.
In the same way, $\vdis(s_2)$ is: 

\begin{align}
\label{eq: v2_backwards}
\begin{split}
    & \vdis(s_2)\\
    & = \sum \limits_{t = 0}^\infty \gamma^t (1 - p_\ell - p_d)^t [\gamma p_d r_d + p_\ell r_\ell + p_\ell \gamma \vdis(s_1)]\\
    &=  \frac{\gamma p_d r_d + p_\ell r_\ell + p_\ell \gamma \vdis(s_1)}{u}\\
    &= r_d \parenfrac{\gamma p_d}{u} + r_\ell \parenfrac{p_\ell}{u} + \frac{\gamma p_\ell}{u} \left( r_d \parenfrac{\gamma p_{d}}{u} + r_\ell \parenfrac{p_\ell}{u} + r_d \parenfrac{\gamma^2 p_\ell p_{d0}}{uz} \right)\\
    &= r_d \parenfrac{\gamma p_d}{u} + r_\ell \parenfrac{p_\ell}{u} + r_d \parenfrac{\gamma^2 p_\ell p_{d}}{u^2} + r_\ell \parenfrac{\gamma p_\ell^2}{u^2} + r_d \parenfrac{\gamma^3 p_\ell^2 p_{d0}}{u^2 z}
\end{split}
\end{align}

This yields to a general form: 
\begin{equation}
    \vdis(s_n) =  r_d \parenfrac{\gamma p_{d0}}{z} \parenfrac{\gamma p_\ell}{u}^n + (\gamma p_d r_d + p_\ell r_\ell) \sum\limits_{t = 0}^n \parenfrac{\gamma p_\ell}{u}^n
    = r_d \parenfrac{\gamma p_{d0}}{z} \parenfrac{\gamma p_\ell}{u}^n + (\gamma p_d r_d + p_\ell r_\ell) \parenfrac{1 - (\gamma p_\ell / u)^n}{1 -\gamma(1 - p_d)}
\end{equation}

\section{AI policies for chainworld humans}
\label{appendix: chainworld-proof}

\begin{figure}[ht] 
    \centering
    \begin{minipage}{0.5\textwidth}
    \centering  
        \begin{subfigure}{0.7\linewidth}
             \centering              \includegraphics[width=1\linewidth]{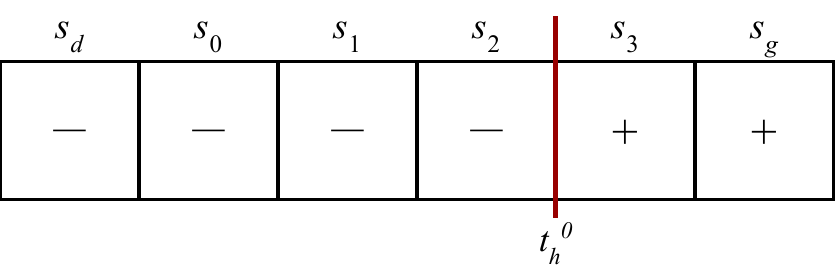}
             \caption{\textbf{Human's optimal policy} when $\ai{a} = 0$. The human is planning in an MDP with default values for $\gamma, r_b$.}
             \label{fig: human-policy-ai-0}
         \end{subfigure}
         
        \begin{subfigure}{0.7\linewidth}
             \centering              
             \includegraphics[width=1\linewidth]{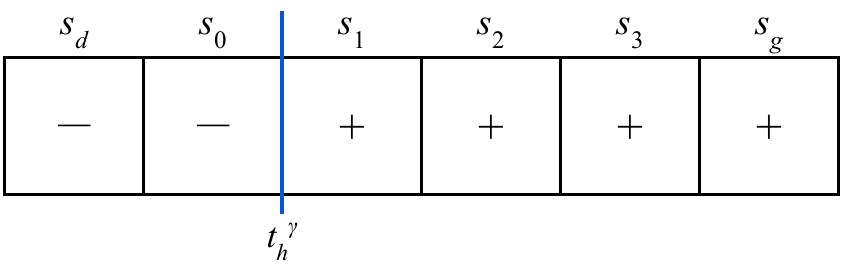}
             \caption{\textbf{Human's optimal policy}  when $\ai{a} = a_\gamma$. The human is planning in an MDP with default $r_b$ and increased $\gamma' = \gamma + \Delta_\gamma$.}
             \label{fig: human-policy-ai-1}
         \end{subfigure}
 
        \begin{subfigure}{0.7\linewidth}
             \centering              
             \includegraphics[width=1\linewidth]{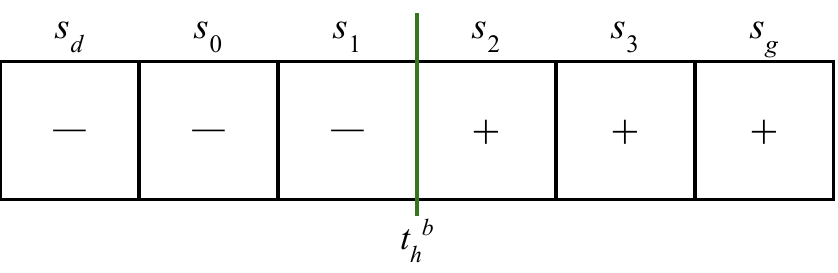}
             \caption{\textbf{Human's optimal policy}  when $\ai{a} = a_b$. The human is planning in an MDP with default $\gamma$ and decreased $r_b - \Delta_b$.}
             \label{fig: human-policy-ai-2}             
         \end{subfigure}
    \end{minipage}%
    \begin{minipage}{0.5\textwidth}
    \centering  
        \begin{subfigure}{0.7\linewidth}
             \centering              
             \includegraphics[width=1\linewidth]{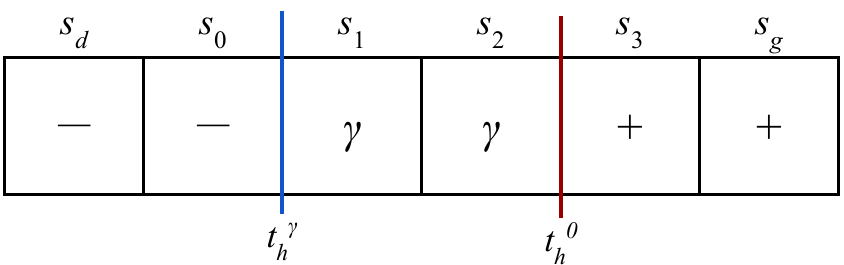}
             \caption{\textbf{Example of AI's optimal policy}, given the example thresholds for $t_h^0, t_h^\gamma, t_h^b$ from \cref{fig: human-policy-ai-0}, \cref{fig: human-policy-ai-1}, \cref{fig: human-policy-ai-2}}
             \label{fig: ai-policy-example-thresholds}      
         \end{subfigure}
    \end{minipage}
    \caption{Example of how the human's optimal policies for chainworld following AI actions $\ai{a} = 0$ (\cref{fig: human-policy-ai-0}), $\ai{a} = a_\gamma$ (\cref{fig: human-policy-ai-1}), and $\ai{a} = a_b$ (\cref{fig: human-policy-ai-2}) result in the AI optimal policy in \cref{fig: ai-policy-example-thresholds}.}
    \label{fig: threshold-ai-policy-overview}
\end{figure}

\subsection{Proof that chainworld AI has a 3-window policy}
\label{appendix: proof-ai-policy-chainworld}
The following is the proof for \cref{thm: chainworld-policies}. 

\begin{proof}
\setlength{\parindent}{0cm}
We will prove this on a case-by-case basis. 

\textbf{The optimal AI policy is takes action $\boldsymbol{0}$ for states $\boldsymbol{s_n}$ 
where $\boldsymbol{n \le \tmin}$.}
We will prove by negation. 

Let $\pi_1$ be defined as in \cref{eq: chainworld-ai-policy}. 
Suppose $\pi_1$ is not optimal. 
This implies that there must exist some optimal policy, $\pi_2$, whose actions are non-zero for a subset of $\{s_m\}$ states, where $m < \tmin$. 

Note that, $$V^{\pi_1}(s_m) = \aidis\parenfrac{\gamma p_d}{1 - \gamma(1 - p_d)},$$ because the human will not take action in states before the threshold $\tmin$, so disengagement is inevitable. 
Similarly, $$V^{\pi_2}(s_m) =\aiint + \aidis \parenfrac{\gamma p_d}{1 - \gamma(1 - p_d)},$$ where the human outcome remains the same, but the AI agent receives additional penalty $\aiint$ for sending an intervention. 

Since $\pi_2$ is an optimal policy, $V^{\pi_2}(s_m) \ge V^{\pi_1}(s_m) \implies \aiint + \aidis \parenfrac{\gamma p_d}{1 - \gamma(1 - p_d)} > \aidis \parenfrac{\gamma p_d}{1 - \gamma(1 - p_d)} \implies \aiint > 0$. This cannot be true, since $\aiint < 0$ by construction. 
\\ \\

\textbf{The optimal AI policy takes action $\boldsymbol{0}$ for states $\boldsymbol{s_n}$ 
where  $\boldsymbol{n > \human{t}^0}$.}
We will prove by negation. 

Let $\pi_1$ be defined as in \cref{eq: chainworld-ai-policy}. 
Suppose $\pi_1$ is not optimal. 
This implies that there must exist some optimal policy, $\pi_2$, whose actions are non-zero for a subset of $\{s_m\}$ states, where $m > \human{t}^0$. 

Note that, $$V^{\pi_1}(s_m) =   \aigoal \parenfrac{\ai{\gamma}\ p_g}{1 - \ai{\gamma}(1 - p_g)}^{N - m},$$ because the human will always take action in states after the threshold $\human{t}^0$, so they will always reach the goal state. 
Similarly, $$V^{\pi_2}(s_m) \le \aiint +\aigoal \parenfrac{\ai{\gamma}\ p_g}{1 - \ai{\gamma}(1 - p_g)}^{N - m},$$ where the human outcome remains the same, but the AI agent receives additional penalty $\aiint$ for sending an intervention in this state, and (possibly) subsequent states. 

Since $\pi_2$ is an optimal policy, $V^{\pi_2}(s_m) \ge V^{\pi_1}(s_m) \implies \aiint +\aigoal \parenfrac{\ai{\gamma}\ p_g}{1 - \ai{\gamma}(1 - p_g)}^{N - m} > \aigoal \parenfrac{\ai{\gamma}\ p_g}{1 - \ai{\gamma}(1 - p_g)}^{N - m}$. This cannot be true, since $\aiint < 0$ by construction. 
\\ \\

\textbf{The optimal AI policy takes action $\boldsymbol{0}$ when $\boldsymbol{\tmin < n \le \ai{t}}$ and takes action $\boldsymbol{a_\gamma}$ when $\boldsymbol{ \ai{t} < n \le \human{t}^0}$.}

Without loss of generality, assume $\human{t}^\gamma < \human{t}^b$.


By construction, an episode in the chainworld is finite with two absorbing states, so the optimal AI policy must choose weather to influence the human toward the goal or disengagement state. 

Let $\ai{\pi}^g$ denote the \emph{highest value} policy to the goal state. Such a policy has two behaviors. First, the policy will always take $\ai{\pi}^g(s_n) = a_\gamma$ in states where the AI agent can move the threshold before the current state; this corresponds to states $s_n$ where $\tmin \le < n < \human{t}^0$. This is because \emph{witholding intervention} on some states by taking action $\ai{\pi}^g(s_n) = 0$ means that the human will not pursue the goal, since $n < \human{t}^0$, and the AI agent will receive a disengagement penalty. Second, the goal policy will always take action $0$ in states $s_n$ where $n \ge \human{t}^0$, as we showed earlier in the proof. 
The value of $\ai{\pi}^g$ is, $$V^{\ai{\pi}^g}(s_n) =  \aigoal \parenfrac{\ai{\gamma}\ p_g}{1 - \ai{\gamma}(1 - p_g)}^{N - n} + \aiint\parenfrac{1 - (\ai{\gamma}\ p_g / (1 - \ai{\gamma}(1 - p_g)))^{N - n}}{1 - (\ai{\gamma}\ p_g / (1 - \ai{\gamma}(1 - p_g)))},$$ which is the discounted reward of the goal reduced by the cost of the interventions to reach the goal.

Let $\ai{\pi}^d$ denote the \emph{highest value} policy to disengagement, which takes action $\ai{\pi}^d(s_n) = 0$ in states $s_n$ where $\tmin < n \le \human{t}^0$. This is for the same reason as we showed earlier in the proof; when disengagement is inevitable, it is better to withold intervention to avoid the additional cost. 
The value of $\ai{\pi}^d$ is, $$V^{\ai{\pi}^d}(s_n) = \aidis \parenfrac{\gamma p_d}{1 - \gamma(1 - p_d)}.$$

When $V^{\ai{\pi}^d}(s_n) \ge V^{\ai{\pi}^g}(s_n)$, the value of disengagement outweighs the cost of reaching the goal, and the optimal AI policy will take action $0$. By definition, $V^{\ai{\pi}^d}(s_n) \ge V^{\ai{\pi}^g}(s_n)$ when $n \le \ai{t}$. So, the optimal AI policy will take action $0$ when $n \le \ai{t}$.

Similarly, when $V^{\ai{\pi}^d}(s_n) < V^{\ai{\pi}^g}(s_n)$, the optimal AI policy will take action $a_\gamma$. By definition, $V^{\ai{\pi}^d}(s_n) < V^{\ai{\pi}^g}(s_n)$ when $n > \ai{t}$. So, the optimal AI policy will take action $a_\gamma$ when $n > \ai{t}$.
\end{proof}

\subsection{Proof that chainworlds can cover the entire space of 3-window policies}
The following is the proof for \cref{thm: window-based-equivalence}. 
\label{appendix: proof-chainworld-equivalence-class}
\begin{proof}
    Note that the optimal chainworld AI policy in \cref{eq: chainworld-ai-policy} depends on four quantities: the human thresholds $\human{t}^0, \human{t}^\gamma, \human{t}^b$ and the AI policy threshold $\ai{t}$. Since the AI policy threshold is a direct result of the human thresholds and the AI MDP, it suffices to show in this proof that there exists a $\theta$ that can produce \emph{all possible values of } $\human{t}^0, \human{t}^\gamma, \human{t}^b$. 

    We will prove this in three parts, and each part will be similar in structure: 
    \begin{enumerate}
        \item First, we will show that there exist chainworld parameters without considering AI intervention effects ($\theta \setminus \{\Delta_\gamma, \Delta_b\}$) that can define any $\human{t}^0$.
        \item Then, we will show that there exists a $\Delta_b$ that can define any $\human{t}^b$, given the chainworld parameters from step (1). 
        \item Finally, we will show that there exists a $\Delta_\gamma$ that can define any $\human{t}^\gamma$ given the chainworld parameters from step (1). 
    \end{enumerate}

    \textbf{There exists $\boldsymbol{\theta \setminus \{\Delta_\gamma, \Delta_b\}}$ that can define any $\human{t}^0$.} We will refer to $\human{t}^0$ as $t_0$ for brevity. 
    By \cref{def: threshold}, any $t_0 \in \{0, \ldots, N-1\}$ must satisfy two constraints: 
    $$\vgoal(s_{t_0}) < \vdis(s_{t_0}) \And \vgoal(s_{t_0 + 1}) > \vdis(s_{t_0+1}).$$
    We will prove that there exists $\theta \setminus \{\Delta_\gamma, \Delta_b\}$ so that these constraints are always satisfied. 

    Let $r_d = 0, r_\ell = 0, p_g = 1, p_d = 0, p_{d0} = 0, p_ell = 0$. The constraints become: 
    \begin{align}
    \label{eq: t-0-conditions}
        \begin{split}
            & \vgoal(s_{t_0}) < \vdis(s_{t_0}) \And \vgoal(s_{t_0 + 1}) > \vdis(s_{t_0+1}) \\
            & \implies r_g \gamma^{N - t_0} - r_b \parenfrac{1 - \gamma^{N - t_0}}{1 - \gamma} > 0 \And r_g \gamma^{N - t_0 - 1} - r_b \parenfrac{1 - \gamma^{N - t_0 - 1}}{1 - \gamma} < 0 \\
            & \implies r_g  > r_b \parenfrac{1 - \gamma^{N - t_0}}{\gamma^{N - t_0} (1 - \gamma)} \And r_g < r_b \parenfrac{1 - \gamma^{N - t_0 - 1}}{\gamma^{N - t_0 - 1}(1 - \gamma)}\\
            & \implies r_b \parenfrac{1 - \gamma^{N - t_0}}{\gamma^{N - t_0} (1 - \gamma)} < r_g < r_b \parenfrac{1 - \gamma^{N - t_0 - 1}}{\gamma^{N - t_0 - 1}(1 - \gamma)}.
        \end{split}
    \end{align}

    These constraints are satisfied so long as there exists valid chainworld parameters so that the following inequality holds: 
    \begin{align}
        \begin{split}
            & \implies r_b \parenfrac{1 - \gamma^{N - t_0}}{\gamma^{N - t_0} (1 - \gamma)} < r_b \parenfrac{1 - \gamma^{N - t_0 - 1}}{\gamma^{N - t_0 - 1}(1 - \gamma)} \\
            & \implies \parenfrac{1 - \gamma^{N - t_0}}{\gamma^{N - t_0}} < \parenfrac{1 - \gamma^{N - t_0 - 1}}{\gamma^{N - t_0 - 1}}\\
            & \implies \gamma(1 - \gamma^{N - t_0 -1}) < 1 - \gamma^{N - t_0}\\
            & \implies \gamma < 1.
        \end{split}
    \end{align}

    So, there exists 
    \begin{align}
    \label{eq: theta-t0}
    \begin{split}
        & \theta \setminus \{\Delta_\gamma, \Delta_b\} = \{p_g = 1, p_d = 0, p_{d0} = 0, p_\ell = 0, r_d = 0, r_\ell = 0, \\
        & \quad \quad \quad \quad \quad\ \  \quad  r_g\text{ (that satisfies condition in \cref{eq: t-0-conditions})}, r_b \text{ (that satisfies condition in \cref{eq: t-0-conditions})}, \gamma < 1\},
    \end{split}
    \end{align} 
    which defines any $t_0 \in \{0, \ldots, N - 1\}$. 

    \textbf{There exists an AI effect on burden $\boldsymbol{\Delta_b}$ that can define any human threshold $\boldsymbol{\human{t}^b}$.} We will refer to the human threshold following a burden intervention $\human{t}^b$ as $t_b$ for brevity. By \cref{def: threshold}, the threshold $t_b$ must satisfy the following constraints: 
    $$\vgoal(s_{t_b}; r_b + \Delta_b) < \vdis(s_{t_b}) \And \vgoal(s_{t_b + 1}; r_b + \Delta_b) < \vdis(s_{t_b}),$$
    where $\vgoal(\cdot ; r_b + \Delta_b)$ represents the human's value of goal pursuit under the burden $r_b + \Delta_b$. 
      
    Suppose $\theta \setminus \{\Delta_\gamma, \Delta_b\}$ is defined as in \cref{eq: theta-t0}. Then, 
    \begin{align}
        \begin{split}
            & \vgoal(s_{t_b}; r_b + \Delta_b) < \vdis(s_{t_b}) \And \vgoal(s_{t_b + 1}; r_b + \Delta_b) > \vdis(s_{t_b+1}) \\
            & \implies r_g \gamma^{N - t_b} - (r_b + \Delta_b) \parenfrac{1 - \gamma^{N - t_b}}{1 - \gamma} > 0 \And r_g \gamma^{N - t_b - 1} - (r_b + \Delta_b) \parenfrac{1 - \gamma^{N - t_b - 1}}{1 - \gamma} < 0 \\
            & \frac{r_g (1 - \gamma) \gamma^{N - t_b}}{1 - \gamma^{N - t_b}} < \Delta < \frac{r_g (1 - \gamma) \gamma^{N - t_b - 1}}{1 - \gamma^{N - t_b - 1}}.
        \end{split}
    \end{align}

    These constraints are satisfied so long as there exists parameters so that the following inequality holds: 
    \begin{align}
        \begin{split}
            & \frac{r_g (1 - \gamma) \gamma^{N - t_b}}{1 - \gamma^{N - t_b}} < \frac{r_g (1 - \gamma) \gamma^{N - t_b - 1}}{1 - \gamma^{N - t_b - 1}} \\
            & \implies \frac{\gamma^{N - t_b}}{1 - \gamma^{N - t_b}} < \frac{\gamma^{N - t_b - 1}}{1 - \gamma^{N - t_b - 1}} \\
            & \implies \gamma^{N - t_b} < \gamma^{N - t_b - 1} \\
            & \implies \gamma < 1,
        \end{split}
    \end{align}

    which is true by definition. So, there exists $\Delta_b$ that defines any $t_b \in {0, \ldots, N - 1}$. 

    \textbf{There exists an AI effect on discounting $\boldsymbol{\Delta_\gamma}$ that can define any human threshold $\boldsymbol{\human{t}^\gamma}$.} We will refer to the human threshold following a discount intervention $\human{t}^\gamma$ as $t_\gamma$ for brevity. By definition \cref{def: threshold}, the threshold $t_\gamma$ must satisfy the constraints: 
    $$\vgoal(s_{t_\gamma}; \gamma + \Delta_\gamma) < \vdis(s_{t_\gamma}; \gamma + \Delta_\gamma) \And \vgoal(s_{t_\gamma + 1}; \gamma + \Delta_\gamma) < \vdis(s_{t_\gamma + 1}; \gamma + \Delta_\gamma),$$  where $\vgoal(\cdot ; \gamma + \Delta_\gamma)$ represents the human's value of goal under the discount rate $\gamma + \Delta_\gamma$. The same applies to $\vdis(\cdot ; \gamma + \Delta+\gamma)$. 

    Suppose $\theta \setminus \{\Delta_\gamma, \Delta_b\}$ is defined as in \cref{eq: theta-t0}. Then, 
    \begin{align}
        \begin{split}
            & \vgoal(s_{t_\gamma}; \gamma + \Delta_\gamma) < \vdis(s_{t_\gamma}; \gamma + \Delta_\gamma) \And \vgoal(s_{t_\gamma + 1}; \gamma + \Delta_\gamma) < \vdis(s_{t_\gamma + 1}; \gamma + \Delta_\gamma) \\
            & \implies r_g (\gamma + \Delta_\gamma)^{N - t_\gamma} - r_b \parenfrac{1 - (\gamma + \Delta_\gamma)^{N - t_\gamma}}{1 - (\gamma + \Delta_\gamma)} > 0 \And r_g (\gamma + \Delta_\gamma)^{N - t_\gamma - 1} - r_b \parenfrac{1 - (\gamma + \Delta_\gamma)^{N - t_\gamma - 1}}{1 - (\gamma + \Delta_\gamma)} < 0 \\
            & \implies r_g \parenfrac{(\gamma + \Delta_\gamma)^{N - t_\gamma} (1 - \gamma - \Delta_\gamma)}{1 - (\gamma + \Delta)^{N - t_\gamma}} < r_b < r_g \parenfrac{(\gamma + \Delta_\gamma)^{N - t_\gamma - 1} (1 - \gamma - \Delta_\gamma)}{1 - (\gamma + \Delta)^{N - t_\gamma - 1}}
        \end{split}
    \end{align}

    These constraints are satisfied so long as there exists parameters so that the following inequality holds: 
    \begin{align}
        \begin{split}
            & r_g \parenfrac{(\gamma + \Delta_\gamma)^{N - t_\gamma} (1 - \gamma - \Delta_\gamma)}{1 - (\gamma + \Delta)^{N - t_\gamma}} < r_g \parenfrac{(\gamma + \Delta_\gamma)^{N - t_\gamma - 1} (1 - \gamma - \Delta_\gamma)}{1 - (\gamma + \Delta)^{N - t_\gamma - 1}} \\
            & \implies \frac{(\gamma + \Delta_\gamma)^{N - t_\gamma}}{1 - (\gamma + \Delta)^{N - t_\gamma}} < \frac{(\gamma + \Delta_\gamma)^{N - t_\gamma - 1}}{1 - (\gamma + \Delta)^{N - t_\gamma - 1}} \\
            & \implies \frac{\gamma + \Delta_\gamma}{1 - (\gamma + \Delta_\gamma)^{N - t_\gamma}} < \frac{1}{1 - (\gamma + \Delta_\gamma)^{N - t_\gamma - 1}}\\
            & \implies \gamma + \Delta_\gamma < 1,
        \end{split}
    \end{align}

    which is true by definition. So, there exists $\Delta_\gamma$ that defines any $t_b \in {0, \ldots, N - 1}$. 

    Furthermore, there exists chainworld parameters $\theta$ that define any $\human{t}^0, \human{t}^\gamma, \human{t}^b$, and therefore, any AI policy $\ai{\pi} \in \policyclass$. 
\end{proof}

\section{Equivalence Proofs}
\label{appendix: equivalence}
Throughout this section, we will distinguish chainworld parameters from parameters in other worlds with a\ \ $\widehat \cdot$. For example, we will refer to the human's goal utility $r_g$ with $\widehat r_g$. 

\subsection{Proof of equivalence with \emph{monotonic chainworlds}}
\label{appendix: monotonic-proof}

\begin{theorem}[Chainworld and monotonic chainworld equivalence]
If $\mdp \in \monotonicchains$, then there exists $\widehat \mdp \in \simplechains$ such that $\mdp \equiv \widehat \mdp$ with identity mapping $f(s) = s$ and $g_s(a) = a$. 
\end{theorem}
\begin{proof}
    Like our chainworlds, optimal human policies in monotonic chainworlds are defined by a threshold. 
    Specifically, under each AI intervention, monotonic chainworlds result in the thresholds $\human{t}^0, \human{t}^\gamma, \human{t}^b$ under AI actions $0, a_\gamma, a_b$, respectively. As a result, the proof from \cref{thm: chainworld-policies} holds exactly for monotonically increasing chainworlds. 
\end{proof}

\subsection{Proof of equivalence with \emph{negative effect of AI intervention}}
\label{appendix: negative-ai-effect-proof}

\begin{theorem}[Chainworld equivalence under negative effect of AI intervention]    
If $\mdp$ has the same states, actions, rewards, transitions, and discount as a chainworld except that $\Delta_b > 0$ or $\Delta_\gamma < 0$, then there exists a chainworld MDP $\chainworld \in \simplechains$ such that $\chainworld \equiv \mdp$.
\end{theorem}
\begin{proof}
\ 
\begin{itemize}
    \item  As it is defined in \cref{eq: chainworld-ai-policy}, the AI's optimal policy \emph{only depends on the minimum human threshold}, $\tmin = \min\left\{\human{t}^0, \human{t}^0, \human{t}^0\right\}$
    \item If the AI agent's intervention on $\gamma$ has the negative intended effect, then $\human{t}^\gamma > \human{t}^0$.
    \item Similarly, if the AI intervention on $r_b$ has the negative intended effect, then $\human{t}^b > \human{t}^0$. 
    \item As a result, $\tmin = \human{t}^0$.
    \item As shown in \cref{thm: chainworld-policies}, if $\tmin = \human{t}^0$, then this results in an optimal AI policy where $\ai{\pi}^*(s) = 0$ for all states $s \in \ai{\mathcal{S}}$. 
    \item $\ai{\pi}^* \in \policyclass$, because this is a ``three-window'' AI policy where the intervention window is size $0$. 
    \item Since $\ai{\pi}^* \in \policyclass$, it is in the equivalence class of chainworlds. 
\end{itemize}
\end{proof}

\subsection{Proof that AI equivalence is achieved if human MDPs are equivalent}
\label{appendix: human-equiv-ai-equiv}
In the remaining proofs, we will first equate the rewards and transitions of two human MDPs and show that this carries into AI equivalence. In this section, we prove that two human-level MDPs with the same rewards and transitions result in two AI MDPs with the same optimal policy.

\begin{theorem}
\label{thm: human_to_ai_equivalence}
\setlength{\parindent}{0cm}    
    Suppose we are given two different \textbf{human MDPs} with matching discount factors, 
    $\human{\mdp} = \langle \human{\mathcal{S}}, \human{\mathcal{A}}, \human{T}, \human{R}, \human{\gamma} \rangle \text{ and } \human{\widehat \mdp} = \langle \human{\widehat{\mathcal{S}}}, \human{\widehat{\mathcal{A}}}, \human{\widehat T}, \human{\widehat R},  \human{\gamma} \rangle,$
    whose rewards and transitions are equivalent under some mapping between the state and action spaces. Specifically, under a state mapping $f: \human{\mathcal{S}} \rightarrow \human{\widehat{\mathcal{S}}}$ and (state-specific) action mapping $g_{\human{s}}: \human{\mathcal{A}} \rightarrow \human{\widehat{\mathcal{A}}}$,  
    $$\human{T}(\human{s}, \human{a}, \human{s'}) = \human{\widehat T}(\ f(\human{s}), g_{\human{s}}(\human{a}), f(\human{s'})\ ) \text{ and } \human{R}(\human{s}, \human{a}) = \human{\widehat R}(\ f(\human{s}), g_{\human{s}}(\human{a})\ ), \quad \quad  \forall \human{s} \in \human{\mathcal{S}}, \human{A} \in \human{\mathcal{A}}.$$
    
    Assume that both human agents follow the same action selection strategy. For example, both agents select actions according to the \emph{optimal policy} for $\mdp$ and $\widehat \mdp$. 

    Suppose we are also given two \textbf{AI MDPs} that correspond to each of the respective human MDPs,
    $$\ai{\mdp} = \langle \ai{\mathcal{S}}, \ai{\mathcal{A}}, \ai{T}, \ai{R}, \ai{\gamma} \rangle \text{ and } \ai{\widehat \mdp}= \langle \ai{\widehat{\mathcal{S}}}, \ai{{\mathcal{A}}}, \ai{\widehat T}, \ai{\widehat R}, \ai{\gamma} \rangle,$$
    with matching discount functions and action spaces. The AI rewards are also mapping under the mappings so that  
    $$\ai{R}([\human{s}, \human{a}], \ai{a}) = \ai{\widehat R}(\ [f(\human{s}), g_{\human{s}}(\human{a})], \ai{a}\ ), \quad \quad  \forall \human{s} \in \human{\mathcal{S}}, \human{A} \in \human{\mathcal{A}}.$$


    Then, the \textbf{optimal AI policies} are equal, where 
    $$\ai{\pi}^*(\ [\human{s}, \human{a}]\ ) = \ai{\widehat \pi}^*(\ [\ f(\human{s})\ ,\ g_{\human{s}}(\human{a})\ ]\ ), \quad \quad \forall \human{s} \in \human{\mathcal{S}}, \human{A} \in \human{\mathcal{A}}.$$
\end{theorem}

\begin{proof} 
\setlength{\parindent}{0cm}
    \begin{align}
        \begin{split}
            & \ai{\pi^*}(\ [\human{s}, \human{a}]\ )\\
            & = \argmax\limits_{\ai{a}} \ai{Q}^*(\ [\human{s}, \human{a}]\ , \ai{a})\\
            &  = \argmax\limits_{\ai{a}}\ai{R}([\human{s}, \human{a}], \ai{a}) + \ai{\gamma} \sum\limits_{\human{s}', \human{a}'} \ai{T}([\human{s}, \human{a}], \ai{a},[\human{s}', \human{a}'])  \argmax\limits_{\ai{a}'} \ai{Q}^*(\ [\human{s}', \human{a}']\ , \ai{a}) \\
            &  = \argmax\limits_{\ai{a}}\ai{R}([\human{s}, \human{a}], \ai{a}) + \ai{\gamma} \sum\limits_{\human{s}', \human{a}'} \human{T}(\human{s}, \human{a'}, \human{s}') \human{\pi}(\human{a}\ |\ \human{s}, \ai{a}) \argmax\limits_{\ai{a}'} \ai{Q}^*(\ [\human{s}', \human{a}']\ , \ai{a}).
        \end{split}
    \end{align}
    
    Since we are given $\ai{R} \equiv \ai{\widehat R}, \human{T} \equiv \human{\widehat T}$ under mappings $f, g$, 
    \begin{align}
        \begin{split}
            & = \argmax\limits_{\ai{a}} 
              \ai{\widehat R}(\ [f(\human{s}), g_{\human{s}}(\human{a})], \ai{a}\ ) 
              + \ai{\gamma} \sum\limits_{\human{s}', \human{a}'} 
                \human{\widehat T}(\ f(\human{s}), g_{\human{s}}(\human{a}), f(\human{s'})\ ) 
                \human{\pi}(\human{a}\ |\ \human{s}, \ai{a}) 
                \argmax\limits_{\ai{a}'} \ai{Q}^*(\ [\human{s}', \human{a}']\ , \ai{a}).
        \end{split}
    \end{align}

    Note that $\human{T} \equiv \human{\widehat T} \text{ and } \human{R} \equiv \human{\widehat R} \text{ under } f, g \implies \human{Q}^*(\human{s}, \human{a}) = \human{\widehat Q}^* (\ f(\human{s}), g_{\human{s}}(\human{a})\ ) \forall \human{s}, \human{a} \implies \human{\pi}(\human{s}, \human{a}) = \human{\widehat \pi}(\ f(\human{s}), g_{\human{s}}(\human{a})\ ) \forall \human{s},\human{a}$. So, 

    \begin{align}
        \begin{split}
                & = \argmax\limits_{\ai{a}} 
              \ai{\widehat R}(\ [f(\human{s}), g_{\human{s}}(\human{a})], \ai{a}\ ) 
              + \ai{\gamma} \sum\limits_{\human{s}', \human{a}'} 
                \human{\widehat T}(\ f(\human{s}), g_{\human{s}}(\human{a}), f(\human{s'})\ ) 
                \human{\widehat \pi}(\ g_{\human{s}}(\human{a})\ |\  f(\human{s}), \ai{a} \ )
                \argmax\limits_{\ai{a}'} \ai{Q}^*(\ [\human{s}', \human{a}']\ , \ai{a}) \\
                & = \argmax \limits_{\ai{a}} \ai{\widehat Q}^*(\ [f(\human{s}), g_{\human{s}}(\human{s})], g_{\human{s}} \ )\\
                & = \ai{\widehat \pi}^*(\ [f(\human{s}), g_{\human{s}}(\human{s})] \ ).
        \end{split}
    \end{align}

\end{proof}

\newpage
\subsection{Proof of equivalence for \emph{progress worlds}}
\label{appendix: progress-proof}

\begin{figure}[ht]
    \centering
    \begin{subfigure}{0.45\linewidth}
    \centering
        \includegraphics[width=0.3\linewidth]{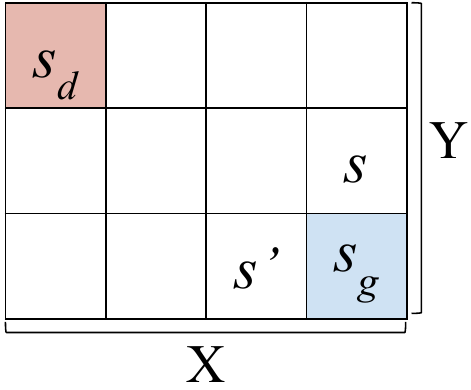}
        \caption{Example of a progress world, because $s$ and $s'$ are the same distance from $s_g$ and are the same distance from $s_d$.}
    \end{subfigure}\hfill
    \begin{subfigure}{0.45\linewidth}
    \centering
        \includegraphics[width=0.3\linewidth]{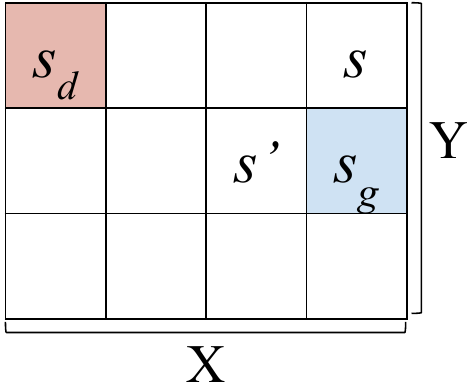}        
        \caption{Example of a graph that is \emph{not} a progress world, because $s$ and $s'$ are the same distance from $s_g$ but different distances from $s_d$.}
        \label{fig: big-small-bad}
    \end{subfigure}
    \label{fig: progress-world-example}
\end{figure}

\begin{definition}[Progress worlds]
\label{def: progress-worlds-extended}
Let $\mdp = \langle \mathcal{S}, \mathcal{A}, T, R, \gamma, \Delta_\gamma, \Delta_b \rangle$ denote a graph MDP, defined as follows:

\begin{itemize}
    \item $\mathcal{S} = \mathcal{S}_1 \times \mathcal{S}_2 \times \ldots \times \mathcal{S}_D$.     
    The states are $D$ dimensional and discrete, where $\mathcal{S}_d$ refers to the set of discrete states in the $d$-th dimension. There is an absorbing goal state $s_g \in \mathcal{S}$, and an absorbing disengagement state $s_d \in \mathcal{S}$. 
    \item $\mathcal{A} = \mathcal{A}_g \cup \mathcal{A}_d$. Actions allow movement bewteen states on the graph. The set of action $\mathcal{A}_g$ are actions that lead closer to the goal state. The set of actions $\mathcal{A}_d$ are actions that lead closer to the disengagement state. 
    \item The graph must be path-connected. The transitions are parametrized by $p$; the agent moves in the intended direction with probability $p$ and stays in place with probability $1-p$. The states $s_d$ and $s_g$ are absorbing. 
    \item $\twoabs{R}(s, a) = \begin{cases}
        r_d > 0, & s = s_d \\ 
        r_g > 0, & s = s_g \\
        r_b < 0, & a \in \mathcal{A}_g
    \end{cases}$. 
    \item $\gamma \in (0, 1)$
    \item $\Delta_\gamma, \Delta_b$. AI action $a = a_\gamma$ increases $\gamma$ by $\Delta_\gamma$ and AI action $a = a_b$ reduces $r_b$ by $\Delta_b$. 
\end{itemize}

\end{definition}

In summary, a progress world $\mdp_\theta \in \twoabs{\mdpclass}$ is parameterized by $\theta = \{\underbrace{p}_\text{transitions}, \underbrace{r_g, r_d, r_b}_\text{rewards}, \underbrace{\gamma}_\text{discount}, \underbrace{\Delta_\gamma, \Delta_b}_\text{app effect} \}$.

AI equivalence with chainworld holds for a subset of progress worlds. This is the subset of worlds for which no two states on the graph are the same distance from the goal state $s_g$, but different distances from the disengagement state $s_d$ (see example in \cref{fig: big-small-bad}). At a high level, this means that all shortest paths between the goal and disengagement state are the same length, which means that a single chainworld can represent all paths (and therefore, the entire world). We prove this in \cref{thm: progress-world-equivalence}.

\begin{theorem}[Chainworld and progress world equivalence]
\label{thm: progress-world-equivalence}
Suppose $\mdp_\theta \in \twoabs{\mdpclass}$. Let $d(s, s')$ denote the shortest graph distance from $s$ to $s'$. Then, there exists $\widehat{\mdp}_{\widehat \theta} \in \simplechains$ such that $\mdp_\theta \equiv \widehat \mdp_{\widehat \theta}$ under state mapping, 
$$f(s) = \begin{cases}
    \widehat s_{d(s, s_d) - 1}, & d(s, s_d) > 0 \\
    \widehat s_d, & d(s, s_d) = 0,
\end{cases}$$

and action mapping $g_s(a) = \indicator{a \in \mathcal{A}_g}$,
where actions in the progress-world that move the human toward the goal correspond to chainworld actions $\widehat a = 1$. 
 
\end{theorem}
\begin{proof}
Consider the following chainworld parameters $\widehat \theta$: 
\begin{itemize}
    \item Length of chain $\widehat N = d(s_g, s_d)$
    \item Goal reward $\widehat r_g = r_g$
    \item Disengagement reward $\widehat r_d = r_d$
    \item Progress loss reward $\widehat r_\ell = 0$
    \item Burden reward $\widehat r_b = r_b$
    \item Probability of moving toward goal $\widehat p_g = p$
    \item Probability of losing progress $\widehat p_\ell = p$
    \item Probability of disengagement $\widehat p_d = 0$
    \item Probability of disengagement at state $0$ $\widehat p_{d0} = p$
    \item Discount factor $\widehat \gamma = \gamma$
    \item Effect of AI intervention on discount $\widehat{\Delta}_\gamma = \Delta_\gamma$
    \item Effect of AI intervention on burden $\widehat{\Delta}_b = \Delta_b$
\end{itemize}

\begin{table}[ht]
    \centering
    \begin{tabular}{p{1in} | c c c | c c c | c c}
         Notes & $d(s, s_d)$ & $a$ & $d(s', s_d)$ & $f(s)$ & $g_s(a)$ & $f(s')$ & $T(s, a, s')$ & $\widehat T(\ f(s), g_s(a), f(s')\ )$\\ \toprule
         $d(s, s_d) > 0$ & $d(s, s_d)$ & $a \in \mathcal{A}_g$ & $d(s, s_d) + 1$ & $\widehat s_{d(s, s_d) - 1}$ & $1$ & $\widehat s_{d(s, s_d)}$ & $p$ & $\widehat p_g$\\
         $d(s, s_d) > 0$ & $d(s, s_d)$ & $a \in \mathcal{A}_g$ & $d(s, s_d)$ & $\widehat s_{d(s, s_d) - 1}$ & $1$ & $\widehat s_{d(s, s_d)}$ & $1 - p$ & $1 - \widehat p_g$ \\ \midrule
         $d(s, s_d) > 1$ & $d(s, s_d)$ & $a \in \mathcal{A}_d$ & $d(s, s_d) - 1$ & $\widehat s_{d(s, s_d) - 1}$ & $0$ & $\widehat s_{d(s, s_d)}$ & $p$ & $\widehat p_\ell$\\
         $d(s, s_d) > 1$ & $d(s, s_d)$ & $a \in \mathcal{A}_d$ & $d(s, s_d)$ & $\widehat s_{d(s, s_d) - 1}$ & $0$ & $\widehat s_{d(s, s_d) - 1}$ & $1 - p$ & $1 - \widehat p_\ell$\\ \midrule
         & $d(s, s_d) = 1$ & $a \in \mathcal{A}_d$ & $d(s, s_d) = 0$ & $\widehat s_{0}$ & $0$ & $\widehat s_{d}$ & $p$ & $\widehat p_{d0}$\\ 
         $d(s, s_d) = 1$ & $1$ & $a \in \mathcal{A}_d$ & $0$ & $\widehat s_{0}$ & $0$ & $\widehat s_{d}$ & $1 - p$ & $1 - \widehat p_{d0}$\\ \midrule
    \end{tabular}
    \caption{Equivalence of progress world transitions. \textnormal{All possible \emph{progress world} transitions in $T$ are equivalent to the \emph{chainworld} transitions in $\widehat T$ under mappings $f$ and $g$. Transitions $T(s, a, s')$ with probability $0$ are not shown; these are clearly still $0$ probability under $\widehat T(\ f(s), g_s(a), f(s')\ )$, since the grouped rows sum to $1$ for both $T$ and $\widehat T$.}}
    \label{tab: progress-transitions}
\end{table}

\begin{table}[ht]
    \centering
    \begin{tabular}{p{1.5in} | c c c | c c c | c c}
         Notes & $s$ & $a$ & $s'$ & $f(s)$ & $g_s(a)$ & $f(s')$ & $R(s, a)$ & $\widehat R(\ f(s), g_s(a)\ )$\\ \toprule
         & $s_d$ & --- & --- & $\widehat s_d$ & --- & --- & $r_d$ & $\widehat r_d$ \\
         & $s_g$ & --- & --- & $\widehat s_g$ & --- & --- & $r_g$ & $\widehat r_g$ \\
         $d(s', s_d) < d(s, s_d)$ where $s'$ results from action $a$ & --- & $a$ & --- & --- & 1 & --- & $r_b$ & $\widehat r_b$ \\
    \end{tabular}
    \caption{Equivalence of progress world rewards. \textnormal{All possible \emph{progress world} rewards in $R$ are equivalent to the \emph{chainworld} rewards in $\widehat R$ under mappings $f$ and $g$. We use ``---'' to represent any action or state. For all other $s, a, s'$ combinations not shown, $R(s, a, s') = \widehat R(\ f(s), g_s(a), f(s')\ ) = 0$.} }
    \label{tab: progress-rewards}
\end{table}

In \cref{tab: progress-transitions}, we show that $T(s, a, s') = \widehat T_{\widehat \theta} (\ f(s), g_{s}(a), f(s) \ )$ for all $s \in \mathcal{S}, a \in \mathcal{A}, s' \in \mathcal{S}$. In \cref{tab: progress-rewards}, we show that $R(s, a, s') = \widehat R_{\widehat \theta} (\ f(s), g_{s}(a), f(s) \ )$ for all $s \in \mathcal{S}, a \in \mathcal{A}, s' \in \mathcal{S}$. As a result, we can invoke \cref{thm: human_to_ai_equivalence}.

\end{proof}

\subsection{Proof of equivalence with \emph{multi-chain disengagement worlds}} 
\label{appendix: multi-chain-proof}

\begin{definition}[Multi-chain disengagement worlds]
Let $\multi{\mdpclass}$ denote the class of multi-chain MDPs, so that an MDP $\mdp \in \multi{\mdpclass} = \langle \mathcal{S}, \mathcal{A}, T, R, \gamma, \Delta_\gamma, \Delta_b \rangle$ is defined as follows: 

\begin{itemize}
    \item $s = [s_0, s_2, \ldots, s_C]$, where $s_c \in \{0, \ldots, N_c\}$ denotes the current placement along chain $c$ of length $N_c$. 
    \begin{itemize}
        \item The first chain, $c = 0$, represents the \emph{goal chain}; when the human reaches the end of this chain, they have reached the goal. The set of goal states is $\mathcal{S}_g = \{s \ | \ s_0 = N_0\}$. 
        \item The remaining chains, $s_1, \ldots, s_C$, represent \emph{disengagement chains}. When the human reaches the end of any of these chains, they disengage. The set of disengagement states is $\mathcal{S}_d = \{s \ | \ \exists s_c = N_c \forall c \in \{1, \ldots, C\} \}$ 
    \end{itemize}    
    \item $\mathcal{A} = \{0, 1, 2\}$. The action $a = 1$ allows the human to move along goal chain $c =0$. The action $a = 0$ allows the human to move along the disengagement chain $c > 0$. The action $a = 2$ allows the human recover, by moving backwards on all the disengagement chains. 
    \item $T$. Each chain $c$ is associated with a probability of movement, $p_c$, conditioned on actions as described below: 
    \begin{itemize}
        \item When $a = 0$, the human loses progress in the goal chain with probability $p_\ell$ and independently moves along each disengagement chain $c>0$ with probability $p_c^0$. 
        \item When $a = 1$, the human moves along each chain with probability $p_c^1$. 
        \item When $a = 2$, the human moves backwards on each disengagement chain $c>0$ with probability $1$. The human stays still in the goal chain. 
    \end{itemize}
    \item $R(s, a, s') = \begin{cases}
        r_d & s_c = N_c \text{ for } 0 < 1 \le C \text{ (reached end of a disengagement chain)} \\
        r_g & s_0 = N_0 \text{ (reached end of goal chain)} \\
        r_\ell & s_0' < s_0 \\
        r_b, & a = 1  \text{ or } a = 2
    \end{cases}$, 
    
    where $r_d > 0$, $r_g > 0$, $r_b < 0$, $r_l < 0$. 
    \item $\gamma > 0$
    \item $\Delta = [\Delta_\gamma, \Delta_b]$. AI action $\ai{a} = a_\gamma$ increases $\gamma$ by $\Delta_\gamma$ and AI action $\ai{a} = a_b$ reduces $r_b$ by $\Delta_b$. 
\end{itemize}
\end{definition}

In summary, a multi-chain world $\mdp_\theta \in \multi{\mdpclass}$ is parameterized by $\theta = \{\underbrace{p_\ell, p_1^0, p_2^0, \ldots, p_C^0, p_0^1, p_2^1, \ldots, p_C^1 }_\text{transitions}, \underbrace{r_g, r_d, r_\ell, r_b}_\text{rewards}, \underbrace{\gamma}_\text{discount}, \underbrace{\Delta_\gamma, \Delta_b}_\text{app effect} \}$. 
We prove equivalence for two subsets of multi-chain disengagement worlds, shown in \cref{fig: multi-chain-cases}. 

\begin{figure}[ht]
    \centering
    \begin{subfigure}{0.3\linewidth}
        \centering
        \includegraphics[width=0.95\linewidth]{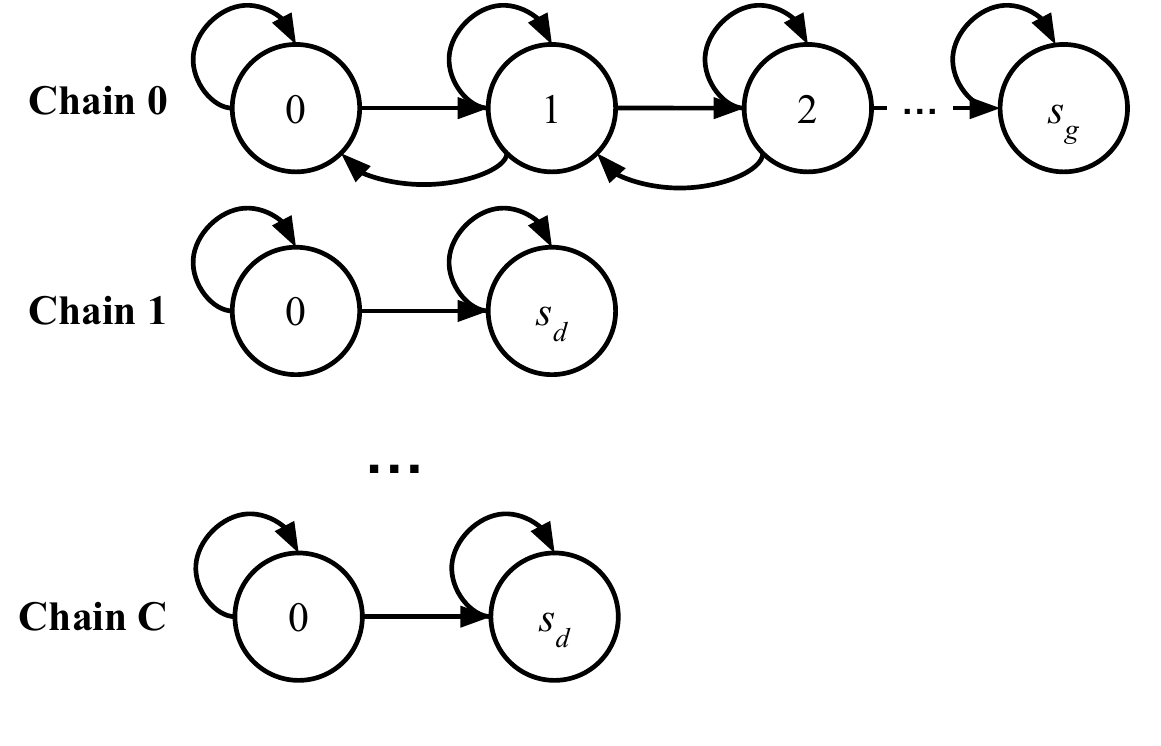}
        \caption{Case A}
        \label{fig: multi-chain-case-A}
    \end{subfigure}%
    \begin{subfigure}{0.3\linewidth}
        \centering
        \includegraphics[width=0.95\linewidth]{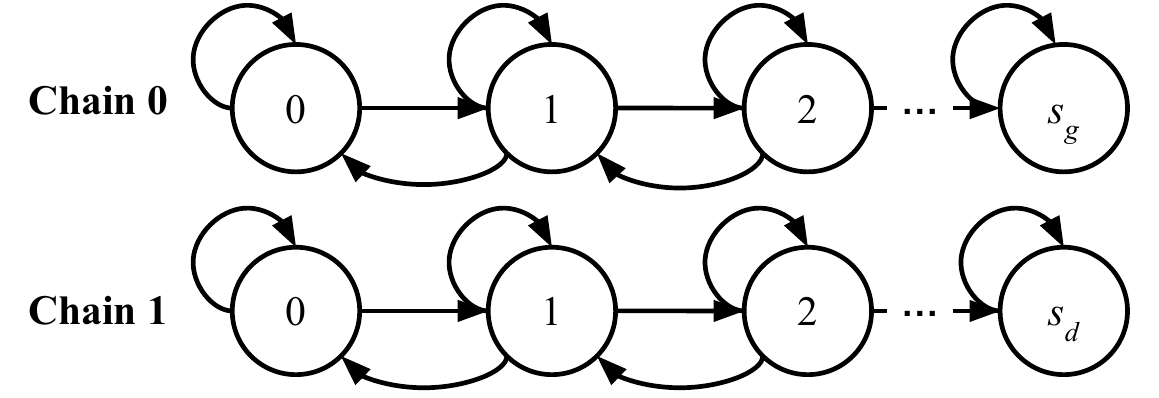}
        \caption{Case B}
        \label{fig: multi-chain-case-B}
    \end{subfigure}%
    \caption{Two types of multi-chain disengagement worlds }
    \label{fig: multi-chain-cases}
\end{figure}

\newpage
\subsubsection{Multi-chain Disengagement Case A}
\label{appendix: multi-chain-case-A}
\begin{theorem}[Chainworld equivalence to case A]
Suppose $\mdp_\theta \in \multi{\mdpclass}$ is a multi-chain MDP like the one shown in \cref{fig: multi-chain-case-A}, where:
\begin{itemize}
    \item There are $C$ chains and the disengagement chains are all of length $2$ ($N_c = 2$ for $c = 1, \ldots, C - 1$). Note that the recovery action $a = 2$ is no longer available in this setting, because the disengagement state is absorbing and the disengagement chains are of length $2$, so that once the human moves along any disengagement chain, they have disengaged and cannnot recover. 
    \item When the agent takes action to move along the goal chain, it stays still in all other disengagement chains. Specifically, movement along all disengagement chains is impossible when $a =1$ so that $p_c^1 = 0$ for all $c = 1, \ldots, N_c$. Movement along the goal chain is still possible, $p_0^1 > 0$. 
    \item When $a=0$, the agent either loses progress or disengages, but not both.
\end{itemize}

Then, there exists a chainworld $\widehat \mdp_{\widehat \theta} \in \simplechains$ such that $\mdp_\theta \equiv \widehat \mdp_{\widehat \theta}$ under state mapping, 

$$f(s = [s_0, s_1, \ldots, s_{C - 1}]) =
\begin{cases}
    \widehat s_{s_0}, & s_0 < N_0 \text{ and } s_c = 0 \text{ for all } c = 1, \ldots, C-1\\
    \widehat s_d, & s_c = 1 \text{ for any } c \ge 1\\
    \widehat s_g & s_0 = N_0
\end{cases}
$$
and action mapping $g_s(a) = \indicator{a > 0}.$
        
\end{theorem}
\begin{proof}
Consider the following chainworld parameters $\widehat \theta$: 
\begin{itemize}
    \item Length of chain $\widehat N = N_0$
    \item Goal reward $\widehat r_g = r_g$
    \item Disengagement reward $\widehat r_d = r_d$
    \item Progress loss reward $\widehat r_\ell = r_\ell$
    \item Burden reward $\widehat r_b = r_b$
    \item Probability of moving toward goal $\widehat p_g = p_0^1$
    \item Probability of losing progress $\widehat p_\ell = p_\ell$
    \item Probability of disengagement $\widehat p_d = 1 - \prod \limits_{c = 1}^C (1 - p_c^0)$
    \item Probability of disengagement at state $0$ $\widehat p_{d0} = 1 - \prod \limits_{c = 1}^C (1 - p_c^0)$
    \item Discount factor $\widehat \gamma = \gamma$
    \item Effect of AI intervention on discount $\widehat{\Delta}_\gamma = \Delta_\gamma$
    \item Effect of AI intervention on burden $\widehat{\Delta}_b = \Delta_b$
\end{itemize}

\begin{table}[ht]
    \centering
    \begin{tabular}{p{1in} | c c c | c c c | c c}
         Notes & $s$ & $a$ & $s'$ & $f(s)$ & $g_s(a)$ & $f(s')$ & $T(s, a, s')$ & $\widehat T(\ f(s), g_s(a), f(s')\ )$\\ \toprule
         $n  = \{0, \ldots, N_0 - 1\}$ & $[n, 0, 0, \ldots, 0]$ & $1$ & $[n, 0, 0, \ldots, 0]$ & $\widehat s_n$ & $1$ & $\widehat s_{n}$ & $1 - p_0^1$ & $1 - \widehat p_g$\\
         $n  = \{0, \ldots, N_0 - 1\}$ & $[n, 0, 0, \ldots, 0]$ & $1$ & $[n+1, 0, 0, \ldots, 0]$ & $\widehat s_n$ & $1$ & $\widehat s_{n+1}$ & $p_0^1$ & $\widehat p_g$\\ \midrule
         $n  = \{1, \ldots, N_0 - 1\}$ & $[n, 0, 0, \ldots, 0]$ & $0$ & $[n, 0, 0, \ldots, 0]$ & $\widehat s_n$ & $0$ & $\widehat s_{n}$ & $ \prod \limits_{c = 1}^C (1 - p_c^0) - p_\ell $ & $1 - \widehat p_\ell - \widehat p_d$\\
         $n  = \{1, \ldots, N_0 - 1\}$ & $[n, 0, 0, \ldots, 0]$ & $0$ & $[n, \ldots, s_c, \ldots], \exists s_c = 1$ & $\widehat s_n$ & $0$ & $\widehat s_{d}$ & $1 - \prod \limits_{c = 1}^C (1 - p_c^0)$ & $\widehat p_d$\\
         $n  = \{1, \ldots, N_0 - 1\}$ & $[n, 0, 0, \ldots, 0]$ & $0$ & $[n - 1, 0, 0, \ldots, 0]$ & $\widehat s_n$ & $0$ & $\widehat s_{n - 1}$ & $p_\ell$ & $\widehat p_\ell$\\ \midrule
         & $[0, 0, \ldots, 0]$ & $0$ & $[0, 0, \ldots, 0]$ & $\widehat s_0$ & $0$ & $\widehat s_0$ & $\prod \limits_{c = 1}^C (1 - p_c^0)$ & $1 - \widehat p_d$\\
         & $[0, 0, \ldots, 0]$ & $0$ & $[n, \ldots, s_c, \ldots], \exists s_c = 1$& $\widehat s_0$ & $0$ & $\widehat s_d$ & $1 - \prod \limits_{c = 1}^C (1 - p_c^0)$ & $\widehat p_d$\\
    \end{tabular}
    \caption{Equivalence of multi-chain case A transitions. \textnormal{All possible \emph{multi-chain world} transitions in $T$ are equivalent to the \emph{chainworld} transitions in $\widehat T$ under mappings $f$ and $g$. Transitions $T(s, a, s')$ with probability $0$ are not shown; these are clearly still $0$ probability under $\widehat T(\ f(s), g_s(a), f(s')\ )$, since the grouped rows sum to $1$ for both $T$ and $\widehat T$.}}
    \label{tab: multi-chain-A-transitions}
\end{table}

\begin{table}[ht]
    \centering
    \begin{tabular}{p{1in} | c c c | c c c | c c}
         Notes & $s$ & $a$ & $s'$ & $f(s)$ & $g_s(a)$ & $f(s')$ & $R(s, a)$ & $\widehat R(\ f(s), g_s(a)\ )$\\ \toprule
         & $[N_0, \ldots]$ & --- & --- & $\widehat s_g$ & --- & --- & $r_g$& $\widehat r_g$\\
         & $[s_0, \ldots, s_c, \ldots], \exists s_c = 1$ & --- & --- & $\widehat s_d$ & --- & --- & $r_d$& $\widehat r_d$\\
         & $[n, 0, 0, \ldots, 0]$ & --- & $[n-1, 0, 0, \ldots, 0]$ & $\widehat s_n$ & 0 & $\widehat s_{n - 1}$ & $r_\ell$& $\widehat r_\ell$\\
         & --- & $1$ & --- & --- & $1$ & --- & $r_b$ & $\widehat r_b$\\
    \end{tabular}
    \caption{Equivalence of multi-chain case A rewards. \textnormal{All possible \emph{multi-chain world} rewards in $R$ are equivalent to the \emph{chainworld} rewards in $\widehat R$ under mappings $f$ and $g$. We use ``---'' to represent any action or state. For all other $s, a, s'$ combinations not shown, $R(s, a, s') = \widehat R(\ f(s), g_s(a), f(s')\ ) = 0$. }}
    \label{tab: multi-chain-A-rewards}
\end{table}

In \cref{tab: multi-chain-A-transitions}, we show that $T(s, a, s') = \widehat T_{\widehat \theta} (\ f(s), g_{s}(a), f(s) \ )$ for all $s \in \mathcal{S}, a \in \mathcal{A}, s' \in \mathcal{S}$. In \cref{tab: multi-chain-A-rewards}, we show that $R(s, a, s') = \widehat R_{\widehat \theta} (\ f(s), g_{s}(a), f(s) \ )$ for all $s \in \mathcal{S}, a \in \mathcal{A}, s' \in \mathcal{S}$. As a result, we can invoke \cref{thm: human_to_ai_equivalence}. 
\end{proof}

\newpage
\subsubsection{Multi-chain Disengagement Case B}
This world represents situations in which the human makes progress toward a goal (e.g. a rehabilitated shoulder), but must also manage an additional factor that may cause them to disengage (e.g. exercise fatigue). To manage this, the human has an additional action, $a = 2$, which allows them to recover from fatigue in exchange for not making progress toward the goal (e.g. taking a rest day).

\begin{theorem}[Chainworld equivalence to case B]
If $\mdp_\theta \in \multi{\mdpclass}$ is a multi-chain MDP with $C = 2$, a goal chain ($c = 0$) and a disengagement chain ($c = 1$), rewards $r_d = 0, r_\ell = 0$, and transitions: 
\begin{itemize}
    \item When $a=0$, the agent stays still in the goal chain ($p_\ell = 0$) and always moves along the disengagement chain ($p_1^0 = 1$)
    \item When $a = 1$, the agent deterministically makes progress along both chains, $p_0^1 = p_1^1 = 1$. 
    \item When $a=2$, the agent deterministically moves backward on the disengagement chain. 
\end{itemize}

then there exists $\widehat \mdp_{\widehat \theta} \in \simplechains$ such that $\mdp_\theta \equiv \widehat \mdp_{\widehat \theta}$ under state mapping, 

$$f(s = [s_0, s_1]) = 2 N_0 - N_0 + s_0 - \indicator{N_0 - s_0 > N_1 - s_1}(N_0 - s_0 - N_1 + s_1)$$, 

and action mapping $g_s(a) = \indicator{a > 0}.$
\end{theorem}
\begin{proof}
Consider the following chainworld parameters $\widehat \theta$: 
\begin{itemize}
    \item Length of chain $\widehat N = 2 N_0$
    \item Goal reward $\widehat r_g = r_g$
    \item Disengagement reward $\widehat r_d = r_d$
    \item Progress loss reward $\widehat r_\ell = 0$
    \item Burden reward $\widehat r_b = r_b$
    \item Probability of moving toward goal $\widehat p_g = 1$
    \item Probability of losing progress $\widehat p_\ell = 1$
    \item Probability of disengagement $\widehat p_d = 0$
    \item Probability of disengagement at state $0$, $\widehat p_{d0} = 1$
    \item Discount factor $\widehat \gamma = \gamma$
    \item Effect of AI intervention on discount $\widehat{\Delta}_\gamma = \Delta_\gamma$
    \item Effect of AI intervention on burden $\widehat{\Delta}_b = \Delta_b$
\end{itemize}

\begin{table}[ht]
    \centering
    \begin{tabular}{p{1.9in} | c c c | c c c | c c}
         Notes & $s$ & $a$ & $s'$ & $f(s)$ & $g_s(a)$ & $f(s')$ & $T(s, a, s')$ & $\widehat T(\ f(s), g_s(a), f(s')\ )$\\ \toprule
         Movement along chain $0$ and chain $1$ & $[s_0, s_1]$ & $1$ & $[s_0 + 1, s_1 + 1]$ & $\widehat s_{f(s)}$ & $1$ & $\widehat s_{f(s) + 1}$ & $1$ & $\widehat p_g = 1$\\
         Movement along only chain $1$ & $[s_0, s_1]$ & $0$ & $[s_0, s_1 + 1]$ & $\widehat s_{f(s)}$ & $0$ & $\widehat s_{f(s) - 1}$ & $1$ & $\widehat p_\ell = 1$\\
         Reaching end of chain $1$ & $[s_0, N_1 - 1]$ & $0$ & $[s_0, N_1]$ & $\widehat s_{f(s)}$ & $0$ & $\widehat s_{f(s) - 1}$ & $1$ & $\widehat p_{d0} = 1$\\
         Backwards along chain $1$, for $s_1 > 0$ & $[s_0, s_1]$ & $2$ & $[s_0, s_1 - 1]$ & $\widehat s_{f(s)}$ & $1$ & $\widehat s_{f(s) + 1}$ & $1$ & $\widehat p_g = 1$\\
    \end{tabular}
    \caption{Equivalence of multi-chain case B transitions. \textnormal{All possible \emph{multi-chain} transitions in $T$ are equivalent to the \emph{chainworld} transitions in $\widehat T$ under mappings $f$ and $g$. Transitions $T(s, a, s')$ with probability $0$ are not shown; these are clearly still $0$ probability under $\widehat T(\ f(s), g_s(a), f(s')\ )$, since all rows are $1$ for both $T$ and $\widehat T$}.}
    \label{tab: multi-chain-B-transitions}
\end{table}

\begin{table}[ht]
    \centering
    \begin{tabular}{p{1in} | c c c | c c c | c c}
         Notes & $s$ & $a$ & $s'$ & $f(s)$ & $g_s(a)$ & $f(s')$ & $R(s, a)$ & $\widehat R(\ f(s), g_s(a)\ )$\\ \toprule
         & --- & --- & $[N_0, s_1]$ & $\widehat s_{g}$ & --- & --- & $r_g$ & $\widehat r_g$\\
         & --- & --- & $[s_0, N_1]$ & $\widehat s_{d}$ & --- & --- & $r_d$ & $\widehat r_d$\\
         & --- & $2$ & --- & --- & $1$ & --- & $r_b$ & $\widehat r_b$\\
         & --- & $1$ & --- & --- & $1$ & --- & $r_b$ & $\widehat r_b$\\
    \end{tabular}
    \caption{Equivalence of multi-chain case B rewards. \textnormal{All possible \emph{multi-chain} rewards in $R$ are equivalent to the \emph{chainworld} rewards in $\widehat R$ under mappings $f$ and $g$. We use ``---'' to represent any action or state. For all other $s, a, s'$ combinations not shown, $R(s, a, s') = \widehat R(\ f(s), g_s(a), f(s')\ ) = 0$.} }
    \label{tab: multi-chain-B-rewards}
\end{table}

In \cref{tab: multi-chain-B-transitions}, we show that $T(s, a, s') = \widehat T_{\widehat \theta}(\ f(s), g_{s}(a), f(s) \ )$ for all $s \in \mathcal{S}, a \in \mathcal{A}, s' \in \mathcal{S}$. In \cref{tab: multi-chain-B-rewards}, we show that $R(s, a, s') = R_{\widehat \theta} (\ f(s), g_{s}(a), f(s) \ )$ for all $s \in \mathcal{S}, a \in \mathcal{A}, s' \in \mathcal{S}$. As a result, we can invoke \cref{thm: human_to_ai_equivalence}. 

\end{proof}

\section{Experimental Details}
\label{appendix: experimental-details}

\subsection{Environment descriptions}
\label{appendix: environment-descriptions}
Throughout the experiments, we fix (do not sample) the following parameters per individual, to make the methods easier to compare: $p_g = 1, \Delta_\gamma = 0.3, \Delta_b = -0.4$. All methods (excluding the oracle) do not have access to any of these parameters and must infer them. 

\vskip0.15cm \emph{Chainworld environment}
The chainworld environment is described in the main body of the text. Every individual's parameters are sampled uniformly from the following ranges: 
\begin{itemize}
    \item $r_b: [-1, -0.2]$
    \item $r_d: [0, 1]$
    \item $r_\ell: [-1, 0]$
    \item $r_g: [5, 15]$
    \item $\human{\gamma}: [0.01, 0.99]$
    \item $p_d: [0.1, 0.5]$
    \item $p_{d0}: [p_d, 0.5]$
    \item $p_\ell: [0, 0.4]$
\end{itemize}

\vskip0.15cm \emph{Noisy parameters experiments.}
The \emph{mean} parameter value for each chainworld human is sampled as in the standard chainworld environment above. Then, every timestep, the parameter of interest is sampled uniformly from a range surrounding this mean, as described in the main body of the text. For example, if the environment is testing sensitivity to noise in burden $r_b$, then every timestep, $r_b \sim \text{Uniform}(\bar r_b - c, \bar r_b + c)$, where $\bar r_b$ is the mean burden and $c = \epsilon 5$ is the range such that $\epsilon \in [0, 1]$ is the error level and $5$ is the range assigned to the reward parameters. 

\vskip0.15cm \emph{Distance mapping (gridworld) experiments.}
The gridworld environment is described in the main body of the text. The gridworld has width $X$ and height $Y$. Every individual's parameters are sampled uniformly from the following ranges: 
\begin{itemize}
    \item $r_b: [-1, -0.2]$
    \item $\human{\gamma}: [0.01, 0.99]$
    \item $p: [0.5, 1.0]$
    \item $r_g: [5 \frac{X}{8}, 10 \frac{X}{8}]$
    \item $r_d: [0, \frac{Y}{5}]$
\end{itemize}

We scale the values of the rewards to the size of the gridworld. 

\subsection{Definition of AI agent}
The AI actions are always $\{0, a_\gamma, a_b\}$ for all experiments and the transition is computed directly off of human agent transitions. 

\vskip0.15cm \emph{AI states.}
The AI states $\ai{s} = [\human{s}, \human{a}]$ are composed of the human's current state and previous reward. An AI that uses a chainworld has state space of size $N \times 2$. An AI that plans directly in the gridworld has state space of size $X \times Y \times 4$.

\vskip0.15cm \emph{AI rewards.}
In all experiments, the rewards are as follows: 
\begin{equation}
    \ai{R}(\ai{a}, \ai{a}) = 
    \begin{cases}
        1, & \human{s} = s_g \\
        -50, & \human{s} = s_d \\
        -1, & \ai{a} \ne 0 \\
        -0.5, & \text{otherwise}
    \end{cases}
\end{equation}

\vskip0.15cm \emph{AI discount}
We use an AI agent discount of $\ai{\gamma} = 0.99$.

\subsection{Optimizing chainworld parameters}
\label{appendix: optimizing-chainworld-parameters}
The AI agent must infer the chainworld parameters $\theta$ from the data $\ai{\mathcal{D}} = \{(\ai{s}, \ai{a}, \ai{r}, \ai{s'})\}$. Maximizing the likelihood of the data corresponds to maximizing the likelihood of the observed transitions, since the chainworld parameters are all contained within the AI's transition function: 
$$P(\ai{\mathcal{D}}\ |\ \theta ) = \ai{T}^\theta(\ai{s}, \ai{a}, \ai{s'}).$$

We follow a simple maximization scheme in which we randomly sample possible values of $\theta$ and select the one with the highest likelihood. The candidate $\theta$'s are sampled uniformly from the following ranges: 
\begin{itemize}
    \item $r_b: [-1, 0]$
    \item $r_d: [0, 5]$
    \item $r_\ell: [-5, 0]$
    \item $r_g: [5, 50]$
    \item $\human{\gamma}: [0.01, 0.99]$
    \item $p_g: [0, 1]$
    \item $p_\ell$: see below description
    \item $p_d$: see below description
    \item $p_{d0}: [0, 1]$
    \item $\tau: [0.01, 0.3]$
    \item $\Delta_\gamma: [0, 1]$
    \item $\Delta_b: [-1, 0]$
\end{itemize}

The parameters $p_d$ and $p_\ell$ are constrainted so that $p_d + p_\ell \le 1$. To sample them so that they respect this constraint, we sample them uniformly from a triangle whose vertices are at $[0,0], [0,1], [1,0]$. The parameter $\tau$ refers to the noise level in the softmax action selection policy.

\section{Additional experiments / results}
\label{appendix: supplemental-results}
\subsection{Effect of learning rate for model-free baseline}
In our experiments, the model-free baseline is given a learning rate of $0.9$. Throughout the results, the model-free baseline performs poorly-- equal to using a random AI policy. This is because it requires \emph{much more data} to estimate the optimal value function $\ai{Q}^*$ well. As we show in \cref{fig: model-free}, the model-free baseline requires at least $100$ episodes to learn a policy that outperforms random.

\begin{figure}[ht]
    \centering
    \includegraphics[width=0.5\linewidth]{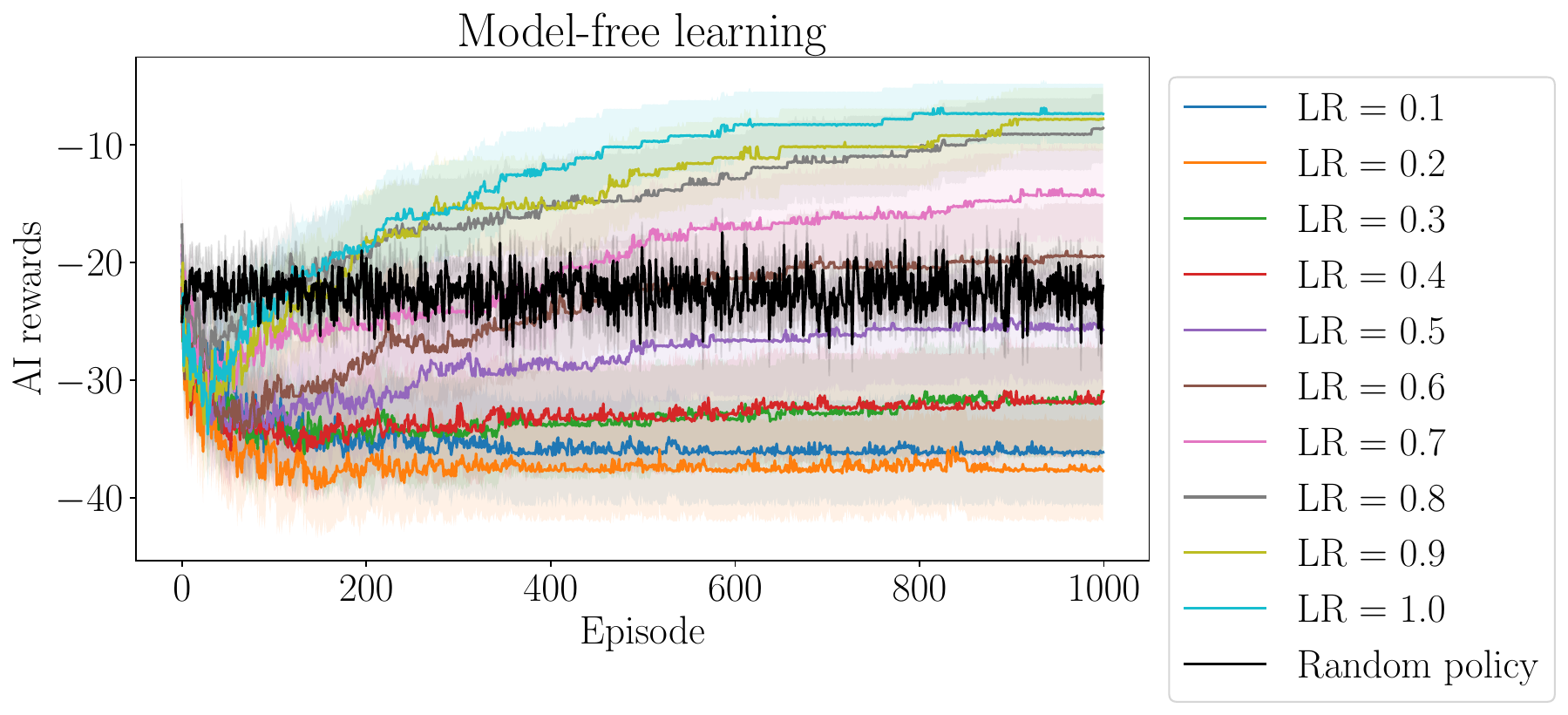}
    \caption{Model-free baseline performance under different learning rates (LR) in the chainworld environment.}
    \label{fig: model-free}
\end{figure}

\subsection{Results with and without filtering for individuals that cannot reach goal}
\label{appendix: helpless-removed}
\begin{figure}[ht]
    \centering
    \begin{subfigure}{0.5\linewidth}
        \includegraphics[width=1\linewidth]{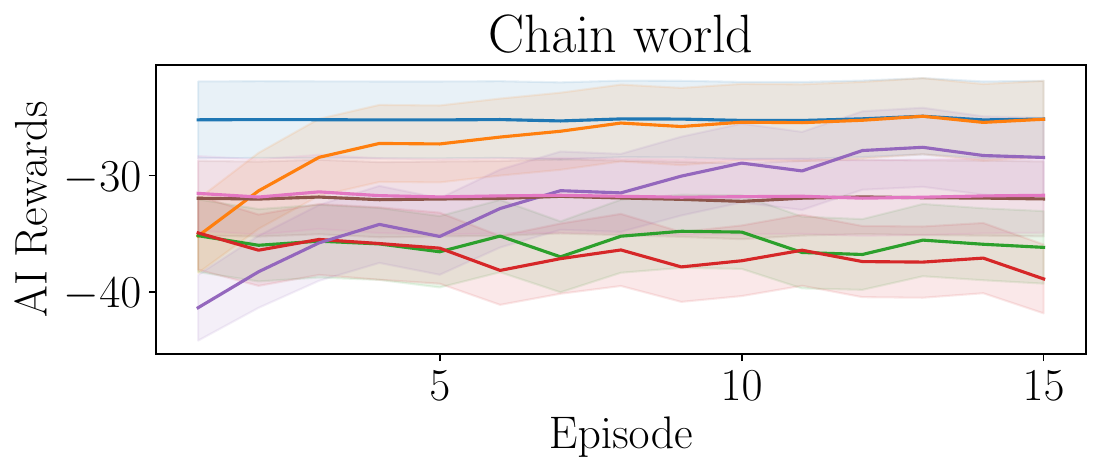}
        \caption{Included humans that cannot reach goal state}
    \end{subfigure}%
    \begin{subfigure}{0.5\linewidth}
        \includegraphics[width=1\linewidth]{figures/res-chain.pdf}
        \caption{Filtered out humans that cannot reach goal state}
    \end{subfigure}%
    \caption{Results with (left) and without (right) individuals that will not reach the goal state under the oracle AI policy.}
    \label{fig: helpless-keep-vs-remove}
\end{figure}

\subsection{Full plots from robustness experiments}
\label{appendix: full-results}
\begin{figure}[ht]
    \centering
    \begin{subfigure}{0.199\linewidth}
        \includegraphics[width=1\linewidth]{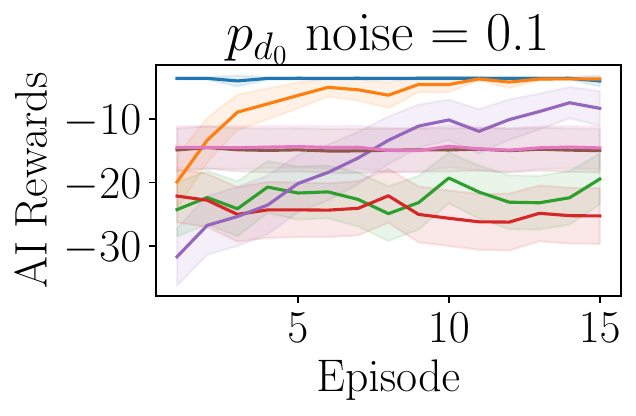}
    \end{subfigure}%
    \begin{subfigure}{0.199\linewidth}
        \includegraphics[width=1\linewidth]{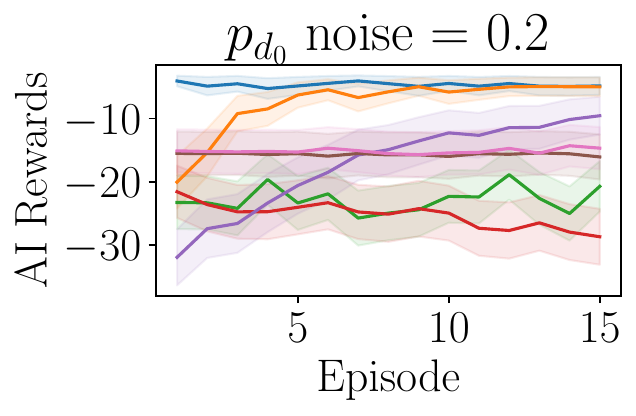}
    \end{subfigure}%
    \begin{subfigure}{0.199\linewidth}
        \includegraphics[width=1\linewidth]{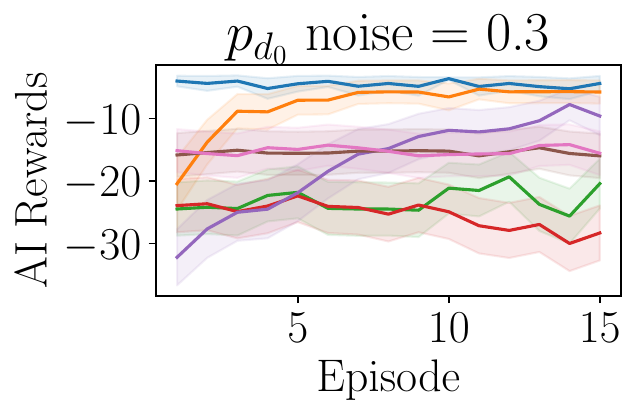}
    \end{subfigure}%
    \begin{subfigure}{0.199\linewidth}
        \includegraphics[width=1\linewidth]{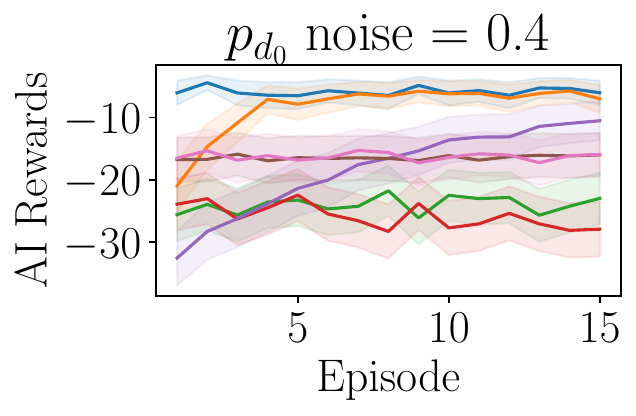}
    \end{subfigure}%
    \begin{subfigure}{0.199\linewidth}
        \includegraphics[width=1\linewidth]{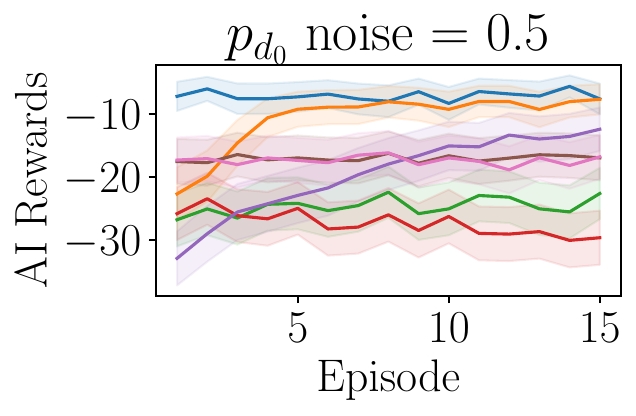}
    \end{subfigure}

    \begin{subfigure}{0.199\linewidth}
        \includegraphics[width=1\linewidth]{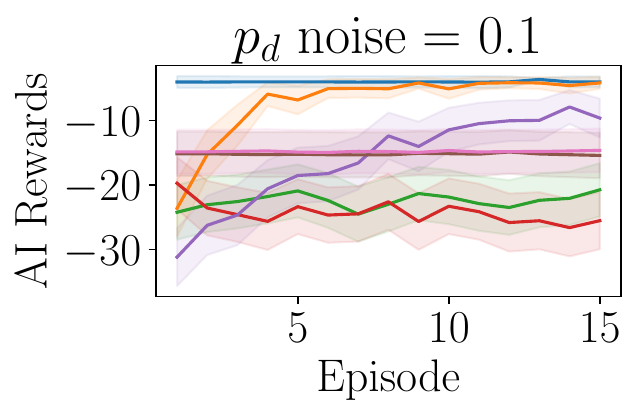}
    \end{subfigure}%
    \begin{subfigure}{0.199\linewidth}
        \includegraphics[width=1\linewidth]{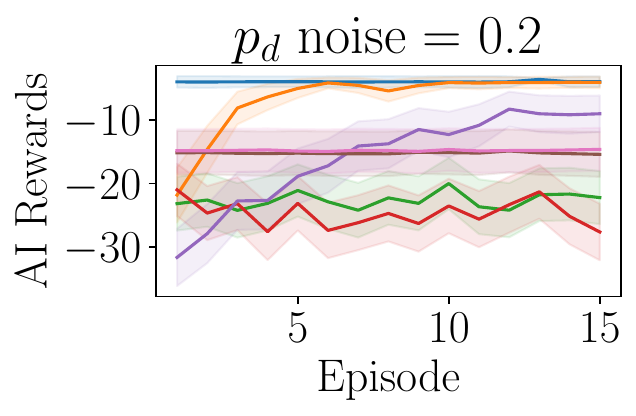}
    \end{subfigure}%
    \begin{subfigure}{0.199\linewidth}
        \includegraphics[width=1\linewidth]{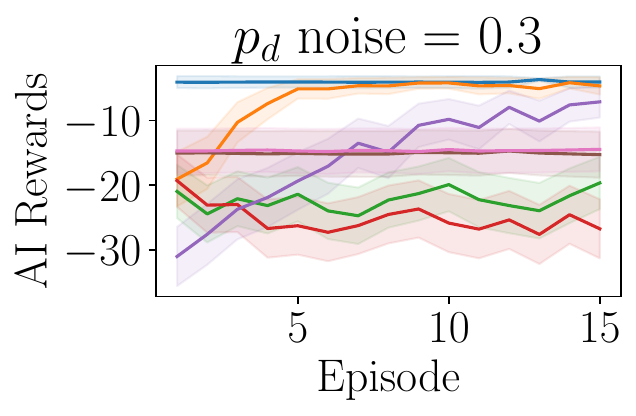}
    \end{subfigure}%
    \begin{subfigure}{0.199\linewidth}
        \includegraphics[width=1\linewidth]{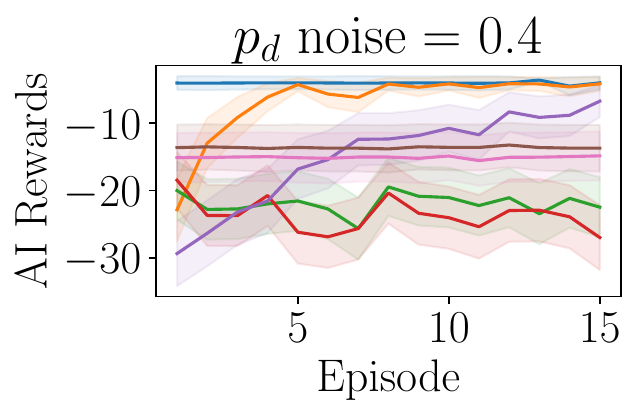}
    \end{subfigure}%
    \begin{subfigure}{0.199\linewidth}
        \includegraphics[width=1\linewidth]{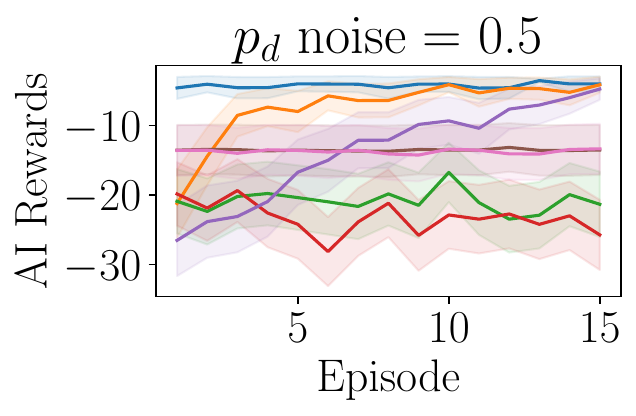}
    \end{subfigure}

    \begin{subfigure}{0.199\linewidth}
        \includegraphics[width=1\linewidth]{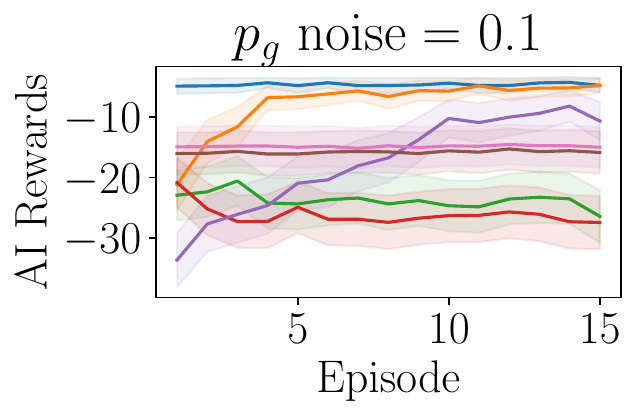}
    \end{subfigure}%
    \begin{subfigure}{0.199\linewidth}
        \includegraphics[width=1\linewidth]{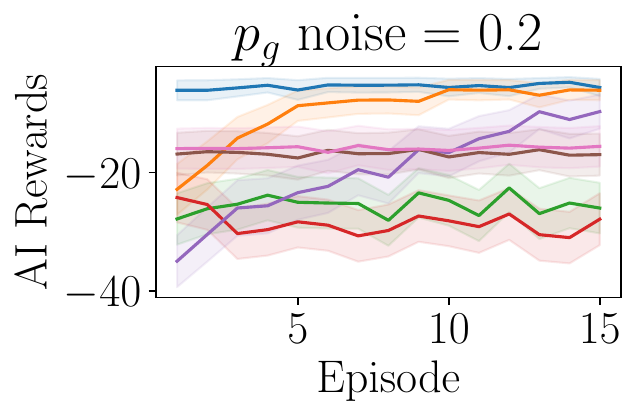}
    \end{subfigure}%
    \begin{subfigure}{0.199\linewidth}
        \includegraphics[width=1\linewidth]{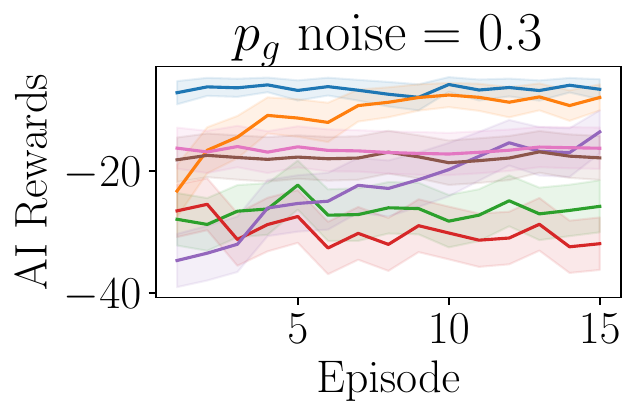}
    \end{subfigure}%
    \begin{subfigure}{0.199\linewidth}
        \includegraphics[width=1\linewidth]{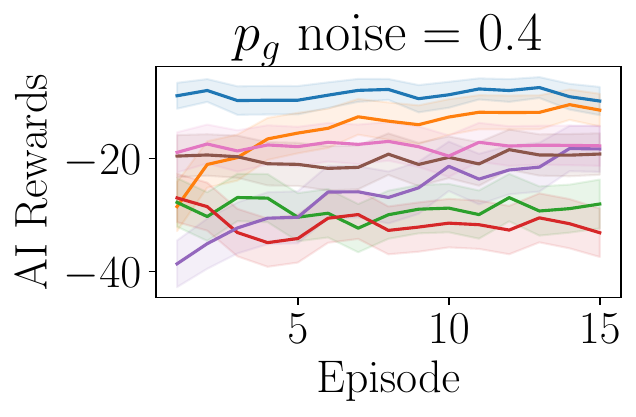}
    \end{subfigure}%
    \begin{subfigure}{0.199\linewidth}
        \includegraphics[width=1\linewidth]{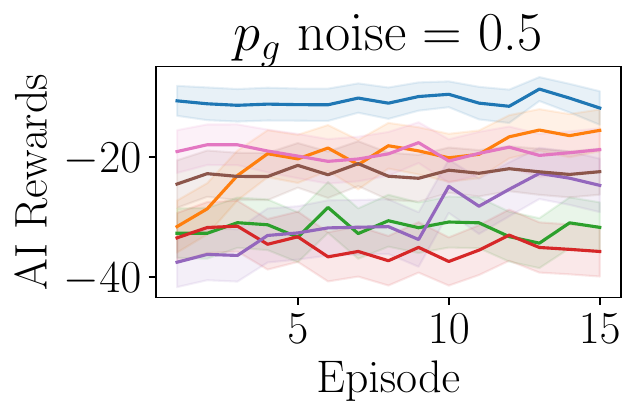}
    \end{subfigure}

    \begin{subfigure}{0.199\linewidth}
        \includegraphics[width=1\linewidth]{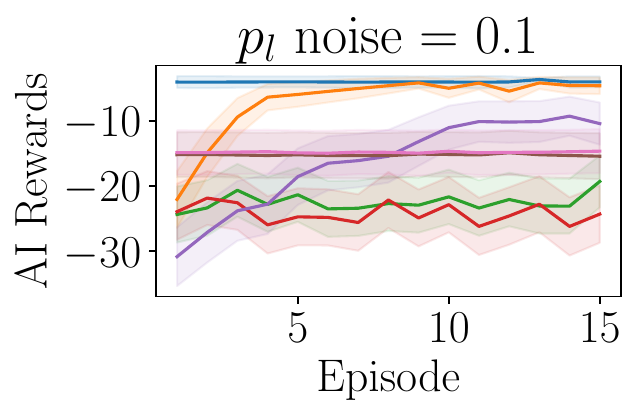}
    \end{subfigure}%
    \begin{subfigure}{0.199\linewidth}
        \includegraphics[width=1\linewidth]{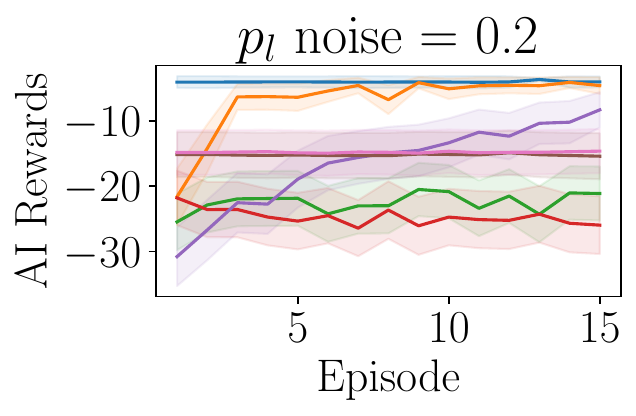}
    \end{subfigure}%
    \begin{subfigure}{0.199\linewidth}
        \includegraphics[width=1\linewidth]{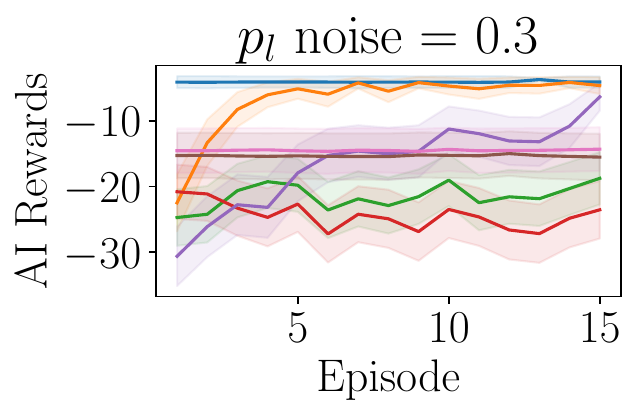}
    \end{subfigure}%
    \begin{subfigure}{0.199\linewidth}
        \includegraphics[width=1\linewidth]{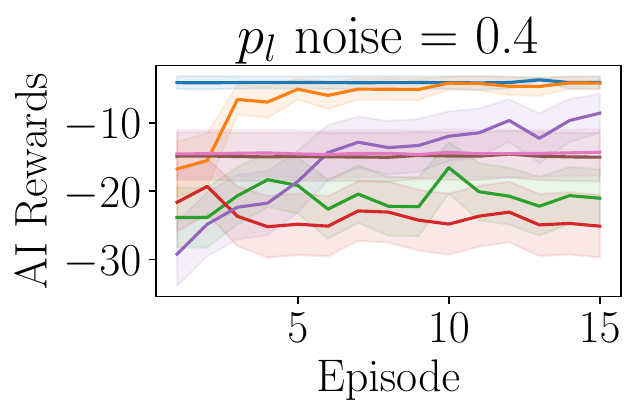}
    \end{subfigure}%
    \begin{subfigure}{0.199\linewidth}
        \includegraphics[width=1\linewidth]{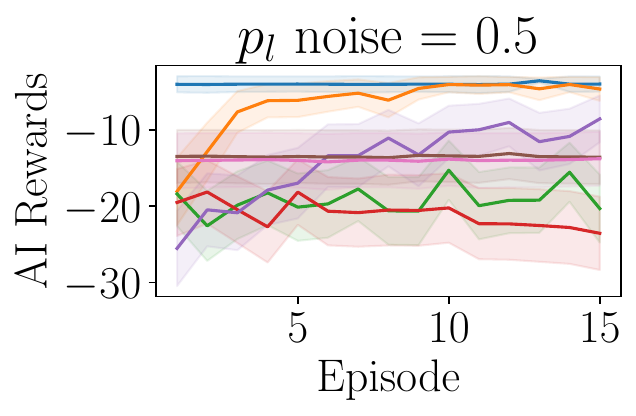}
    \end{subfigure}

    \begin{subfigure}{0.199\linewidth}
        \includegraphics[width=1\linewidth]{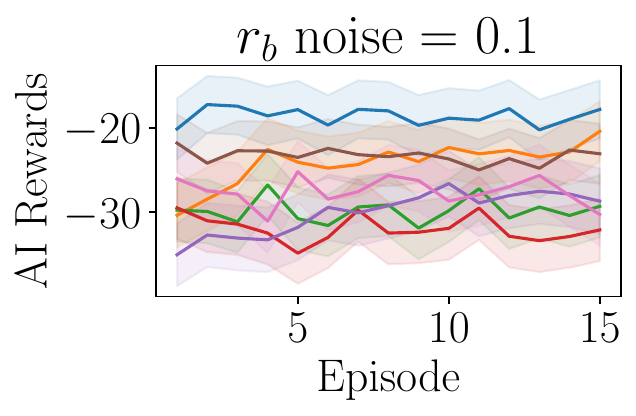}
    \end{subfigure}%
    \begin{subfigure}{0.199\linewidth}
        \includegraphics[width=1\linewidth]{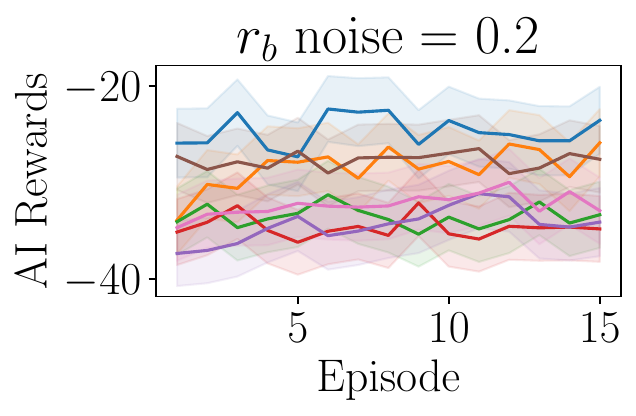}
    \end{subfigure}%
    \begin{subfigure}{0.199\linewidth}
        \includegraphics[width=1\linewidth]{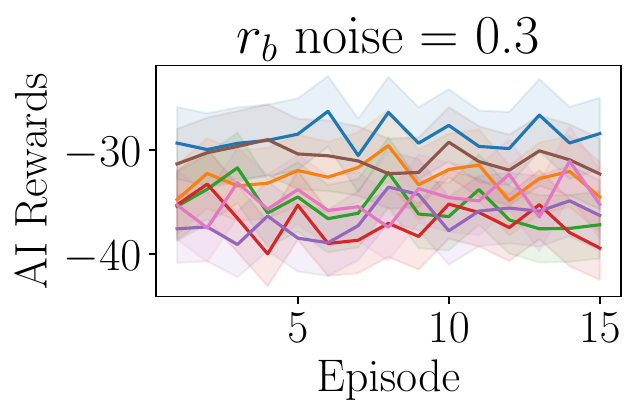}
    \end{subfigure}%
    \begin{subfigure}{0.199\linewidth}
        \includegraphics[width=1\linewidth]{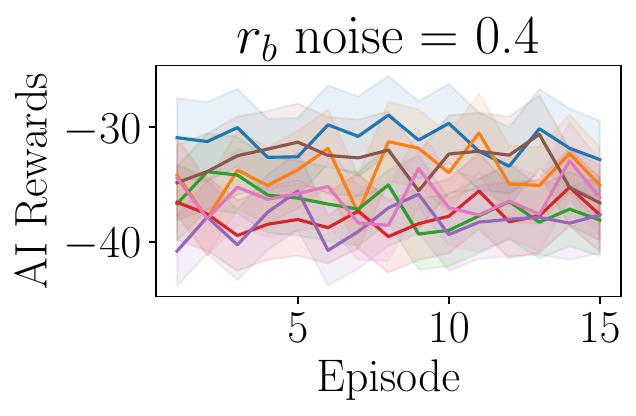}
    \end{subfigure}%
    \begin{subfigure}{0.199\linewidth}
        \includegraphics[width=1\linewidth]{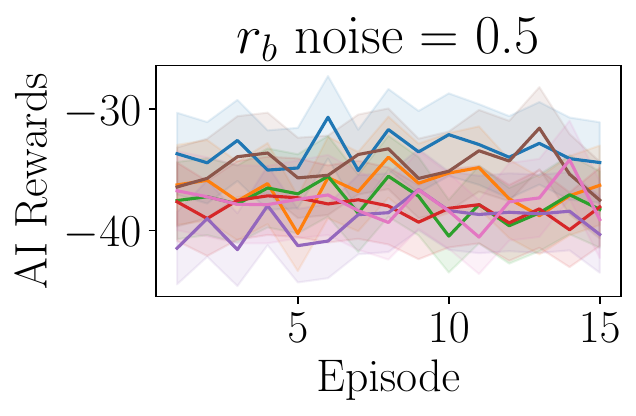}
    \end{subfigure}
    
    \begin{subfigure}{0.199\linewidth}
        \includegraphics[width=1\linewidth]{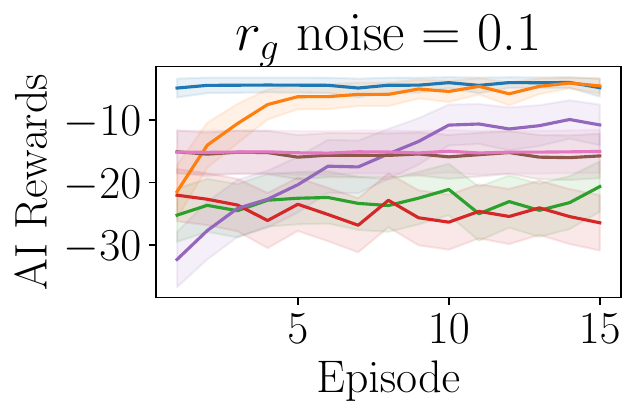}
    \end{subfigure}%
    \begin{subfigure}{0.199\linewidth}
        \includegraphics[width=1\linewidth]{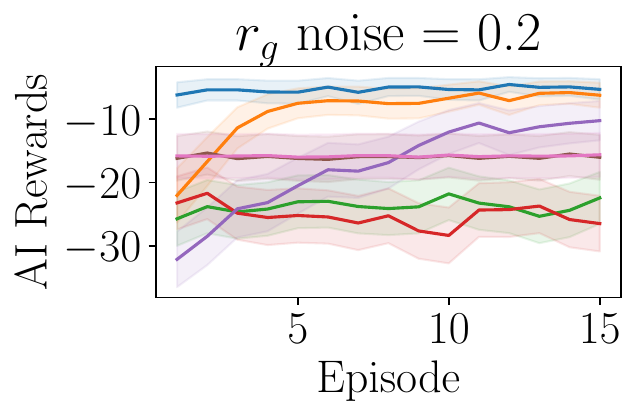}
    \end{subfigure}%
    \begin{subfigure}{0.199\linewidth}
        \includegraphics[width=1\linewidth]{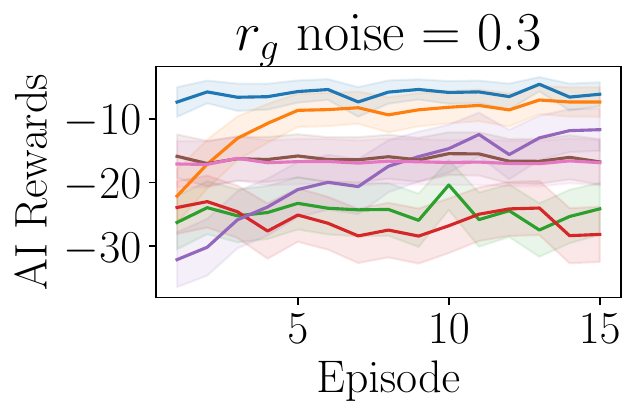}
    \end{subfigure}%
    \begin{subfigure}{0.199\linewidth}
        \includegraphics[width=1\linewidth]{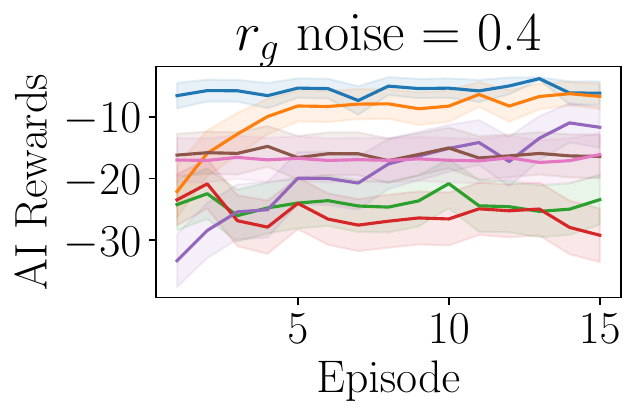}
    \end{subfigure}%
    \begin{subfigure}{0.199\linewidth}
        \includegraphics[width=1\linewidth]{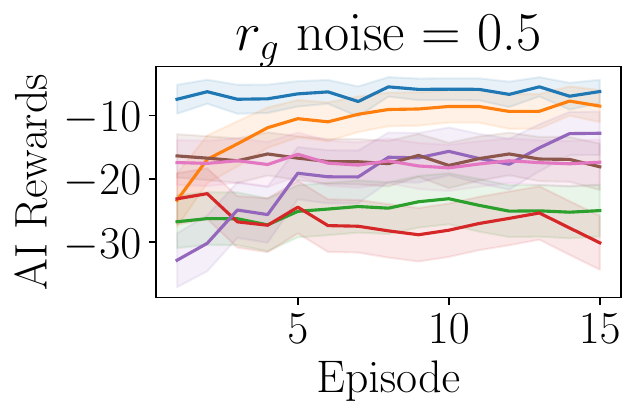}
    \end{subfigure}

    \begin{subfigure}{0.199\linewidth}
        \includegraphics[width=1\linewidth]{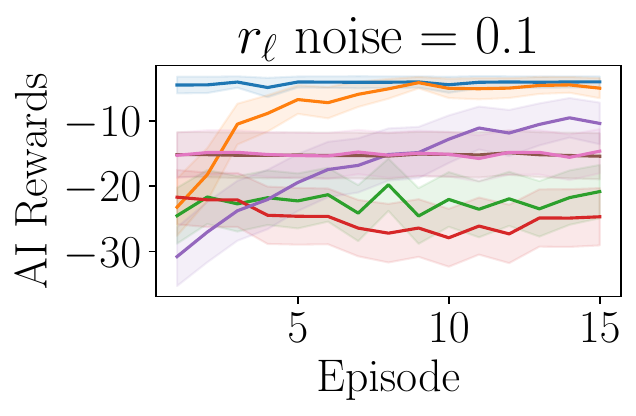}
    \end{subfigure}%
    \begin{subfigure}{0.199\linewidth}
        \includegraphics[width=1\linewidth]{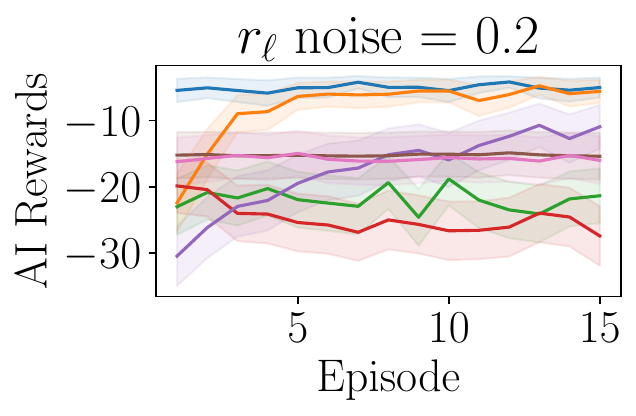}
    \end{subfigure}%
    \begin{subfigure}{0.199\linewidth}
        \includegraphics[width=1\linewidth]{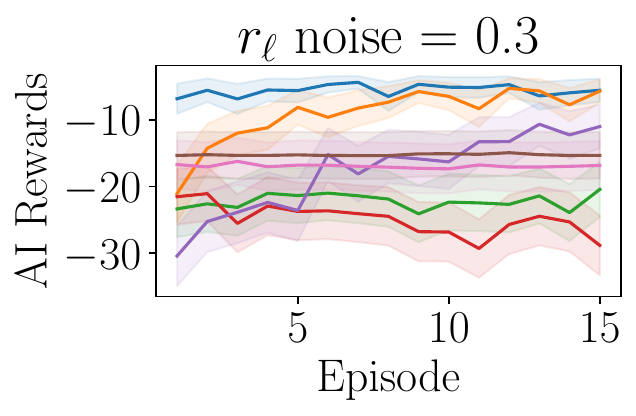}
    \end{subfigure}%
    \begin{subfigure}{0.199\linewidth}
        \includegraphics[width=1\linewidth]{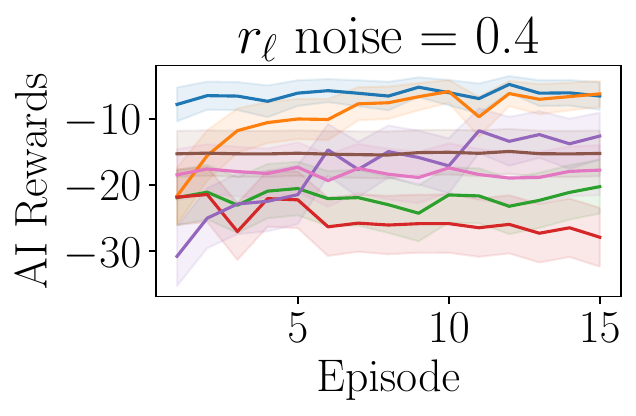}
    \end{subfigure}%
    \begin{subfigure}{0.199\linewidth}
        \includegraphics[width=1\linewidth]{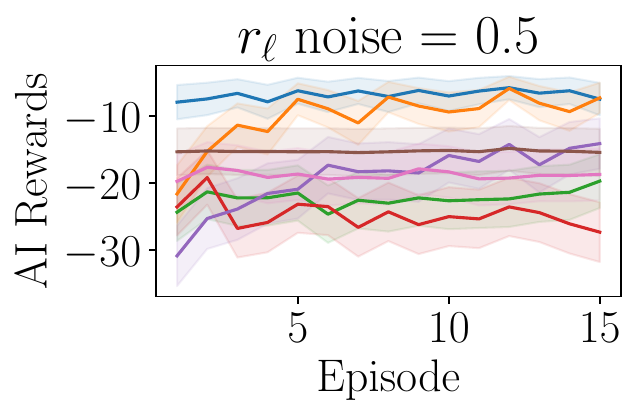}
    \end{subfigure}

    \begin{subfigure}{0.199\linewidth}
        \includegraphics[width=1\linewidth]{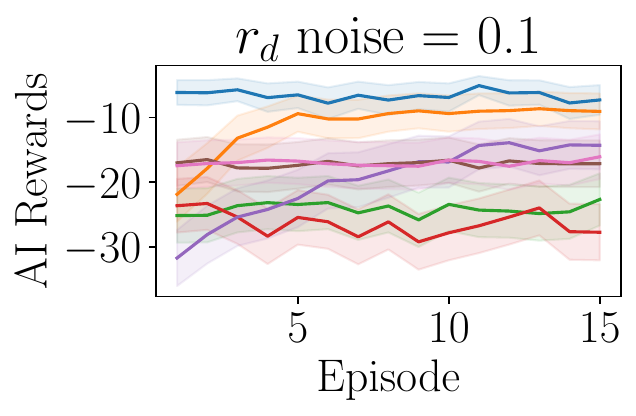}
    \end{subfigure}%
    \begin{subfigure}{0.199\linewidth}
        \includegraphics[width=1\linewidth]{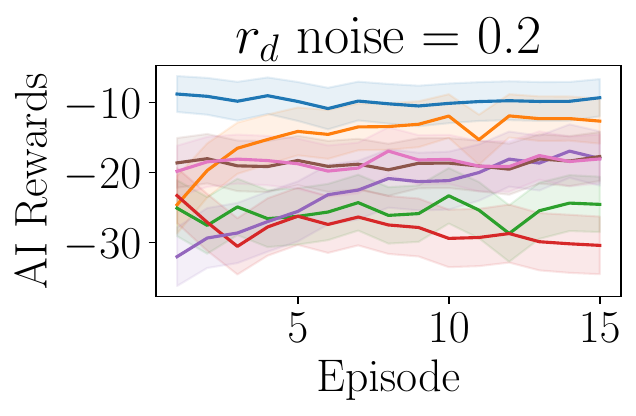}
    \end{subfigure}%
    \begin{subfigure}{0.199\linewidth}
        \includegraphics[width=1\linewidth]{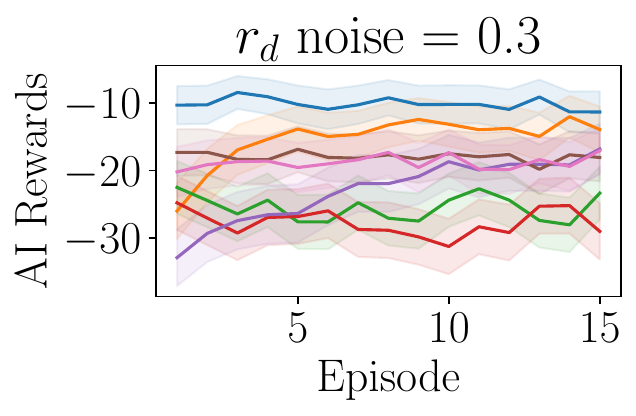}
    \end{subfigure}%
    \begin{subfigure}{0.199\linewidth}
        \includegraphics[width=1\linewidth]{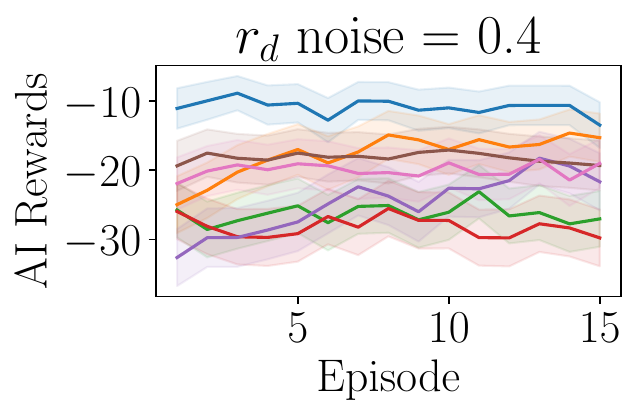}
    \end{subfigure}%
    \begin{subfigure}{0.199\linewidth}
        \includegraphics[width=1\linewidth]{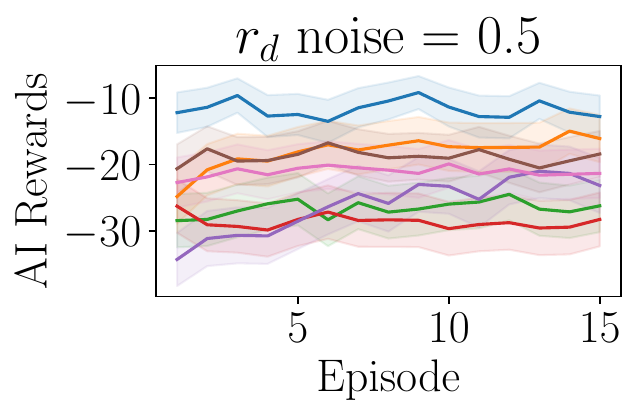}
    \end{subfigure}

    \begin{subfigure}{0.199\linewidth}
        \includegraphics[width=1\linewidth]{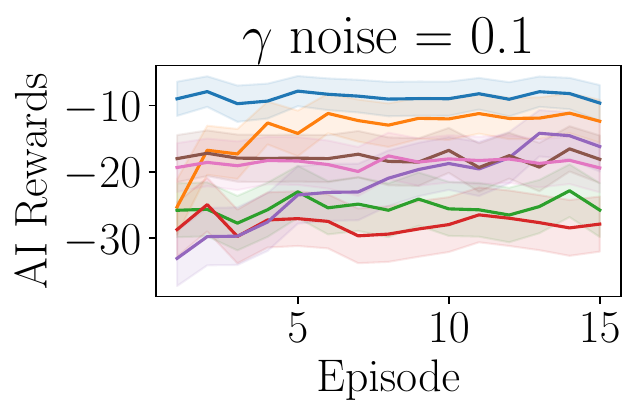}
    \end{subfigure}%
    \begin{subfigure}{0.199\linewidth}
        \includegraphics[width=1\linewidth]{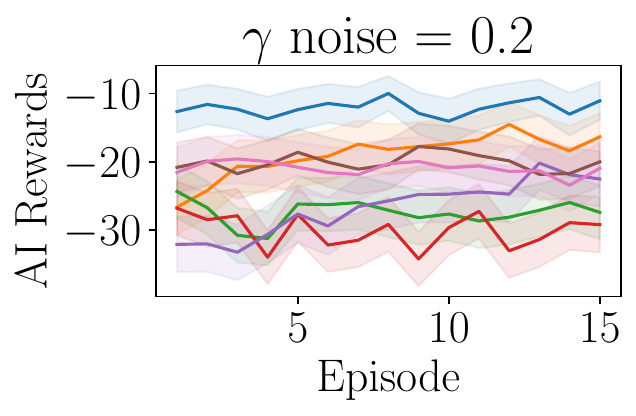}
    \end{subfigure}%
    \begin{subfigure}{0.199\linewidth}
        \includegraphics[width=1\linewidth]{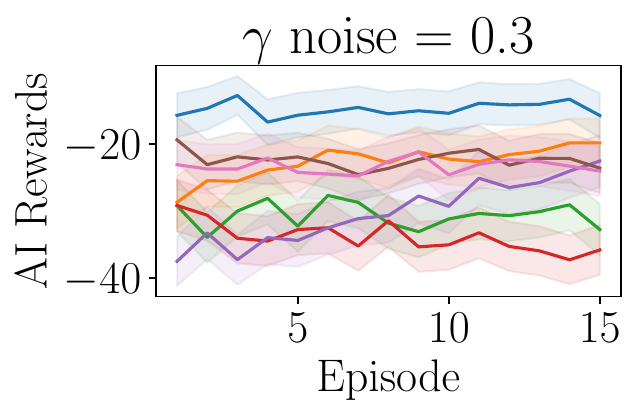}
    \end{subfigure}%
    \begin{subfigure}{0.199\linewidth}
        \includegraphics[width=1\linewidth]{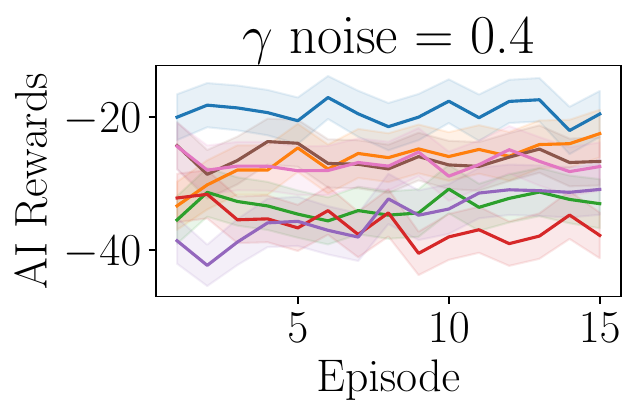}
    \end{subfigure}%
    \begin{subfigure}{0.199\linewidth}
        \includegraphics[width=1\linewidth]{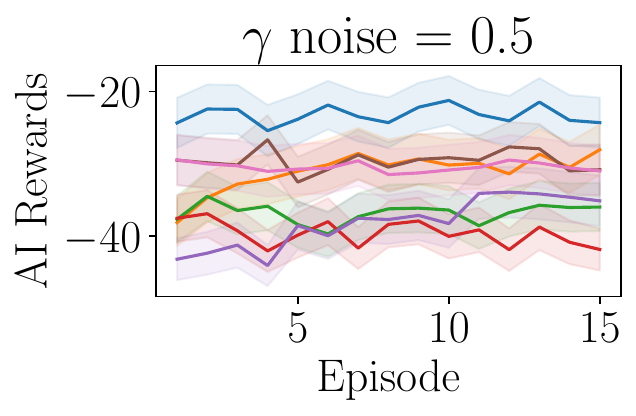}
    \end{subfigure}

    \caption{Full results for environments where parameter varies every timestep.}
\end{figure}

\begin{figure}[ht]
    \centering
    \begin{subfigure}{0.25\linewidth}
        \includegraphics[width=1\linewidth]{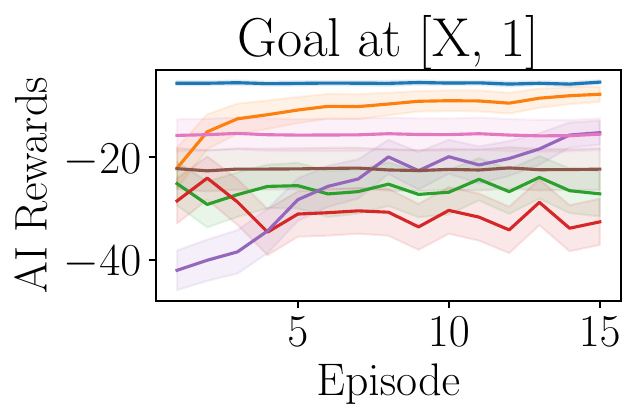}
    \end{subfigure}%
    \begin{subfigure}{0.25\linewidth}
        \includegraphics[width=1\linewidth]{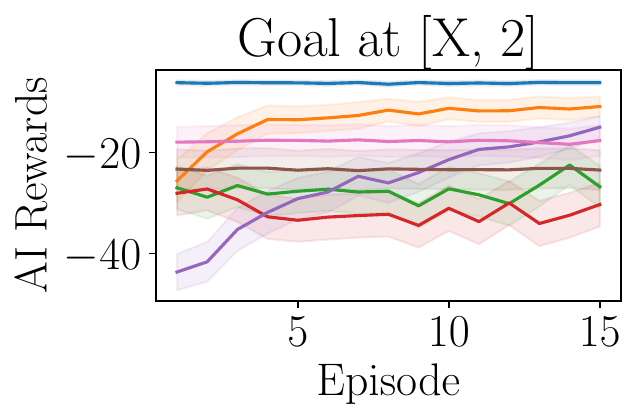}
    \end{subfigure}%
    \begin{subfigure}{0.25\linewidth}
        \includegraphics[width=1\linewidth]{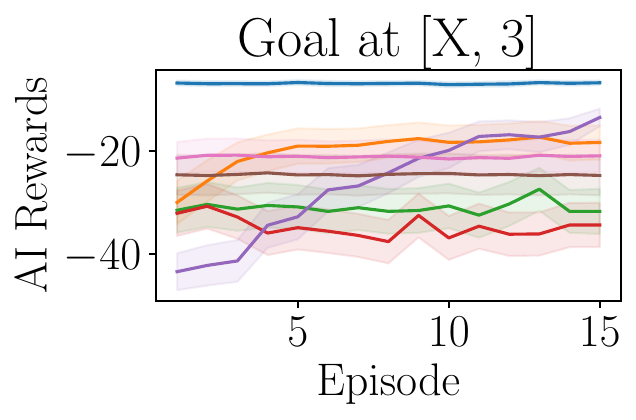}
    \end{subfigure}%
    \begin{subfigure}{0.25\linewidth}
        \includegraphics[width=1\linewidth]{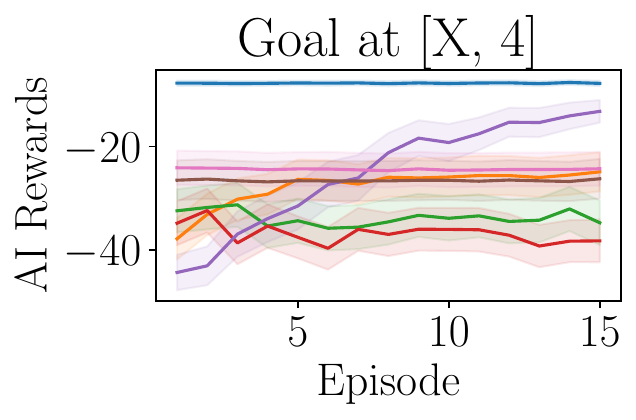}
    \end{subfigure}
    
    \begin{subfigure}{0.25\linewidth}
        \includegraphics[width=1\linewidth]{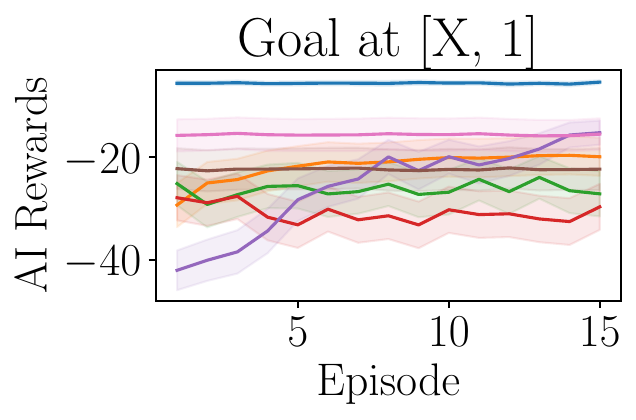}
    \end{subfigure}%
    \begin{subfigure}{0.25\linewidth}
        \includegraphics[width=1\linewidth]{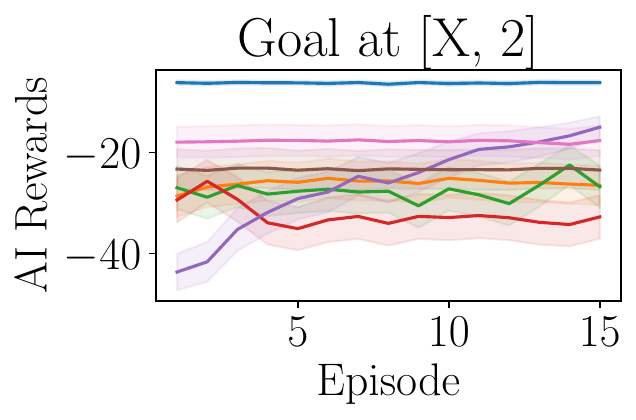}
    \end{subfigure}%
    \begin{subfigure}{0.25\linewidth}
        \includegraphics[width=1\linewidth]{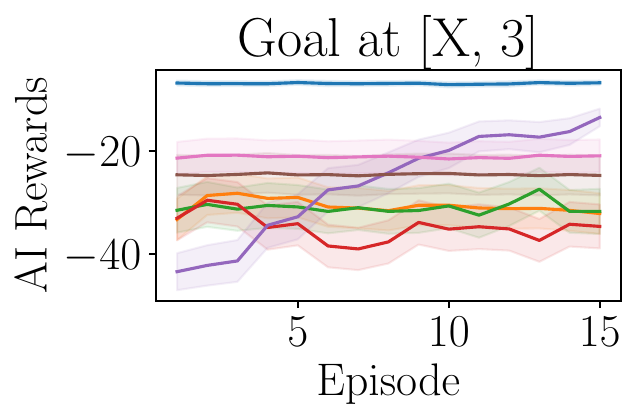}
    \end{subfigure}%
    \begin{subfigure}{0.25\linewidth}
        \includegraphics[width=1\linewidth]{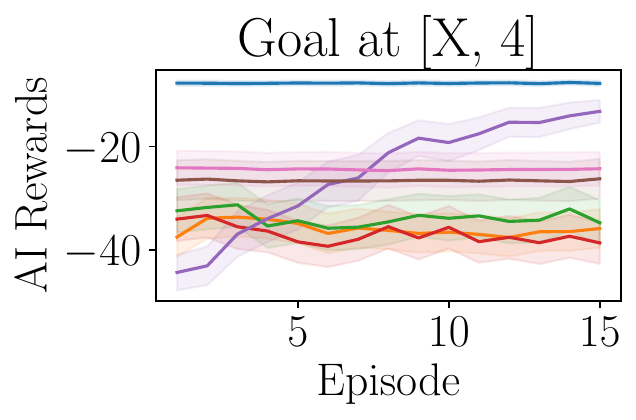}
    \end{subfigure}
    \begin{subfigure}{0.5\linewidth}
        \includegraphics[width=1\linewidth]{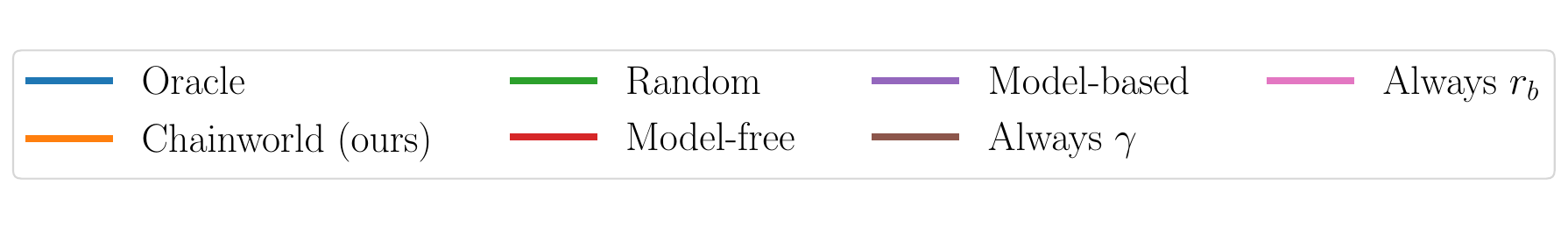}
    \end{subfigure}    
    \caption{Full results for gridworld environments in which goal state $s_g$ moves (and therefore gridworld is not equivalent to chainworld). Chainworld is correctly modeling distance to goal in top row and distance to disengagement in bottom row.}
\end{figure}

\begin{figure}
    \centering
    \begin{subfigure}{0.25\linewidth}
        \includegraphics[width=1\linewidth]{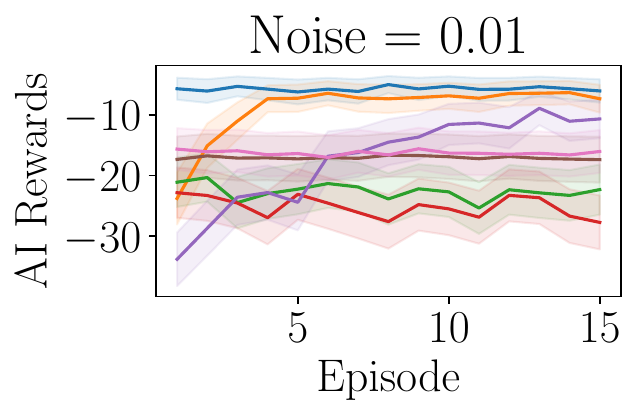}
    \end{subfigure}%
    \begin{subfigure}{0.25\linewidth}
        \includegraphics[width=1\linewidth]{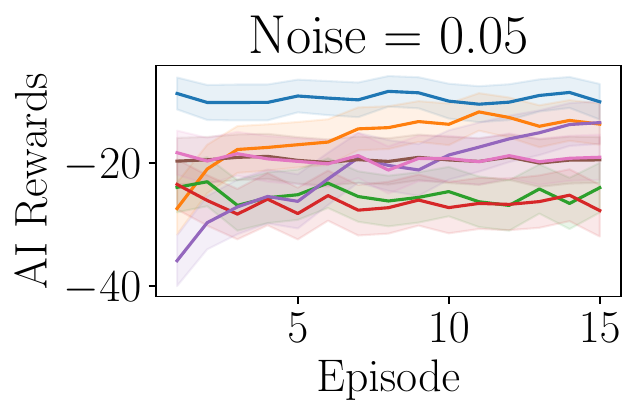}
    \end{subfigure}%
    \begin{subfigure}{0.25\linewidth}
        \includegraphics[width=1\linewidth]{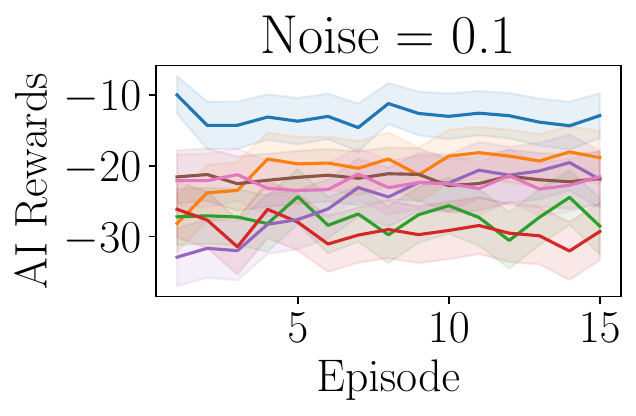}
    \end{subfigure}%
    \begin{subfigure}{0.25\linewidth}
        \includegraphics[width=1\linewidth]{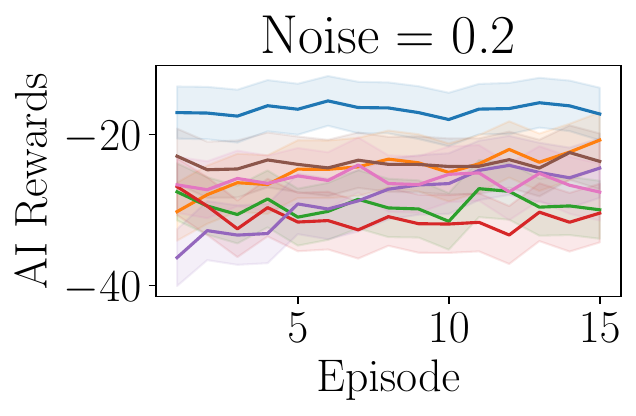}
    \end{subfigure}    
    \begin{subfigure}{0.5\linewidth}
        \includegraphics[width=1\linewidth]{figures/res-legend.pdf}
    \end{subfigure}    
    \caption{Full results for environments in which human follows softmax action selection.}
\end{figure}

\end{document}